\documentclass[twoside]{article}

\usepackage[accepted]{aistats2021}
\usepackage[round]{natbib}

\usepackage{hyperref}
\usepackage{url}
\usepackage{amsmath}
\usepackage{amssymb}
\usepackage{amsthm}
\usepackage{algorithm}
\usepackage{algorithmic}
\usepackage{graphicx}
\usepackage{xcolor}
\usepackage{wrapfig}
\usepackage[thinlines]{easytable}
\usepackage{titlesec}
\usepackage{subcaption}
\usepackage{mathtools}
\usepackage{bold-extra}
\usepackage{bm}
\usepackage{dsfont}
\usepackage{booktabs}
\usepackage{cancel}
\usepackage[font=footnotesize]{caption}
\usepackage{lipsum}

\newcommand\blfootnote[1]{%
  \begingroup
  \renewcommand\thefootnote{}\footnote{#1}%
  \addtocounter{footnote}{-1}%
  \endgroup
}

\newtheorem{theorem}{Theorem}
\newtheorem{lemma}{Lemma}
\newtheorem{corollary}{Corollary}

\newtheorem{proposition}{Proposition}

\newtheorem{assumption}{Assumption}

\newcommand{\E}[2]{\mathbb{E}_{#1}\left[#2\right]}
\newcommand{\Ehat}[1]{\hat{\mathbb{E}}\left[#1\right]}
\newcommand{\Var}[2]{\mathrm{Var}_{#1}\left(#2\right)}
\newcommand{\Cov}[2]{\textrm{\textbf{Cov}}_{#1}\left[#2\right]}
\newcommand\independent{\protect\mathpalette{\protect\independenT}{\perp}}
\def\independenT#1#2{\mathrel{\rlap{$#1#2$}\mkern2mu{#1#2}}}

\newcommand{\Eprime}[2]{\mathbb{E}'_{#1}\left[#2\right]}

\newcommand{\Ptilde}[0]{\widetilde{\Pr}}
\newcommand{\Pbar}[0]{\overline{\Pr}}
\newcommand{\Phat}[0]{\hat{\Pr}}

\newcommand{\lf}[0]{\lambda}

\newcommand{\X}[0]{\mathcal{X}}

\newcommand{\N}[0]{\mathcal{N}}
\newcommand{\B}[0]{\mathcal{B}}

\newcommand{\KL}[0]{D_{\mathrm{KL}}}

\newif\ifsinglecolumn

\usepackage{enumitem}

\setlist{nosep}

\begin{document}

\runningtitle{The Value of Labeled and Unlabeled Data in Latent Variable Estimation}

\runningauthor{Chen, Cohen-Wang, Mussmann, Sala, R\'e}

\singlecolumnfalse

\twocolumn[

\aistatstitle{Comparing the Value of Labeled and Unlabeled Data \\in Method-of-Moments Latent Variable Estimation}

\aistatsauthor{ Mayee F. Chen* \And \; \; Benjamin Cohen-Wang*\ \And \qquad \qquad  Stephen Mussmann \And \newline \qquad Frederic Sala \And Christopher R\'e }

\aistatsaddress{Stanford University} ]

%\saythanks

\begin{abstract}
  Labeling data for modern machine learning is expensive and time-consuming. Latent variable models can be used to infer labels from weaker, easier-to-acquire sources operating on unlabeled data.
Such models can also be trained using labeled data, presenting a key question: should a user invest in few labeled or many unlabeled points? We answer this via a framework centered on model misspecification in method-of-moments latent variable estimation.
Our core result is a bias-variance decomposition of the generalization error, which shows that the unlabeled-only approach incurs additional bias under misspecification. We then introduce a correction that provably removes this bias in certain cases.
We apply our decomposition framework to three scenarios---well-specified, misspecified, and corrected models---to 1) choose between labeled and unlabeled data and 2) learn from their combination. We observe theoretically and with synthetic experiments that for well-specified models, labeled points are worth a constant factor more than unlabeled points. With misspecification, however, their relative value is higher due to the additional bias but can be reduced with correction. We also apply our approach to study real-world weak supervision techniques for dataset construction.

\end{abstract}

 \section{Introduction}
\vspace{-.5em}
%What is the problem?
A key challenge in data-driven fields is the quality of training data. A fixed data collection budget can
provide a large amount of incomplete training data or a smaller but cleaner dataset. Given a choice between these two options, which should we select and which factors should determine this decision? This fundamental question is especially relevant to modern machine learning, where vast amounts of unlabeled data is available. To exploit this without extensive hand-labeling, powerful techniques relying on \emph{latent variable models}---in particular, \emph{method-of-moments}---have been developed to generate labels.

%where huge datasets are needed to train powerful models. The scale of these models has pushed against the limits of labeling budgets, motivating algorithmic generation of datasets. These can be generated from fitting a small sample of labeled data, or from unlabeled data with labels inferred via a \textit{latent variable graphical model}.

\begin{figure*}[t]
    \centering
    \includegraphics[width=.75\textwidth]{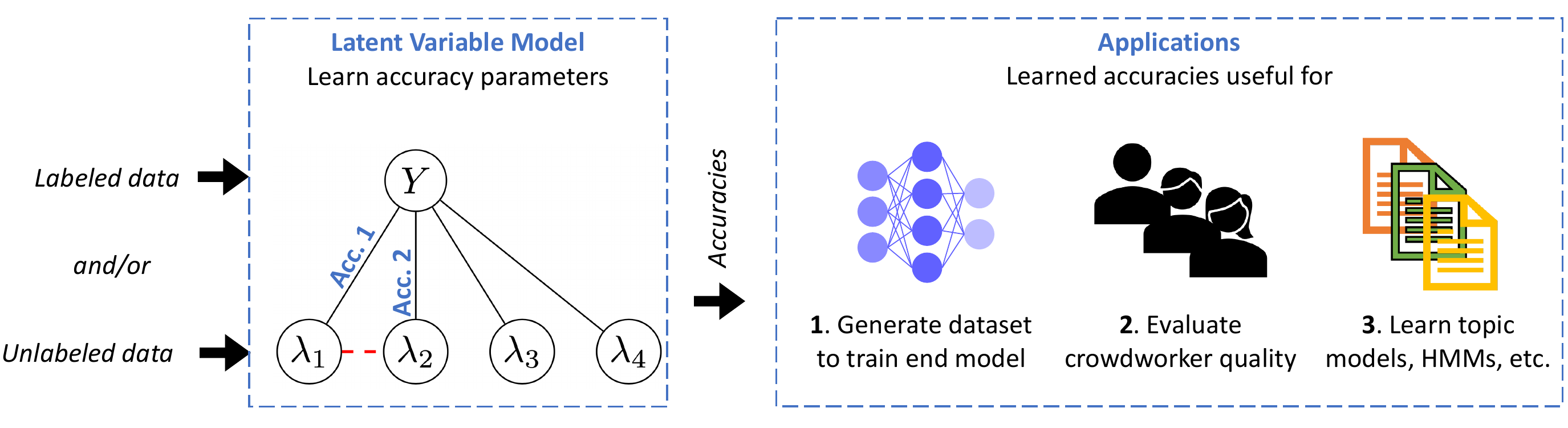}
    \caption{Latent variable methods (e.g., method-of-moments) can infer an unobserved variable ($Y$) by learning the accuracies of correlated sources ($\lf_1, \ldots, \lf_4$). This is done either from unlabeled data or directly from a small amount of labels; we seek a framework to explain the relative value of these choices. A major challenge are unmodeled dependencies between sources (red). Latent variable models have numerous applications.}
    \label{fig:systemdiagram}
\end{figure*}

Latent variable method-of-moments has been used to learn topic models \citep{anandkumar2014tensor} and parse trees \citep{Hsu12}, to evaluate crowdworkers \citep{joglekar2013evaluating}, and to generate training datasets \citep{Ratner19, fu2020fast}.
In these models, the  observable outputs of \textit{sources}
%---variables with some relation to the label---
are used to infer the latent variable, the true label. 
The core challenge is to learn correlations (i.e., \emph{accuracies}) between the sources and the latent variable, which parametrize the model used to infer labels. 
To learn the accuracies, the method of moments, which relies on decomposing multiple observable statistics based on independence among sources, is commonly used to produce simple, closed-form estimators (in contrast to the EM algorithm).
%The accuracies are learned using method-of-moments, which relies on decomposing multiple observable statistics based on independence among sources to produce efficient, closed-form estimators. 
When some labeled data is available, this setup also allows for the accuracy parameters to be directly estimated (\autoref{fig:systemdiagram}).
Thus, given a limited budget, a principle for choosing between labeled and unlabeled data is crucial; this motivates a theoretical framework to understand their relative value. \blfootnote{*Equal contribution. Contact: \texttt{mfchen@stanford.edu}.}

However, unmodeled dependencies among sources---a form of model misspecification---are common and yield inconsistent accuracy estimates, which in turn yield poor inferred labels. This affects the value of data produced with latent variable methods, so misspecification must play a role in our framework. While the question of how to analyze misspecification has been studied in classical statistics and semi-supervised learning, the focus is typically on estimator asymptotics \citep{kleijn2006, kleijn2012, yang2011effect}. Our main challenge, however, is to understand and address misspecification for both parameter estimation and label inference in the finite sample setting.

We theoretically analyze the two alternatives in latent variable methods. In both cases, the output is a conditional distribution for the latent variable given observable sources. For the inputs, the choices are either $n_L$ labeled and/or $n_U$ unlabeled points (and the outputs of $m$ sources per point). We examine misspecification in the form of unmodeled pairwise source dependencies, giving a generalization error analysis for method-of-moments latent variable model performance of the two alternatives. We present a bias-variance decomposition of the generalization error in Theorem \ref{thm:decomposition}, which for both the labeled and unlabeled data cases consists of (i) irreducible error, (ii) variance, and (iii) bias due to model misspecification at inference time. An important consequence is that for unlabeled data, we incur an additional (iv) standing bias due to inconsistent accuracy estimation that scales with the extent of misspecification, namely $\mathcal{O}(d/m)$ for $m$ sources and $d$ unmodeled dependencies among them. 

Next, we turn to correcting this standing misspecification bias. In particular, a simple median-based approach is able to produce consistent estimators given $d = o(m^2)$ and sufficient amounts of unlabeled data. Therefore, in certain cases, the bias $\mathcal{O}(d/m)$ from misspecification can be completely eliminated (Proposition \ref{prop:medians}). This creates three scenarios to consider for our framework: well-specified (i.e. no unmodeled dependencies), misspecified, and corrected settings.

We give two applications of our theoretical framework for the three scenarios. First, we develop a criterion, \textit{the data value ratio}, for choosing between labeled and unlabeled data, which is based on the relative minimum amount of labeled points needed to perform as well as a fixed amount of unlabeled points in terms of generalization error. For well-specified models, labeled data is a constant factor more valuable than unlabeled, but for misspecified models the value grows linearly in $d$ and $n_U$. Furthermore, corrected models are able to \emph{improve the value of unlabeled data}. Second, we combine the estimated parameters from the unlabeled approach, which are biased (and potentially inconsistent), with ones from the labeled approach---in certain cases outperforming either individually. We validate our framework with synthetic experiments, verify the scaling of our generalization error, data value ratio, and the performance of combined estimators across the three settings. 

An important real-world application of our results on latent variable methods are weak supervision (WS) frameworks, in particular data programming \citep{Ratner16}, used in a huge range of products and systems across industry and academia. WS frameworks construct datasets without ground-truth annotations by using unlabeled points and distant or weak sources, such as heuristics~\citep{gupta2014improved}, external knowledge bases~\citep{mintz2009distant,craven:ismb99,takamatsu:acl12}, or noisy crowd-sourced labels~\citep{karger2011iterative,dawid1979maximum}. Data programming encompasses many such prior approaches, and has shown excellent results with the method-of-moments approach \citep{fu2020fast}. We perform a real-world WS case study, where ground-truth source dependencies are not known, but sources are likely to be correlated to some extent. We observe that the relative value of labeled data is large, but the value of unlabeled data can be increased via our median correction. With equal amounts of data, the F1-score of the WS model for constructing datasets with a baseline unlabeled approach is 64.81 and the score of a labeled approach is 71.79, but the score of an unlabeled approach with correction is 68.12. This suggests that our theoretical explanation of the effects of misspecification can account for some of the behavior of models on real data.

 \section{Related Work}
 \vspace{-.5em}
%\paragraph{Method-of-Moments Estimators}
%A tranche of latent variable model literature relies on the method-of-moments approach. These estimation approaches 
%often involve decomposing multiple observable ``views'' of latent variables. The decomposition is applied to closed-form systems of equations~\citep{joglekar2013evaluating, fu2020fast} and tensors~\citep{anandkumar2014tensor, chaganty2014estimating, Bhaskara14}, but in all cases, conditional independence among the views is required for the structural relationships to hold. We focus on a particular approach in this paper, but our analysis can provide error decomposition under misspecification for other method-of-moments estimators.

\vspace{-0.5em}
\paragraph{Misspecification in Graphical Models}
The asymptotic effect of misspecification on parameter estimation is studied by \cite{kleijn2012}, extending the Bernstein-Von Mises theorem to cases where observed samples are not of the assumed parametric distribution. %They find that certain estimators, such as MLE, converge to a normal distribution centered at the point that results in a parametrized distribution closest in Kullback-Leibler divergence to the true distribution, and this property is directly used in asymptotic analysis of estimators in Bayesian and variational inference \cite{wang2019variational, hongmartin2020}. 
However, their main results do not fully extend to method-of-moments estimators. Other analyses of model misspecification directly examine distribution families, such as \cite{JogL15}'s lower bound on KL-separation of Gaussian graphical models. This bound is important for modeling errors in inference,
%a version of their mutual information bounds appear as an inference bias in our error analysis, 
but it does not illustrate our additional error in parameter estimation. More generally, works on misspecification either study a particular class of techniques~\citep{Blasi13} or a particular model and propose repairs~\citep{Grunwald17}, while we compare effects on data types.

\vspace{-0.5em}
\paragraph{Structure Learning}
One way to reduce misspecification is to produce a more refined model. Graphical model structure learning aims to do so in both the supervised~\citep{Ravikumar11, Loh13} and unsupervised cases~\citep{Chandrasekaran12,Meng14, bach2017learning, varma2019learning}. However, these works present computational challenges, require (often strong) conditions to hold, and do not analyze the downstream impact of errors. Our approach instead focuses on understanding the impact of errors, but it is also applicable to partial recovery that often results from structure learning.

\paragraph{Semi-Supervised Learning} involves learning from a small set of labeled points and a larger set of unlabeled points~\citep{Chapelle09, zhu2009introduction}. There are several works on the relative value of labeled and unlabeled data in semi-supervised settings, typically requiring assumptions about the data distribution (e.g., cluster, manifold)~\citep{castelli1996relative, singh2008unlabeled, ben2008does}. In contrast, our work explicitly considers violations of model assumptions by quantifying how misspecification influences the relative value of labeled and unlabeled data. This direction has been explored by \cite{yang2011effect}, who study asymptotic performance degradation due to misspecification in semi-supervised maximum likelihood estimation; however, their results only describe the conditions under which degradation occurs.
We further bound the extent of degradation, handle the finite sample case, and propose a way to mitigate misspecification.

\paragraph{Valuation of Data} Several methods have been proposed for measuring the value of individual data points, often based on the Shapley value \citep{ghorbani2019data, jia2019towards}. Such valuations can then be used to inform what additional data should be acquired to improve a model. Our goal in valuing labeled versus unlabeled data is similar, but we do not value individual data points and instead compare performance of classifiers trained on labeled data, unlabeled data, and on both. 
%because we are not valuing individual data points, our methods are more straightforward, involving comparing a classifier trained on entirely labeled data to one trained entirely on unlabeled data.

% There are a wide variety of such techniques: the missing labels can be modeled as latent variables and expectation maximization (EM) can be used~\citep{Nigam2000}. Another approach uses a graph structure to propagate labels, relying on MRFs~\citep{Zhu02} or kernels~\citep{Kondor02}. In contrast to such techniques, WS can operate with no labels. It does not focus on extending any given labels, but rather learning the weak sources that can be applied to any data. 

%does not require any labels, and by learning the accuracies of sources that can then be applied to any unlabeled data

%\todo{maybe data shapley or stuff like that?}
%There has also been recent work in data valuation \citep{ghorbani2019data}, but they focus on a per-point valuation for a fixed supervision setting rather than a comparison between labeled and unlabeled data approaches. 
\section{Background and Problem Setting} \label{sec:background}
\vspace{-.5em}
We start with background on latent variable models and introduce the model we analyze. We explain the two stages---learning accuracies and inferring labels---for both the labeled and unlabeled cases, and conclude with how to evaluate the model.

\begin{figure*}[t]
    \centering
    \hspace*{-.5cm}\includegraphics[width=.7\textwidth]{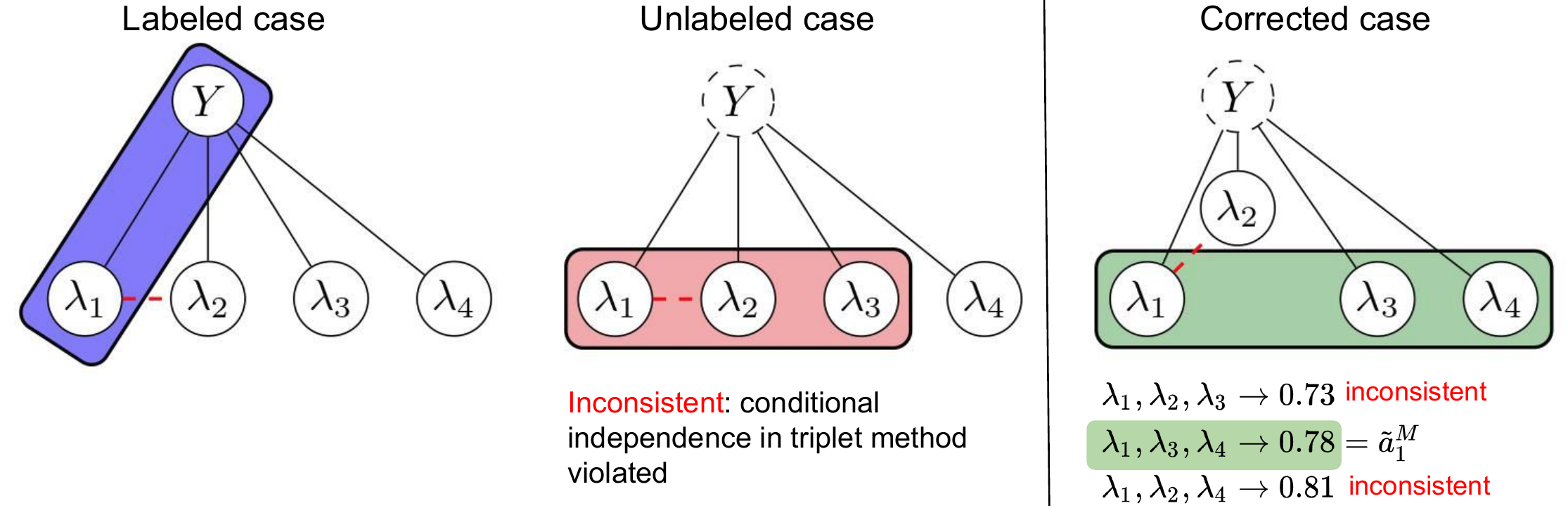}
    % \hspace*{-.6cm}\includegraphics[width=.525\textwidth]{figures/models.pdf}
    %
    \caption{Estimating accuracy parameter of $\lf_1$ with unmodeled dependency (red edge), leading to misspecification. 
    Boxes indicate observable variables used for accuracy estimation. 
    Left: model with access to label data. 
    The accuracy parameter is directly estimated using $Y$ and $\lf_1$ and is not impacted by the unmodeled dependency. 
    Center: latent model with unlabeled data and unobserved $Y$. The boxed triplet includes the unmodeled dependency, leading to inconsistent estimate $\widetilde{a}_1^U$. 
    Right: Corrected model using medians. The boxed triplet, chosen as the median estimate among ${m-1 \choose 2}$ triplets, excludes the dependency, yielding a consistent $\widetilde{a}_1^M$.}
    \label{fig:biases}
\end{figure*}

%In this section, we provide background on the weak supervision problem setting, the graphical model assumptions for the true distribution as well as the misspecified model used, and the method-of-moments approach commonly used to learn model parameters when labeled data is unavailable.  

%This work explores when and how a small number of labeled data points can be used to improve the quality of noisy labels produced by the data programming pipeline. In Section \ref{sec:setting_dp} we provide background about data programming and methods for combining labeling functions. In Section \ref{sec:setting_adding_labels} we describe the modified setting where some labeled data is available. Finally, we discuss evaluating data programming methods in Section \ref{sec:setting_evaluation}.

%\subsection{Background: Data Programming}
%\label{sec:setting_dp}

\paragraph{Setup} In latent variable models, a number of sources are observed and used to infer the latent variable.
%In weak supervision frameworks, e.g., data programming \citep{Ratner16}, a labeled dataset is built by acquiring noisy sources operating on unlabeled data. %The outputs of these sources %. Often a programmatic heuristic written by practitioners, each weak source either votes or abstains on a data point, thus overall producing multiple noisy labels per point that are then
%are aggregated by a \emph{label model} to output the training label for each point. 
The input is usually $n_U$ unlabeled data points, but in our setting we also consider a small \emph{labeled} dataset of $n_L$ samples. The output is a large, labeled dataset.

Let $X \in \X$ and $Y \in \mathcal{Y} = \{-1, 1\}$. We consider an unlabeled dataset $\bm{X}_U = \{x_i^U\}_{i = 1}^{n_U}$ and a labeled dataset $(\bm{X}_L, \bm{Y}_L) = \{(x_i^L, y_i^L) \}_{i = 1}^{n_L}$ drawn from the distribution of $(X, Y)$.
There are $m$ sources, each
outputting a value in $\{-1, +1\}$ via a deterministic function
%on $X$ via a deterministic \textit{labeling function} (LF)
$\lf_j: \X \rightarrow \mathcal{Y}$ for all
$j \in [m]$. % ($\lf_j(X) = 0$ is an abstain). 
Our goal is to use the outputs of $\bm{\lf}$, the vector of sources, to construct a model to infer $Y$.

To infer $Y$, we learn the model $\Pr(Y | \bm{\lf})$ to produce soft labels $\widetilde{y}_i := 2 \Pr(Y = 1 | \bm{\lf} = \bm{\lf}(x_i)) - 1 \in [-1, 1]$ for each $x_i$ by applying
the $m$ sources' functions to the datasets $\bm{X}_U$ or $(\bm{X}_L, \bm{Y}_L)$. The overall approach has two steps: (i) learn the latent variable model (using labeled or unlabeled data), and (ii) infer labels $\widetilde{y}_i$.

%We use the label model outputs as the final predictions for $Y$ such that its generalization error precisely captures the effect of model misspecification and labeled versus unlabeled data. However, the labels are often used to train an additional end model, for which our analysis on the quality of training data is still valid.

% \paragraph{Misspecification in weak supervision}
% To learn the joint distribution $\Pr(Y,\bm{\lf})$, we typically assume that there exist subsets of labeling functions that are conditionally independent given $Y$. When 

\paragraph{Theoretical model}
We pick a simple model that captures many latent variable model settings and still presents all of the challenges for comparing between the types of data. We assume an Ising model for $\Pr(Y, \bm{\lf})$; the only difference between the labeled and unlabeled setting is that $Y$ is latent in the latter. %Following the standard notation,
%We present the standard graphical model assumptions for the true distribution $\Pr(Y, \bm{\lf})$. Then, we discuss the misspecified graphical model used in weak supervision and how it outputs labels.
The dependency graph is $G = (V, E)$, where $V = Y \cup \bm{\lf}$ and $E$ consists of edges from $Y$ to the sources as well as the $d$ edges among the sources, $E_{\lf}$. %The graph $G$ specifies dependencies using standard technical notions from the PGM literature~\citep{koller2009probabilistic, Lauritzen, wainwright2008graphical}. 
The lack of an edge in $G$ between a pair of
variables indicates independence conditioned on a separator set~\citep{Lauritzen}, so the true distribution can be modeled as
\begin{align}
    \Pr(Y, \bm{\lf}; \theta) = \frac{1}{Z} \exp \Big(\theta_Y + \sum_{i = 1}^m \theta_i \lf_i Y + \sum_{\mathclap{(i, j) \in E_{\lf}}} \theta_{ij} \lf_i \lf_j \Big),
     \nonumber 
    \vspace{-0.5em}
\end{align}
with cumulant function $Z$ and the set of canonical parameters $\theta = \{\theta_Y, \theta_i \; \forall i, \theta_{ij} \; \forall (i , j) \in E_{\lf} \}$. For cleaner presentation, we assume $\theta \ge 0$ (no sources that disagree with others or $Y$ on average) and $E_{\lf}$ is sparse enough such that $\text{deg}(\lf_i) \le 2$ for all $\lf_i$ (each source is conditionally dependent on at most one other source). %This leads to model misspecification when edges are unknown.

%With this, we can explain the two steps involved in the process.

\paragraph{Inference} The label is computed using a naive Bayes approach that assumes all sources are conditionally independent with $E_{\lf} = \emptyset$:
\ifsinglecolumn
\begin{align}
    \widetilde{\Pr}(Y &= 1 | \bm{\lf} = \bm{\lf}(X)) = \frac{\prod_{i = 1}^m \widetilde{\Pr}(\lf_i = \lf_i(X) | Y = 1) \Pr(Y = 1)}{\hat{\Pr}(\bm{\lf} = \bm{\lf}(X))}, 
    \label{eq:inference}
\end{align}
\else
\begin{align}
    \widetilde{\Pr}(Y &= 1 | \lf = \lf(X)) \nonumber \\
    &= \frac{\prod_{i = 1}^m \widetilde{\Pr}(\lf_i = \lf_i(X) | Y = 1) \Pr(Y = 1)}{\hat{\Pr}(\lf = \lf(X))}, 
    \label{eq:inference}
\end{align}
\fi
where the class balance $\Pr(Y = 1)$ is assumed to be known, $\hat{\Pr}$ is an empirical probability
, and $\widetilde{\Pr}$ indicates an estimated probability resulting from the parameter estimation step described below. %The class balance $\Pr(Y = 1)$ is assumed to be known, so the only term to be estimated is $\widetilde{\Pr}(\lf_i = \lf_i(X) | Y = 1)$ for each $\lf_i$, which we describe how to do below. 
In practice, the conditional independence assumptions required for (\ref{eq:inference}) may not hold, but dependencies among sources are often unknown%\footnote{We discuss techniques to rectify this and their costs at the end of the following section.}
. Therefore, conditional independence is assumed, and we may suffer from misspecification in inferring our probabilistic labels. %In the following we compute the cost of this error and its implications for the value of data. 

%Note that in the case of labeled data, we can compute this directly, but for unlabeled data, we estimate this probability using the graphical model structure.

\paragraph{Learning parameters with method-of-moments}
For the labeled dataset, we learn $\widetilde{\Pr}(\lf_i = \lf_i(X) | Y = 1)$ in (\ref{eq:inference}) directly from samples, as $Y$ is observed. 

For the unlabeled dataset, we use the method-of-moments estimator from \cite{fu2020fast} (described in Appendix \ref{sec:alg}), which relies on the property
%\begin{proposition} 
that if $\lf_i \independent \lf_j | Y$, then $\lf_i Y \independent \lf_j Y$.
%\label{prop:triplet}
%\end{proposition}
%
This implies that $\E{}{\lf_i Y} \cdot \E{}{\lf_j Y} = \E{}{\lf_i \lf_j Y^2} = \E{}{\lf_i \lf_j},$ which is directly estimable.
Define $a_i := \E{}{\lf_i Y}$ as the unknown \textit{accuracy} of $\lf_i$. If we can introduce a third $\lf_k$ that is conditionally independent of $\lf_i$ and $\lf_j$, we have a system of equations that can be solved using observable statistics. We use this \textit{triplet method} to recover these accuracies: we choose two $\lf_j$, $\lf_k$ at random for each $\lf_i$ and solve up to sign:
\begin{align}
    |\widetilde{a}_i^{(j, k)}| := \sqrt{\bigg| \frac{\Ehat{\lf_i \lf_j} \Ehat{\lf_i \lf_k}}{\Ehat{\lf_j \lf_k}} \bigg|},
    \label{eq:triplet}
\end{align}
where $\hat{\mathbb{E}}$ is an empirical estimate of the expectation.

We use the estimated $\widetilde{a}_i^U := \widetilde{a}_i^{(j, k)}$ to directly compute $\widetilde{\Pr}(\lf_i = \pm 1 | Y = 1)$ for \eqref{eq:inference}. However, random $\lf_j$ and $\lf_k$ may not satisfy conditional independence, and thus we incur error in estimating accuracies due to misspecification in a way unique to the unlabeled setting. Figure \ref{fig:biases} (left, center) describes how this misspecification impacts learning the accuracies in the labeled versus unlabeled cases. We aim to capture the role of this misspecification in our evaluation. % The following section characterizes this notion.

\paragraph{Evaluating the model}
We define the model's generalization error as $R = \mathbb{E}_{(Y, \bm{\lf}), \N, \tau}[l(\widetilde{Y}, Y)]$ where expectation is taken over the distribution of $(Y, \bm{\lf})$, $\N$ (the random dataset used), and $\tau$ (the algorithmic randomness if applicable, i.e. the triplets used in method-of-moments). $l(\cdot, \cdot)$ here is the cross entropy loss,
%\steve{maybe move this up to the beginning of the paragraph, so the ordering is more natural?}  defined on a point $(x_i, y_i)$ as
$l(\widetilde{y}_i, y_i) = -\frac{1 + y_i}{2} \log \widetilde{\Pr}(Y = 1 | \bm{\lf} = \bm{\lf}(x_i))
- \frac{1 - y_i}{2} \log \widetilde{\Pr}(Y = -1 | \bm{\lf} = \bm{\lf}(x_i)).$
Let $R_U$ denote the error for the unlabeled dataset and $R_L$ for labeled.

\section{Theoretical Results} \label{sec:theory}
\vspace{-.5em}
We theoretically analyze the quality of the latent variable model, taking into account the impact of misspecification when using unlabeled versus labeled data. %The results produce a practical criterion for valuing supervision. %First, we discuss the error in the labeling functions' accuracy parameters $a_i$ for both the labeled and unlabeled data case. 
%Then, we propagate this error into our computation of the 
In Section \ref{subsec:decomp} we give 
%bounds the accuracy estimation error that arises from using unlabeled data under misspecification. Next, we use these to present
an exact decomposition of the generalization error of the latent variable model, which demonstrates how misspecification is present in both the parameter learning and inference steps of the model when data is unlabeled and only present in the latter when data is labeled. In Section \ref{subsec:scaling}, we bound the generalization error using this framework to show how the unlabeled case has an additional standing bias of $\mathcal{O}(d/m)$. Given this standing bias, in Section \ref{sec:misspec} we introduce a simple method that can in some cases correct for dependency-based misspecification, and we analyze its impact on generalization error. In \ref{sec:synthetic4} we present synthetic experiments that verify our results.

% misspecification impacts parameter estimation only when the data is unlabeled; therefore, this particular term is crucial in determining which dataset has lower generalization error. We then discuss ways to mitigate misspecification's impact on parameter estimation. Lastly, using our error analysis we propose our data value criterion for choosing between and combining unlabeled and labeled data.
%and present upper and lower bounds on the generalization error of our label model, again for both labeled and unlabeled data cases separately. In each part, we contrast the labeled and unlabeled data settings. Lastly, we use these bounds to propose a data value criterion.

\subsection{Decomposition Framework}
\label{subsec:decomp}%for WS Label Model}
\vspace{-0.5em}

Our first result is a decomposition of the generalization error into four components. The last two components, the inference bias and parameter estimation error, reflect the role of misspecification.
\begin{theorem}
The generalization error has the following decomposition:
\ifsinglecolumn
\begin{align}
    &\E{}{l(\widetilde{Y}, Y)} = \underbrace{H(Y | \bm{\lf})}_{{\mathrm{Irreducible \; error}}} - \underbrace{\E{\N}{\KL(\Pr(\bm{\lf}) || \hat{\Pr}(\bm{\lf}))}}_{\mathrm{Observable \; sampling \; noise}} + \sum_{{(i, j) \in E_{\lf}}} \underbrace{I(\lf_i; \lf_j | Y)}_{\mathrm{Inference \; bias}} +  \sum_{i = 1}^m \underbrace{\E{Y, \N, \tau}{\KL(\mathrm{Pr}_{\lf_i | Y} || \widetilde{\Pr}_{\lf_i | Y })}}_{\mathrm{Parameter \;estimation \;error}} \nonumber,
\end{align}
\else 
\begin{align}
    &\E{}{l(\widetilde{Y}, Y)} = \underbrace{H(Y | \bm{\lf})}_{\mathclap{\mathrm{Irreducible \; error}}} - \underbrace{\E{\N}{\KL(\Pr(\bm{\lf}) || \hat{\Pr}(\bm{\lf}))}}_{\mathrm{Observable \; sampling \; noise}} + \nonumber \\ %\E{}{\frac{1 + a}{2} \log \frac{1 + \widetilde{a}}{1 + a}  + \frac{1 - a}{2} \log \frac{1 - \widetilde{a}}{1 - a} } - \nonumber \label{eq:decomposition} \\
    &\sum_{\mathclap{(i, j) \in E_{\lf}}} \;\;\; \underbrace{I(\lf_i; \lf_j | Y)}_{\mathrm{Inference \; bias}} +  \sum_{i = 1}^m \underbrace{\E{Y, \N, \tau}{\KL(\mathrm{Pr}_{\lf_i | Y} || \widetilde{\Pr}_{\lf_i | Y })}}_{\mathrm{Parameter \;estimation \;error}} \nonumber,
\end{align}
\fi

where $I(\lf_i; \lf_j | Y)$ is the conditional mutual information between sources and $H(Y|\bm{\lf})$ is conditional entropy. $\Pr$ refers to the true data distribution, while $\hat{\mathrm{Pr}}$ and $\widetilde{\mathrm{Pr}}$ refer to the estimated probabilities in \eqref{eq:inference}.

\label{thm:decomposition}
\end{theorem}

We now discuss each term above. The first two terms are independent of misspecification and are present in both the unlabeled and labeled cases:
\begin{itemize}
  \setlength\itemsep{0em}
    \item Irreducible error: an intrinsic property of the distribution of $(Y, \bm{\lf})$, always present in bias-variance decomposition. 
    \item %\steve{I'm confused by the negative sign in front of this term.} 
    Observable sampling noise: the expected KL divergence between the true marginal distribution of the observable sources and the empirical distribution. Particular to our inference approach, it is a common notion of sampling noise~\citep{domingos2000unified, yang2020rethinking} and approaches $0$ asymptotically.
\end{itemize}

For the last two terms, misspecification plays a different role depending on the data type. 
\begin{itemize}
  \setlength\itemsep{0em}
    \item Inference bias: the conditional mutual information among dependent sources. Particular to our inference approach, it is the approximation error of using marginal singleton probabilities rather than their product distributions. Therefore, it represents the role of misspecification at the inference step \eqref{eq:inference} and is present for both data types. It is independent of parameter estimation method.
    \item Parameter estimation error: the difference between the true and estimated distribution of $\lf_i | Y$. For the labeled approach, this error corresponds to sampling noise and asymptotically approaches $0$. For the unlabeled approach, it directly depends on the estimation error of accuracies in \eqref{eq:triplet}. However, these estimators are biased, as are many method-of-moments approaches. Furthermore, misspecification makes the estimators inconsistent when $\lf_i, \lf_j,$ and $\lf_k$ used to produce $\widetilde{a}_i^{(j, k)}$ are not pairwise conditionally independent.
\end{itemize}

We now discuss in detail the scaling of these last two terms, which highlights the tradeoff between labeled and unlabeled data under misspecification.

%We now examine the role of misspecification in these bounds. The decomposition of the generalization error in \eqref{eq:decomposition} isolates the accuracy parameters in the first term, where $a_i$ is the true accuracy and $\widetilde{a}_i$ is the estimated accuracy. In the case of unlabeled data, $\widetilde{a}_i$ is always a biased estimator as suggested by \eqref{eq:triplet} with bias $\mathcal{O}(1/\sqrt{n})$. However, it is also inconsistent due to misspecification with additional error $\mathcal{O}(d/m)$, since the $\lf_i, \lf_j, \lf_k$ used to estimate $a_i$ may not be pairwise conditionally independent. On the other hand, when we have labeled data, accuracies are observable and thus the only difference between $a_i$ and $\widetilde{a}_i$ is due to sample noise---misspecification has no impact.

%The remaining terms in \eqref{eq:decomposition} are the same for both labeled and unlabeled data. The conditional entropy $H(Y | \bm{\lf})$ is the irreducible error, and the expected KL-divergence $\E{}{D_{KL}(\Pr(\bm{\lf}) || \hat{\Pr}(\bm{\lf}))}$ is a form of variance caused by using the empirical pdf for the observable labeling functions. Lastly, However, since this term is the same for both the labeled and unlabeled data settings, it does not play a role in choosing between them using the data value criterion; the only place where misspecification matters in the criterion is the accuracy error $\mathcal{O}(\frac{d}{m})$.

\subsection{Scaling of the Generalization Error} %under Misspecification}
\label{subsec:scaling}
\vspace{-0.5em}

We bound the terms in Theorem \ref{thm:decomposition} to understand the scaling of error due to misspecification in both the unlabeled and labeled cases. Since the irreducible error is always present, we bound \textit{excess generalization error}, defined as $R^{e}_L = R_L -  H(Y|\bm{\lf})$ for labeled data and similarly $R^{e}_U$ for unlabeled data.  We use $\B_I = \sum I(\lf_i; \lf_j | Y)$ for the inference bias in these bounds since it is independent of our two cases, and while it scales in $d$, it is simply a measurement over the true data distribution. We present upper bounds here and lower asymptotic bounds in 
Appendix \ref{subsec:supp_lowerbound}.

We first bound $R^{e}_L$. 
\begin{theorem}
Suppose that there are $|E_{\lf}| = d$ unmodeled dependencies. When we use the latent variable model described in section \ref{sec:background} with $n_L$ labeled samples,
%\bcw{Shouldn't $\varepsilon_{\max}$ appear in this expression somewhere?}
\begin{align}
   R_L^{e} \le \frac{m}{2n_L} + \B_I + o(1/n_L).
\end{align}
\label{thm:labeled}
\end{theorem}

%\steve{what about the ``observable sampling noise''?} 
In this bound, $\frac{m}{2n_L}$ is an upper bound on parameter estimation error. It represents the sampling noise of $\widetilde{a}_i^L = \Ehat{\lf_i Y}$, which asymptotically approaches $0$. %$d\log 2$ is a simple upper bound on the inference bias that is obtained by definition of the mutual information. 
Therefore, the only standing bias is $\B_I$ due to inference approach. When the model is well-specified, the excess error is $\mathcal{O}(1/n_L)$, and thus for large $n_L$ our generated labels eventually follow the true distribution $\Pr(Y | \bm{\lf})$.  

We next present an upper bound on the excess generalization error in the unlabeled case. Define $\varepsilon_{ij} = \E{}{\lf_i \lf_j} - \E{}{\lf_i Y} \E{}{\lf_j Y}$ as the extent of misspecification on a single pair of sources, and let $0 \le \varepsilon_{\min} \le \varepsilon_{ij} \le \varepsilon_{\max}$ for all pairs $(i, j)$ under our model assumptions in section \ref{sec:background}. The exact value of $\varepsilon_{ij}$ in terms of canonical parameters is in Appendix \ref{subsec:supp_unlabeledthm}.

%be bounds on the value of $\E{}{\lf_i \lf_j} - \E{}{\lf_i Y} \E{}{\lf_j Y}$ over all pairs $(\lf_i, \lf_j)$ in terms of the canonical parameters $\Theta$. Then,
\begin{theorem}
Suppose that there are $|E_{\lf}| = d$ dependencies. When we use the latent variable model described in section \ref{sec:background} using $n_U$ unlabeled samples, 
\vspace{-0.4em}
\ifsinglecolumn
\begin{align}
    R_U^{e} \le  & \varepsilon_{\max} \left(\frac{c_1 d}{m} + \frac{c_2}{\sqrt{n_U}} + \frac{c_3d }{m n_U}\right) + \frac{c_4 m}{n_U} + \B_I + o(1/n_U), 
\end{align}
\else
\begin{align}
    R_U^{e} \le  & \varepsilon_{\max} \left(\frac{c_1 d}{m} + \frac{c_2}{\sqrt{n_U}} + \frac{c_3d }{m n_U}\right) \\
    + &\frac{c_4 m}{n_U} + \B_I + o(1/n_U), \nonumber
\end{align}
\fi
where $c_1, c_2, c_3,$ and $c_4$ are constants depending on the intrinsic quality of the sources (Appendix \ref{subsec:supp_unlabeledthm}). 

\label{thm:unlabeled}
\end{theorem}

In this bound, we again have an observable sampling noise $\frac{c_4 m}{n_U}$, where the the constant term comes from estimating $\Ehat{\lf_i \lf_j}$ in \eqref{eq:triplet} rather than $\Ehat{\lf_i Y}$ in the labeled approach. However, here the parameter estimation error has an additional term $\B_{\mathrm{est}} :=\varepsilon_{\max} \left(\frac{c_1 d}{m} + \frac{c_2 }{\sqrt{n_U}} + \frac{c_3 d}{m n_U}\right)$ which depends on misspecification. Therefore, asymptotically the unlabeled approach has a standing bias bounded by $\frac{c_1 d \varepsilon_{\max}}{m} +\B_I$ in comparison to the labeled case's $\B_I$, and the finite-sample regime contributes additional sampling noise for the unlabeled approach that scales in $\varepsilon_{\max}$. In the case the model is well-specified $(d = 0, \varepsilon_{\max} = 0)$, the only term present is $\frac{c_4 m}{n_U}$, so our latent variable model would also approach the true distribution of $\Pr(Y | \bm{\lf})$ but at a different rate than the labeled case. 

\vspace{-0.5em}
\paragraph{Partial Recovery} Our results hold almost exactly for the partial recovery case, where $d'$ out of $d$ dependencies are recovered (e.g. via structure learning) and our method in \eqref{eq:triplet} avoids choosing known pairs of dependent sources. In particular, the additional estimation error now scales at rate $\frac{(d - d') \varepsilon_{\max}  }{m - 2d'}$.

\vspace{-0.5em}
\subsection{Correcting for misspecification}\label{sec:misspec}
\vspace{-0.5em}
How can we reduce the penalty for dealing with such unrecovered dependencies? 
%This is a common problem in method-of-moments approaches to latent variable estimation beyond WS~\citep{anandkumar12, chaganty2014estimating}---in particular, multi-view learning---where most literature assumes the dependencies are known or that the true data distribution exhibits conditional independence among all observable variables.
We examine how to reduce misspecification for our estimator described in \eqref{eq:triplet}. Our correction can be applied to other method-of-moments approaches \citep{anandkumar12, chaganty2014estimating}, discussed in Appendix \ref{subsec:supp_mom}.
 %to show its use extended to other weak supervision approaches and more general latent variable estimation problems.

%Note that this correction for misspecification is applicable to many method-of-moments approaches beyond its use in WS, such as \cite{anandkumar12, chaganty14}, in order to more generally learn parameters of latent variable models. 

In our estimation approach, if there exists an $\lf_i$ such that there are no $\lf_j, \lf_k$ where all three sources are pairwise conditionally independent given $Y$, then it is not possible to learn  $a_i$. In less demanding cases, %the natural approach is to learn these dependencies via \emph{structure learning}. This can be done in the unlabeled data setting \citep{bach2017learning, varma2019learning}. However, it may involve solving a challenging optimization problem (e.g., an SDP). 
we suggest an alternative approach based on \emph{medians}. Recall that misspecification impacts accuracy estimation error because random triplets that violate pairwise conditional independence are selected to compute our $\widetilde{a}_i^U$. To reduce this impact, we estimate each $a_i$ by computing the median accuracy over all pairs $\lf_j, \lf_k$ using \eqref{eq:triplet} a total of ${m - 1 \choose 2}$ times, as shown in Figure \ref{fig:biases} (right). The intuition behind this approach is that inconsistent estimates produced by dependent sources have more extreme values and thus may not impact the median.
\begin{proposition}
%Denote $\widetilde{a}_i^{(j, k)}$ as the estimate in \eqref{eq:triplet}. 
 Let $\widetilde{a}_i^M = \mathrm{median}(\{\widetilde{a}_i^{(j, k)} \; \forall \; j, k \neq i \})$. Then $\widetilde{a}_i^M$ is not affected by misspecification and is thus a consistent estimator if $m > 5$, $d < \frac{(m - 1)(m - 2)}{4}$, and $n_U \ge n_0$, where $n_0$ is $\omega(1/\varepsilon_{\min}^2)$.

Refer to $\rho_{n_U} = \max_i \E{}{(\widetilde{a}_i^M - a_i)^2}$ as the maximum MSE for $\widetilde{a}^M$. Under these conditions, the excess generalization error $R_M^e$ from using $n_U$ unlabeled samples and a corrected model is, for constant $c_\rho$,
\begin{align}
    R_M^{e} \le c_\rho m  \rho_{n_U} + \B_I + o(1/n_U).
\end{align}
\label{prop:medians}
\end{proposition}
\vspace{-2em}
%While we assume sparsity of $E_{\lf}$ for simpler analysis in section \ref{sec:theory}, the conditions above on $d$ and $m$ apply more broadly \mayee{?}. 
While $\rho$ can be analyzed in detail as a variant of a medians-of-means estimator, we stress that $\lim_{n_U \rightarrow \infty} \rho_{n_U} = 0$. Thus the standing bias of order $\mathcal{O}(d/m)$ due to misspecification can be eliminated. This reduction has many implications for the value of labeled vs. unlabeled data in \textit{corrected settings}. %We verify this improvement in generalization error in section \ref{sec:exp}. 

%Finally, this median-based approach applies to other steps of the label model. A similar technique, where we partition the sources into $q$ groups, perform inference separately with each group to get $\widetilde{Y}$, and take the median, can be used to reduce the impact of misspecification in the approximate inference step.

%This reduces the probability of selecting an unbiased triplet (as the median) to $O(1-(d/m^2)^{t+1})$. This technique helps us in cases where $d = o(m^2)$. 

\vspace{-0.5em}
\subsection{Synthetic Experiments} \label{sec:synthetic4}
\vspace{-0.5em}
We validate the fundamental principles of our theoretical framework using synthetic data. We measure the excess generalization error vs. $\log(n)$ in the well-specified, misspecified and corrected settings on synthetic data with $m=10$ sources, accuracies drawn uniformly from $[.55, .75]$ and extent of misspecification fixed at $\varepsilon=0.1$. To approximate expected excess generalization error for each $n$, we average results over $1000$ samples. A more detailed protocol for synthetic experiments is available in Appendix \ref{subsec:supp_synthetics}.

%First, in the labeled and unlabeled settings, a well-specified model achieves an asymptotic error of zero. Under misspecification, we see the presence of $\B_I$ in both settings, and an additional parameter estimation bias $\B_{\mathrm{est}}$ in the unlabeled setting. Finally, we expect a corrected model to mitigate this parameter estimation bias. 
%We report how our empirical results match up with theory in \autoref{fig:gen_err}. 
Our results are in \autoref{fig:gen_err}. With no misspecification ($d=0$) the labeled and unlabeled estimators both tend towards zero in the two graphs. Under misspecification ($d=5$), we see that learning from unlabeled data results in an additional standing bias that parallels $\B_{\mathrm{est}}$. Median aggregation reduces this bias and results in error converging to roughly similar values, paralleling $\B_I$, in both the unlabeled and labeled cases in the two graphs. These observations are consistent with our theoretical findings.

\begin{figure}
    \centering
    \ifsinglecolumn
    \includegraphics[width=.75\textwidth]{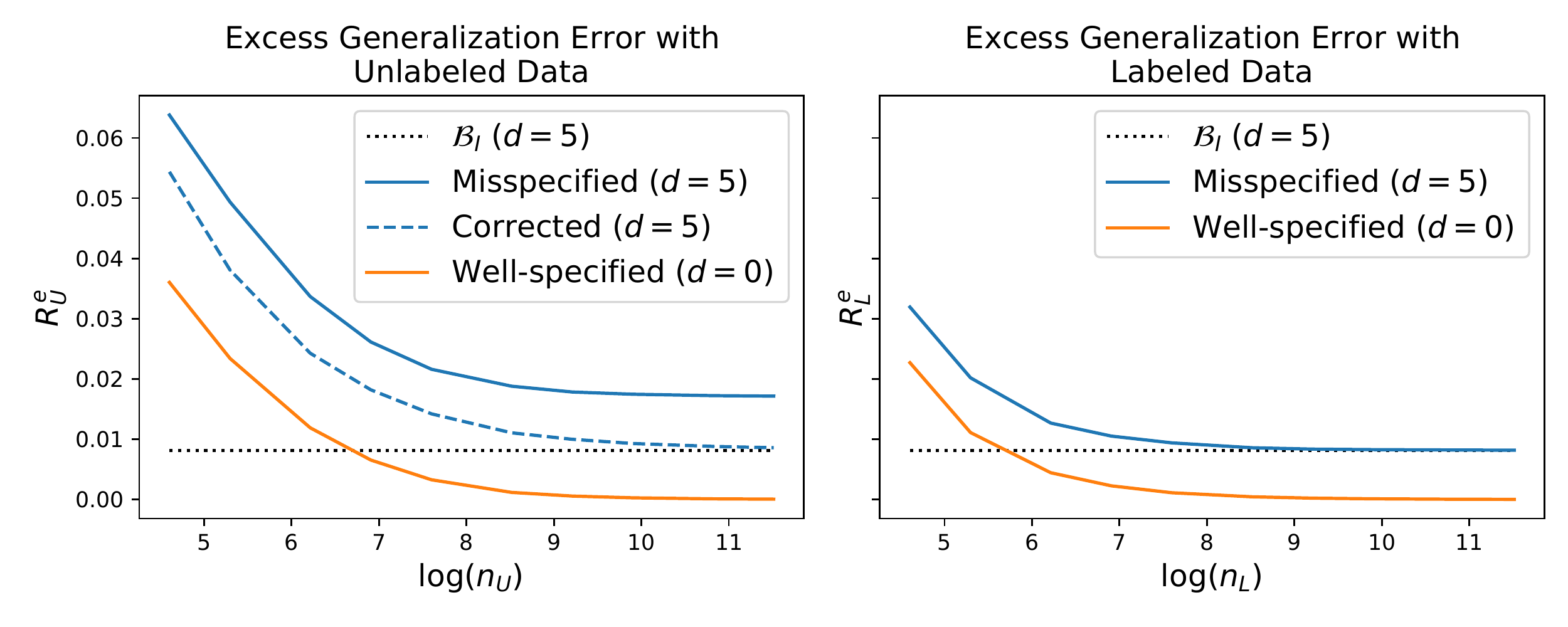}
    \else
    \includegraphics[width=.48\textwidth]{pdf_figures/generalization_error.pdf}
    \fi
    \caption{Excess generalization error vs. $\log(n)$ with different estimators for synthetic data. Left: comparison of unlabeled data performance under the three discussed settings. Right: comparison of labeled data performance for well-specified and misspecified models. A dashed line repesenting an empirical ``$\B_I$'' suggests how inference bias is present in both data cases. % With no misspecification ($d=0$) the labeled and unlabeled estimators both converge to an asymptotic error of zero. Under misspecification ($d=5$), we see that learning from unlabeled data results in an additional $\B_{\mathrm{est}}$ standing bias due to misspecification, and median aggregation corrects for this bias so that error converges to $\B_I$.
    }
    \label{fig:gen_err}
\end{figure}
\section{Applications}\label{sec:applications}
\vspace{-.5em}
Based on our generalization error framework, we now have a rigorous way to analyze misspecification in latent variable models. We examine two practical applications of our theoretical results in three settings---well-specified, misspecified, and corrected:
\begin{itemize}
    % \item \textbf{Correcting for misspecification:} for the weakly labeled data setting, we propose a simple algorithm improving over random selection of triplets $\lf_i, \lf_j, \lf_k$ that removes the error term $\varepsilon_{\max} \left(\frac{c_1 d}{m} + \frac{c_2}{\sqrt{n_U}} + \frac{c_3}{n_U}\right)$ given that $m$ is large and $d$ is $o(m^2)$.
    \item \textbf{Understanding the value of labeled data:} we address our motivating question about the value of labeled data–--is a few labeled samples or many unlabeled samples better? This decision varies per setting, depending on the misspecification parameters ($d$, $\varepsilon_{\max}$), and $n_U$ versus $n_L$.
    \item \textbf{Combining labeled and unlabeled data:} we show how simple linear combinations of the estimators can improve generalization error bounds over using one or the other. We also suggest a James-Stein type estimator from~\cite{GreenStrawderman2001}, which combines an unbiased estimator with biased information, to easily determine the weights of the linear combination. %Based on our misspecification analysis, we show that this estimator offers a \textcolor{red}{X} improvement in generalization error. 
\end{itemize}

We extend our upper bounds on the decomposition in Theorem \ref{thm:decomposition} to these two applications of our framework, presenting theoretical results first and then verifying our results on synthetic data. In Appendix \ref{subsec:supp_lowerbound}, we comment on how lower bounds can be obtained and used for similar analysis as an avenue for future work.
% We discuss each of these two applications of our framework theoretically first and then verify our results on synthetic data.

\subsection{Understanding the value of labeled data}
\vspace{-0.5em}
We use our analysis from Section \ref{subsec:scaling} to develop a criterion for deciding between labeled and unlabeled points. Compute
\[f(n_U) = \min_{n_L \in \mathbb{N}} \text{ s.t. } R_L^{e}(n_L) \le R_U^{e}(n_U),\]
%where $U^{\text{unlabeled}}, U^{\text{labeled}}$ are the upper bounds on the excess generalization error for unlabeled and labeled data, respectively.
and define
%\begin{align}
$V(n_U) = {n_U}/f(n_U)$
%\end{align}
to be the \textit{data value ratio}. The intuitive idea here is to compare, for some amount of unlabeled data $n_U$, what factor less labeled data we would require to produce an equivalent error bound. We consider an approximation of the data value ratio $\widetilde{V}(n_U)$ based on our upper bounds for excess generalization error. We examine the differences in $\widetilde{V}(n_U)$ for our three aforementioned settings:
\begin{itemize}
    \item Well-specified setting: comparing excess risk when $d = 0$ and $\varepsilon_{\max} = 0$ reduces to examining $\frac{m}{2n_L}$ and $\frac{c_4 m}{n_U}$. Thus $\widetilde{V}(n_U) = 2c_4$ and our framework suggests that labeled data is only a \textit{constant factor} more beneficial than unlabeled data.
    \item Misspecified setting: $\widetilde{V}(n_U)$ will capture the tradeoff between $\frac{m}{2n_L}$ and $\B_{\mathrm{est}} + \frac{c_4 m }{n_U}$. We find that $\widetilde{V}(n_U) = 2 \varepsilon_{\max} \Big(\frac{c_1 dn_U}{m} +$ $\frac{c_2 \sqrt{n_U}}{m} + \frac{c_3 d}{m^2} \Big) + 2 c_4$. That is, the value of labeled data \textit{increases linearly in the amount of unlabeled data and misspecification} due to the standing bias in the generalization error for the unlabeled approach.  
    \item Corrected setting: under our conditions from Proposition \ref{prop:medians}, we examine the difference between $\frac{m}{2n_L}$ and $ c_\rho m \rho_{n_
    U}$, and thus $\widetilde{V}(n_U) = 2n_U c_{\rho} \rho_{n_U}$. Since $\rho_{n_U}$ converges to $0$, $\widetilde{V}(n_U)$ is sublinear in $n_U$, showing that the \textit{corrected model increases the relative value of unlabeled data.}
\end{itemize}

\vspace{-0.5em}
\paragraph{Synthetic Experiments} We measure $V(n_U)$ in well-specified, misspecified and corrected settings on synthetic data with the same setup as discussed in \ref{sec:synthetic4}. Our detailed protocol for approximating $V(n_U)$ is in Appendix \ref{subsec:supp_synthetics}.

%We expect three principles to hold: first, for a well-specified model, the data value ratio is small and roughly constant, second, the ratio scales with misspecification and finally, the ratio is smaller for a corrected model than for a misspecified model. \bcw{Get rid of this sentence?} We also expect that our theoretical ratio produces a reasonable estimate of the empirical value, in particular one that tracks the scaling of the misspecification level. 
We present the results in \autoref{fig:data_value_ratio}. In the well-specified case ($d = 0$), $V(n)$ is small (less than $5$) and roughly constant across $n$. Under misspecification however, the data value ratio grows with both $d$ and $n$ albeit much more slowly for the corrected setting, aligning with our theoretical findings.

\begin{figure}
    \centering
    \ifsinglecolumn
    \includegraphics[width=.5\textwidth]{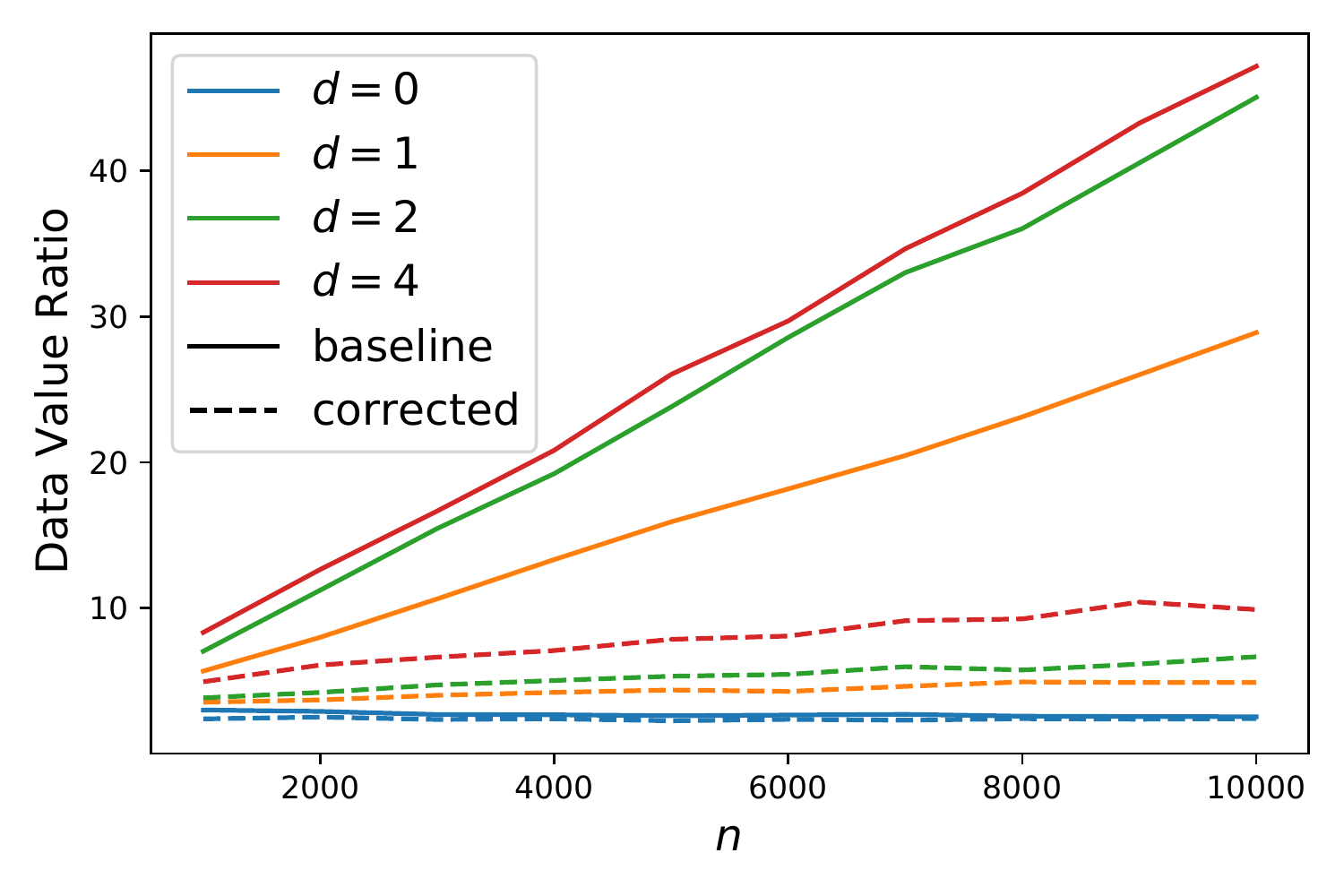}
    \else
    \includegraphics[width=.3\textwidth]{pdf_figures/data_value_ratios.pdf}
    \fi
    \caption{Data value ratio vs. $n$, using both the standard method-of-moments approach and the corrected approach, which aggregates results over triplets using medians. Note that $d=0$ represents the well-specified setting. %We see that in the well-specified case ($d=0$) the data value ratio is small and roughly constant across $n$. Under misspecification, the data value ratio grows with $d$ for both the baseline and corrected models, but grows faster for baseline models, as reflected by our theoretical model.
    }
    \label{fig:data_value_ratio}
\end{figure}

\subsection{Combining labeled and unlabeled data}
\vspace{-0.5em}
While we now have a criterion to choose between datasets, how do we combine information from both? We examine ways to combine the accuracy parameters, namely $\widetilde{a}^U$ as defined in \eqref{eq:triplet} for unlabeled data and an equivalent $\widetilde{a}^L := \Ehat{\bm{\lf} Y}$ for labeled data. Recall that $\widetilde{a}^L$ is unbiased, while $\widetilde{a}^U$ is both biased and inconsistent if not corrected.

First, we consider a simple linear combination, $a^{\mathrm{lin}}(\alpha) = \alpha \widetilde{a}^U + (1 - \alpha) \widetilde{a}^L$ for some weight $\alpha \in [0, 1]$. Using our framework in Theorem \ref{thm:decomposition}, we can derive similar upper bounds on excess generalization error when the estimator is $a^{\mathrm{lin}}(\alpha)$. We summarize our findings across the three settings below, where we consider $\alpha \widetilde{a}^M + (1 - \alpha) \widetilde{a}^L$ for the corrected setting. 
\begin{itemize}
    \item Well-specified setting: the upper bound on excess generalization error using $a^{\mathrm{lin}}$, ignoring $\B_I$ and lower order terms, is $\alpha^2 \frac{c_4 m}{n_U} + (1 - \alpha)^2 \frac{m}{2n_L}$. One can easily verify that there exists an $\alpha \in (0, 1)$ that minimizes this upper bound. Since $n_U$ is usually much larger than $n_L$, plugging in this optimal $\alpha$ shows that this new upper bound is roughly of the same order as the unlabeled case.
    \item Misspecified setting: the upper bound is a cubic polynomial in $\alpha$. We find that the standing bias results in the optimal $\alpha$ weighting the labeled data's estimator more. This suggests that a combined estimator can yield an upper bound much smaller than that for the unlabeled case. 
    \item Corrected setting: the upper bound now consists of $\alpha^2 c_{\rho} m \rho_{n_U} + (1 - \alpha)^2 \frac{m}{2n_L}$. As a function of $\alpha$, this differs from the well-specified setting's expression only in constant coefficients, so this again suggests an optimal $\alpha \in (0, 1)$ and performance roughly similar to the unlabeled case.
\end{itemize}

In practice, we do not know the exact $\alpha$ that optimizes generalization error. However, there is vast literature on combined estimators that dominate the MLE estimator $\widetilde{a}^L$. In particular, we suggest using an approach from~\cite{GreenStrawderman2001}, who propose a way of setting $\alpha$ given knowledge of an unbiased estimator with biased information. % the following combined estimator:

\vspace{-0.5em}
\paragraph{Synthetic Experiments} We investigate the empirical performance of estimators which combine labeled and unlabeled data in well-specified, misspecified and corrected settings. We measure both the error when using the fine-tuned $\alpha$ and the more practical approach of \cite{GreenStrawderman2001}. We fix $n_U=1000$ and vary $n_L$ across a range of smaller values, aligning with the assumption that many more unlabeled than labeled points are typically available. Our results are in Figure \ref{fig:combined}. In the well-specified setting, the combined estimators perform roughly the same as just $\widetilde{a}^U$, matching up with our theoretical observations for large $n_U$. In the misspecified setting, both combined estimators result in lower excess risk than either estimator individually, and as $n_L$ increases, the labeled estimator curve approaches those of the combined estimators, suggesting that the weight on $\widetilde{a}^L$ increases as more labeled data becomes available. Lastly, in the corrected setting both combined estimators perform better than $\widetilde{a}^U$, but not by much.
%suggesting that there is a standing bias present in the former case that makes combining with labeled data more helpful.
%The approach of \cite{GreenStrawderman2001} does not consistently outperform the unlabeled only approach in the well-specified case, where the unlabeled estimator is unbiased. However, it significantly outperforms either individual approach in the misspecified case and outperforms either individual approach in the corrected case. 
The weights $\alpha$ are reported in Appendix \ref{subsec:supp_synthetics}.
The optimal weights for the well-specified and corrected settings are higher (i.e. more weight on the unlabeled estimator) than the misspecified setting, and these weights decrease with $n_L$.

%We expect that in the well-specified case, the combined estimators won't significantly outperform the unlabeled estimator. On the other hand, under misspecification, we expect the combined estimators, which account for the standing bias of unlabeled data, to provide a superior estimate to either individual estimator. Finally, for the corrected case, we expect more modest improvements from combining labeled and unlabeled data, similarly to the well-specified case. As the estimator proposed by \cite{GreenStrawderman1991} assumes that the unlabeled estimator is biased, we expect it to perform the best in the misspecified case compared to the optimal combined estimator. We present the results in \autoref{fig:combined}.

% We further expect the James-Stein-type estimator to perform better under misspecification where the unlabeled estimator is biased and the simple linear combination to perform better for well-specified models. Results in \autoref{fig:combined}.

\begin{figure}
    \centering
    \ifsinglecolumn
    \includegraphics[width=.75\textwidth]{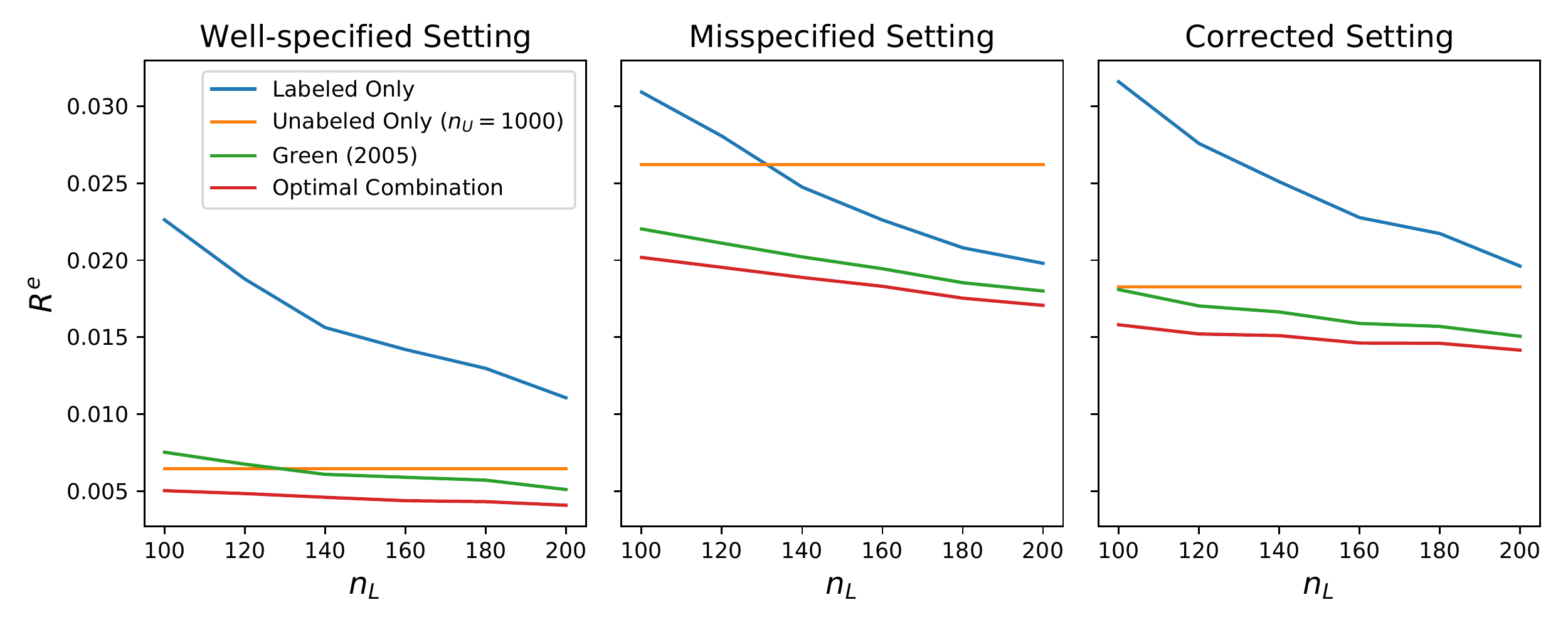}
    \else
    \includegraphics[width=.48\textwidth]{pdf_figures/combined.pdf}
    \fi
    \caption{Excess generalization error for an optimally weighted combination of labeled and unlabeled estimators, and a combination weighted according to \cite{GreenStrawderman2001} across the well-specified (left), misspecified (center), and corrected (right) settings. The number of unlabeled points is fixed at $n_U=1000$. %The approach of \cite{GreenStrawderman2001} does not consistently outperform the unlabeled only approach in the well-specified case, where the unlabeled estimator is unbiased. However, it significantly outperforms either individual approach in the misspecified case and outperforms either individual approach in the corrected case. The combination weights are reported in the \todo{Appendix}. Broadly, the optimal weights for the well-specified and corrected settings are higher than the misspecified setting, and these weights decrease with $n_L$.
    }
    \label{fig:combined}
\end{figure}
\section{Real-World Case Study: Weak Supervision} \label{sec:exp}
\vspace{-.5em}
% We evaluate our theoretical findings and our proposed criterion. We anticipate that the value of data is significantly affected by misspecification, so that labeled data offers an insignificant value versus unlabeled (i.e., a small value ratio) for a well-specified model, but becomes more valuable as the misspecification level increases. We hypothesize that combining both labeled and unlabeled data can offer a superior estimate (as compared to relying one or the other alone). Finally, we expect that the bias due to misspecification is reflected downstream for a variety of models. 

% \subsection{Synthetic Data}

% \paragraph{Protocol}

% Our synthetic data distributions have balanced classes (i.e., $\mathbb{E}[Y]=0$) and $m=10$ labeling functions, with accuracies drawn uniformly from $[0.55, 0.75]$. Dependency strengths are fixed at $\varepsilon_{ij}=0.1$. To measure the expected generalization error of each learning algorithm for a given budget $n$, we perform $1000$ trials and average results, drawing data independently for each trial.

% {\color{red} Combination table goes here. Shows four columns: all labeled, all unlabeled, naive combination, and James-Stein.}

We validate our findings on real-world weak supervision dataset. Unlike our theoretical setting where we limit the number of dependencies $d$ for simplicity, with real-world data we anticipate many small dependencies which cannot be completely corrected by the medians approach. We seek to answer the following key questions.

\begin{itemize}
    \item What is the standing parameter estimation bias due to misspecification? To what extent does the corrected estimator, which only addresses unmodeled source dependencies, mitigate this bias?
    
    \item What is the data value ratio for misspecified and corrected settings? %\steve{I like the other questions; maybe this question isn't necessary though? The standing bias means the labeled data ratio will depend on $n_U$} \bcw{I think I agree with the fact that the plot won't be very informative (it'll just be a line) but might be nice to have just to parallel the theoretical results.}
    
    \item Can a combined estimator with access to a small amount of labeled data provide substantial benefits over using only unlabeled data?
\end{itemize}
\vspace{-0.5em}

\paragraph{Protocol} Our real-world task is the sentiment analysis task of determining whether IMDB movie reviews are positive or negative \citep{maas2011learning}. The dataset contains 50K movie reviews, which we split into a training set of 40K reviews and a test set of 10K reviews. Our weak supervision sources are simple heuristics that vote ``yes'' when positive words appear and ``no'' when negative words appear. We provide further details in Appendix \ref{subsec:supp_realdata}.

% \steve{how did you pick the words? Especially the negative ones seem a little counterintuitive.} \bcw{I sorted words according to predictive power, and picked ones which seemed intuitive. I'm guessing you're referring to ``better'', ``could'' and ``would'' when you talk about the counterintuitive ones. These actually make sense to me, as they're referring to things like ``could have been better''} \steve{Can you add a few words describing the process then? Maybe: ``we chose words that were predictive and intuitive'' or something like that}

Unlike our theoretical model, where we assume that each source has a single accuracy parameter, we find that real-world sources have complex dependencies and can be better modeled with \textit{class conditional} accuracies. The method-of-moments approach in this setting results in a quadratic version of the triplet method \citep{fu2020fast}, the details of which we discuss in Appendix \ref{subsec:supp_mom}.
We use this version for our real-world case study, for which the same principles from our theoretical framework apply.
% We measure model performance in terms of its F1-score on the test set as a proxy for generalization error.
% For each dataset we train models with different label budgets $n_L$. We report the average accuracy on the test set over $1000$ trials (with different points to label selected for each trial). 

\vspace{-0.5em}
\paragraph{Standing bias and correction} For our first real-world experiment, we measure the standing parameter estimation bias when learning from unlabeled data (paralleling $\mathcal{B}_\text{est}$), and measure the decrease in  bias when using a corrected estimator. % Recall that in our model a corrected estimator has an asymptotic estimation bias of $0$ given certain conditions on $d, m,$ and $n$. %when $d<\frac{(m-1)(m-2)}{4}$ for reasonable values of $m$ and $n$. 
% On real data, however, we anticipate many small but complex dependencies that are not necessarily corrected for via our medians approach. 
%; $d$ might even be $m \choose 2$ but with most $\varepsilon_{ij}$ quite small.
% Hence, we anticipate that the corrected estimator reduces $\mathcal{B}_\text{est}$ by mitigating the effects of larger dependencies while still being biased by the many small dependencies. 
We compute the test cross entropy loss for a labeled model, a baseline unlabeled and an unlabeled model with correction while varying $n$ and report results in \autoref{fig:real_biases_data_value_ratio} (left, bottom). Losses appear to converge, with a large gap between the labeled and unlabeled models and a smaller gap between the labeled model and the unlabeled model with correction. These gaps in loss are reflected by gaps in F1-scores, computed using a threshold of $.5$.

\begin{figure}
    \centering
    \ifsinglecolumn
    \begin{subfigure}{.75\textwidth}
      \centering
      \includegraphics[width=\linewidth]{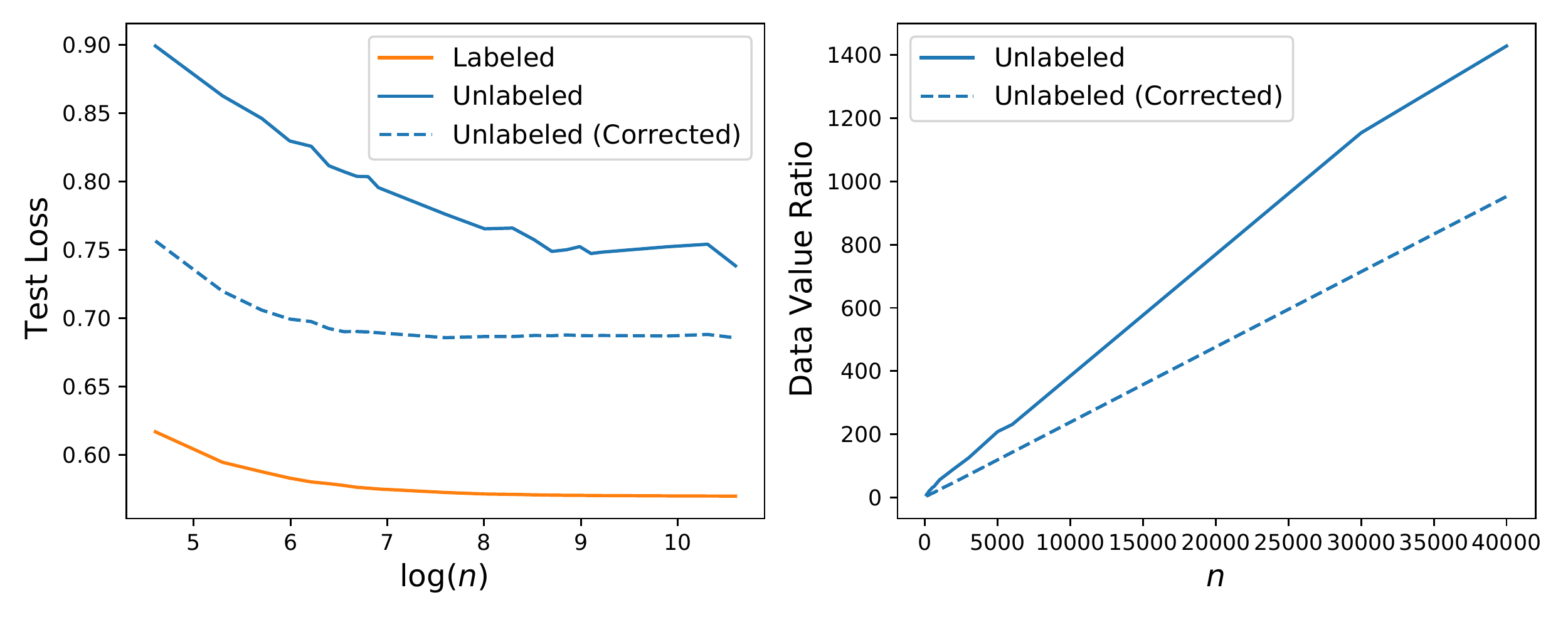}
    \end{subfigure}
    \else
    \begin{subfigure}{.48\textwidth}
      \centering
      \includegraphics[width=\linewidth]{pdf_figures/real_biases_data_value_ratio.pdf}
    \end{subfigure}
    \fi
   \ifsinglecolumn
    \begin{subfigure}{.6\textwidth}
      \small
      \centering
      {\renewcommand{\arraystretch}{1.2}
      \begin{tabular}{lccr}
      \hline
      Model & Loss $(n=40\text{K})$ & F1 $(n=40\text{K})$ \\
      \hline
      Labeled & .570 & 71.79 \\
      Unlabeled & .740 & 64.81 \\
      Corrected & .686 & 68.12 \\
      \hline
      \end{tabular}
      }
    \end{subfigure}   
   \else
    \begin{subfigure}{.48\textwidth}
      \small
      \centering
      {\renewcommand{\arraystretch}{1.2}
      \begin{tabular}{lccr}
      \hline
      Model & Loss $(n=40\text{K})$ & F1 $(n=40\text{K})$ \\
      \hline
      Labeled & .570 & 71.79 \\
      Unlabeled & .740 & 64.81 \\
      Corrected & .686 & 68.12 \\
      \hline
      \end{tabular}
      }
    \end{subfigure}
    \fi
    \caption{We measure test losses and F1-scores for labeled, unlabeled and corrected models on the IMDB dataset. Top Left: losses vs. $n$; each model appears to flatten out by $n=40,000$.  Bottom: losses and F1-scores at $n=40,000$, showing standing gaps in performance. Top Right: data value ratios for the two unlabeled models.}
    \label{fig:real_biases_data_value_ratio}
\end{figure}

\vspace{-0.5em}
\paragraph{Measuring the value of labeled data} Next, we measure the data value ratio in the real-world setting. Since both the unlabeled model and the unlabeled model with correction have a standing bias compared to the labeled model, we anticipate that the data value ratio for both unlabeled approaches grows with $n$, with the data value ratio for the baseline unlabeled model being higher. We report these results in \autoref{fig:real_biases_data_value_ratio} (right).

% \begin{figure}
%     \centering
%     \includegraphics[width=.35\textwidth]{eps_figures/real_data_value_ratio.eps}
%    %
%     \caption{Data value ratio vs. $n$ for unlabeled and corrected models on the IMDB dataset. Both ratios increase linearly with $n$ because of the standing biases of the two unlabeled approaches; however, the ratio for the corrected model increases more slowly because its bias is smaller.}
%     \label{fig:real_data_value_ratio}
% \end{figure}

\vspace{-0.5em}
\paragraph{Combining labeled and unlabeled data} We finally measure the performance of the combined estimator from \cite{GreenStrawderman2001} in the setting where a small number of labeled points and many unlabeled points are available. We let $n_U=40,000$ be the entire training set and vary $n_L$ between $40$ and $400$. We use the corrected estimator for learning from unlabeled data. We report the F1-score using a threshold of $.5$. Results are in \autoref{tab:real_combo}. We observe that the combined estimator outperforms either approach individually for $n_L>40$.

% We use three tasks. First, \textbf{Spam} classification for Youtube comments ~\citep{alberto2015tubespam} has $10$ LFs from~\cite{Ratner19}. Here, the training set has $1586$ examples and the test set has $250$ examples; classes are approximately balanced. 

% In this case, we do not a priori know the level of misspecification. We expect that LFs are neither completely conditionally independent nor so markedly (conditionally) dependent that nothing can be learned from unlabeled data. This leads to data value ratios that are in between the scenarios in between the two extreme synthetic cases we considered above.

\begin{table}[t]
\vskip 0.15in
\renewcommand{\arraystretch}{1.25} % Default value: 1
\begin{center}
\begin{small}
\begin{tabular}{lccccr}
\hline
$n_U$ & $n_L$ & $\text{F1}_\text{Unlabeled}$ & $\text{F1}_\text{Labeled}$ &  $\text{F1}_\text{Combined}$ \\
\hline
40,000 & 40 & 68.12 & 64.70 & 67.06 \\
40,000 & 80 & 68.12 & 67.65 & 68.81 \\
40,000 & 120 & 68.12 & 68.92 & 69.64 \\
40,000 & 200 & 68.12 & 69.97 & 70.41 \\
40,000 & 400 & 68.12 & 70.81 & 71.04 \\
\hline
\end{tabular}
\end{small}
\end{center}
\vskip -0.1in
\caption{F1-scores for unlabeled, labeled and combined approaches on the IMDB dataset. We find that the combination generally outperforms either approach individually, and in particular both in cases where unlabeled only performs better and where labeled only performs better.}
\label{tab:real_combo}
\end{table}

\section{Conclusion}

Motivated by the practical tradeoff between acquiring large unlabeled datasets and small labeled datasets, we introduce a framework that aims to provide theoretically-grounded reasoning for using labeled versus unlabeled data in latent variable graphical models. We present three main technical contributions in this paper: a) a finite-sample decomposition for generalization error with labeled vs unlabeled input, focused on model misspecification; b) a correction approach for method-of-moments to reduce the impact of model misspecification; c) applications of this decomposition framework and correction, namely how to choose and combine the two data types. We show theoretically and validate empirically that labeled data is more valuable when models are misspecified, since learning from unlabeled data relies more heavily on structural assumptions that may be violated. Simple algorithmic corrections, however, can significantly improve the relative value of unlabeled data.

\subsubsection*{Acknowledgments}

We gratefully acknowledge the support of NIH under No. U54EB020405 (Mobilize), NSF under Nos. CCF1763315 (Beyond Sparsity), CCF1563078 (Volume to Velocity), and 1937301 (RTML); ONR under No. N000141712266 (Unifying Weak Supervision); the Moore Foundation, NXP, Xilinx, LETI-CEA, Intel, IBM, Microsoft, NEC, Toshiba, TSMC, ARM, Hitachi, BASF, Accenture, Ericsson, Qualcomm, Analog Devices, the Okawa Foundation, American Family Insurance, Google Cloud, Swiss Re, Total, the HAI-AWS Cloud Credits for Research program, the Stanford Data Science Initiative (SDSI), and members of the Stanford DAWN project: Facebook, Google, and VMWare. The Mobilize Center is a Biomedical Technology Resource Center, funded by the NIH National Institute of Biomedical Imaging and Bioengineering through Grant P41EB027060. The U.S. Government is authorized to reproduce and distribute reprints for Governmental purposes notwithstanding any copyright notation thereon. Any opinions, findings, and conclusions or recommendations expressed in this material are those of the authors and do not necessarily reflect the views, policies, or endorsements, either expressed or implied, of NIH, ONR, or the U.S. Government.

%Use unnumbered third level headings for the acknowledgments. All
%acknowledgments, including those to funding agencies, go at the end of the paper.

\bibliography{bibliography}

\begin{thebibliography}{}

\bibitem[Anandkumar et~al., 2014]{anandkumar2014tensor}
Anandkumar, A., Ge, R., Hsu, D., Kakade, S.~M., and Telgarsky, M. (2014).
\newblock Tensor decompositions for learning latent variable models.
\newblock {\em Journal of Machine Learning Research}, 15:2773--2832.

\bibitem[Anandkumar et~al., 2012]{anandkumar12}
Anandkumar, A., Hsu, D., and Kakade, S.~M. (2012).
\newblock A method of moments for mixture models and hidden markov models.
\newblock volume~23 of {\em Proceedings of Machine Learning Research}, pages
  33.1--33.34, Edinburgh, Scotland. JMLR Workshop and Conference Proceedings.

\bibitem[Bach et~al., 2017]{bach2017learning}
Bach, S.~H., He, B., Ratner, A., and R{\'e}, C. (2017).
\newblock Learning the structure of generative models without labeled data.
\newblock In {\em Proceedings of the 34th International Conference on Machine
  Learning-Volume 70}, pages 273--282. JMLR. org.

\bibitem[Ben-David et~al., 2008]{ben2008does}
Ben-David, S., Lu, T., and P{\'a}l, D. (2008).
\newblock Does unlabeled data provably help? worst-case analysis of the sample
  complexity of semi-supervised learning.
\newblock In {\em COLT}, pages 33--44.

\bibitem[Castelli and Cover, 1996]{castelli1996relative}
Castelli, V. and Cover, T.~M. (1996).
\newblock The relative value of labeled and unlabeled samples in pattern
  recognition with an unknown mixing parameter.
\newblock {\em IEEE Transactions on information theory}, 42(6):2102--2117.

\bibitem[Chaganty and Liang, 2014]{chaganty2014estimating}
Chaganty, A.~T. and Liang, P. (2014).
\newblock Estimating latent-variable graphical models using moments and
  likelihoods.
\newblock In {\em International Conference on Machine Learning}, pages
  1872--1880.

\bibitem[Chandrasekaran et~al., 2012]{Chandrasekaran12}
Chandrasekaran, V., Parrilo, P.~A., and Willsky, A.~S. (2012).
\newblock Latent variable graphical model selection via convex optimization.
\newblock {\em Annals of Statistics}, 40(4):1935--1967.

\bibitem[Chapelle and Scholkopf, 2006]{Chapelle09}
Chapelle, O.and~Zien, A. and Scholkopf, B. (2006).
\newblock {\em Semi-supervised learning}.
\newblock MIT press.

\bibitem[Craven and Kumlien, 1999]{craven:ismb99}
Craven, M. and Kumlien, J. (1999).
\newblock Constructing biological knowledge bases by extracting information
  from text sources.
\newblock In {\em International Conference on Intelligent Systems for Molecular
  Biology (ISMB)}.

\bibitem[Dawid and Skene, 1979]{dawid1979maximum}
Dawid, A.~P. and Skene, A.~M. (1979).
\newblock Maximum likelihood estimation of observer error-rates using the em
  algorithm.
\newblock {\em Applied statistics}, pages 20--28.

\bibitem[De~Blasi and Walker, 2013]{Blasi13}
De~Blasi, P. and Walker, S.~G. (2013).
\newblock Bayesian asymptotics with misspecified models.
\newblock {\em Statistica Sinica}, pages 169--187.

\bibitem[Domingos, 2000]{domingos2000unified}
Domingos, P. (2000).
\newblock A unified bias-variance decomposition.
\newblock In {\em Proceedings of 17th International Conference on Machine
  Learning}, pages 231--238.

\bibitem[Fu et~al., 2020]{fu2020fast}
Fu, D.~Y., Chen, M.~F., Sala, F., Hooper, S.~M., Fatahalian, K., and R{\'e}, C.
  (2020).
\newblock Fast and three-rious: Speeding up weak supervision with triplet
  methods.
\newblock In {\em Proceedings of the 37st International Conference on Machine
  Learning (ICML 2020)}.

\bibitem[Ghorbani and Zou, 2019]{ghorbani2019data}
Ghorbani, A. and Zou, J. (2019).
\newblock Data shapley: Equitable valuation of data for machine learning.

\bibitem[Green et~al., 2005]{GreenStrawderman2001}
Green, E.~J., Strawderman, W.~E., Amateis, R.~L., and Reams, G.~A. (2005).
\newblock {Improved Estimation for Multiple Means with Heterogeneous
  Variances}.
\newblock {\em Forest Science}, 51(1):1--6.

\bibitem[Gr{\"u}nwald et~al., 2017]{Grunwald17}
Gr{\"u}nwald, P., Van~Ommen, T., et~al. (2017).
\newblock Inconsistency of bayesian inference for misspecified linear models,
  and a proposal for repairing it.
\newblock {\em Bayesian Analysis}, 12(4):1069--1103.

\bibitem[Gupta and Manning, 2014]{gupta2014improved}
Gupta, S. and Manning, C.~D. (2014).
\newblock Improved pattern learning for bootstrapped entity extraction.
\newblock In {\em Proceedings of the Eighteenth Conference on Computational
  Natural Language Learning}, pages 98--108.

\bibitem[Hsu et~al., 2012]{Hsu12}
Hsu, D., Kakade, S.~M., and Liang, P. (2012).
\newblock Identifiability and unmixing of latent parse trees.
\newblock In {\em Advances in Neural Information Processing Systems, (NIPS
  2012)}.

\bibitem[Jia et~al., 2019]{jia2019towards}
Jia, R., Dao, D., Wang, B., Hubis, F.~A., Hynes, N., G{\"u}rel, N.~M., Li, B.,
  Zhang, C., Song, D., and Spanos, C.~J. (2019).
\newblock Towards efficient data valuation based on the shapley value.
\newblock In {\em The 22nd International Conference on Artificial Intelligence
  and Statistics}, pages 1167--1176. PMLR.

\bibitem[Jog and Loh, 2015]{JogL15}
Jog, V. and Loh, P. (2015).
\newblock On model misspecification and {KL} separation for gaussian graphical
  models.
\newblock {\em CoRR}, abs/1501.02320.

\bibitem[Joglekar et~al., 2013]{joglekar2013evaluating}
Joglekar, M., Garcia-Molina, H., and Parameswaran, A. (2013).
\newblock Evaluating the crowd with confidence.
\newblock In {\em Proceedings of the 19th ACM SIGKDD international conference
  on Knowledge discovery and data mining}, pages 686--694.

\bibitem[Karger et~al., 2011]{karger2011iterative}
Karger, D.~R., Oh, S., and Shah, D. (2011).
\newblock Iterative learning for reliable crowdsourcing systems.
\newblock In {\em Advances in neural information processing systems}, pages
  1953--1961.

\bibitem[Kleijn and van~der Vaart, 2012]{kleijn2012}
Kleijn, B. and van~der Vaart, A. (2012).
\newblock The bernstein-von-mises theorem under misspecification.
\newblock {\em Electron. J. Statist.}, 6:354--381.

\bibitem[Kleijn and van~der Vaart, 2006]{kleijn2006}
Kleijn, B. J.~K. and van~der Vaart, A.~W. (2006).
\newblock Misspecification in infinite-dimensional bayesian statistics.
\newblock {\em Ann. Statist.}, 34(2):837--877.

\bibitem[Lauritzen, 1996]{Lauritzen}
Lauritzen, S. (1996).
\newblock {\em Graphical Models}.
\newblock Clarendon Press.

\bibitem[Loh and Wainwright, 2013]{Loh13}
Loh, P.-L. and Wainwright, M.~J. (2013).
\newblock Structure estimation for discrete graphical models: Generalized
  covariance matrices and their inverses.
\newblock {\em Annals of Statistics}, 41(6):3022--3049.

\bibitem[Maas et~al., 2011]{maas2011learning}
Maas, A., Daly, R.~E., Pham, P.~T., Huang, D., Ng, A.~Y., and Potts, C. (2011).
\newblock Learning word vectors for sentiment analysis.
\newblock In {\em Proceedings of the 49th annual meeting of the association for
  computational linguistics: Human language technologies}, pages 142--150.

\bibitem[Meng et~al., 2014]{Meng14}
Meng, Z., Eriksson, B., and {III}, A. O.~H. (2014).
\newblock Learning latent variable gaussian graphical models.
\newblock In {\em Proceedings of the 31st International Conference on Machine
  Learning (ICML 2014)}, Beijing, China.

\bibitem[Mintz et~al., 2009]{mintz2009distant}
Mintz, M., Bills, S., Snow, R., and Jurafsky, D. (2009).
\newblock Distant supervision for relation extraction without labeled data.
\newblock In {\em Proceedings of the Joint Conference of the 47th Annual
  Meeting of the ACL and the 4th International Joint Conference on Natural
  Language Processing of the AFNLP: Volume 2-Volume 2}, pages 1003--1011.
  Association for Computational Linguistics.

\bibitem[Ratner et~al., 2019]{Ratner19}
Ratner, A., Hancock, B., Dunnmon, J., Sala, F., Pandey, S., and R{\'e}, C.
  (2019).
\newblock Training complex models with multi-task weak supervision.
\newblock In {\em Proceedings of the AAAI Conference on Artificial
  Intelligence}, volume~33, pages 4763--4771.

\bibitem[Ratner et~al., 2016]{Ratner16}
Ratner, A.~J., De~Sa, C.~M., Wu, S., Selsam, D., and R{\'e}, C. (2016).
\newblock Data programming: Creating large training sets, quickly.
\newblock In {\em Advances in neural information processing systems}, pages
  3567--3575.

\bibitem[Ravikumar et~al., 2011]{Ravikumar11}
Ravikumar, P., Wainwright, M.~J., Raskutti, G., and Yu, B. (2011).
\newblock High-dimensional covariance estimation by minimizing
  $\ell_1$-penalized log-determinant divergence.
\newblock {\em Electronic Journal of Statistics}, 5:935--980.

\bibitem[Singh et~al., 2008]{singh2008unlabeled}
Singh, A., Nowak, R., and Zhu, J. (2008).
\newblock Unlabeled data: Now it helps, now it doesn't.
\newblock {\em Advances in neural information processing systems},
  21:1513--1520.

\bibitem[Takamatsu et~al., 2012]{takamatsu:acl12}
Takamatsu, S., Sato, I., and Nakagawa, H. (2012).
\newblock Reducing wrong labels in distant supervision for relation extraction.
\newblock In {\em Meeting of the Association for Computational Linguistics
  (ACL)}.

\bibitem[Varma et~al., 2019]{varma2019learning}
Varma, P., Sala, F., He, A., Ratner, A., and R{\'e}, C. (2019).
\newblock Learning dependency structures for weak supervision models.
\newblock {\em arXiv preprint arXiv:1903.05844}.

\bibitem[Yang and Priebe, 2011]{yang2011effect}
Yang, T. and Priebe, C.~E. (2011).
\newblock The effect of model misspecification on semi-supervised
  classification.
\newblock {\em IEEE transactions on pattern analysis and machine intelligence},
  33(10):2093--2103.

\bibitem[Yang et~al., 2020]{yang2020rethinking}
Yang, Z., Yu, Y., You, C., Steinhardt, J., and Ma, Y. (2020).
\newblock Rethinking bias-variance trade-off for generalization of neural
  networks.

\bibitem[Zhu and Goldberg, 2009]{zhu2009introduction}
Zhu, X. and Goldberg, A.~B. (2009).
\newblock Introduction to semi-supervised learning.
\newblock {\em Synthesis lectures on artificial intelligence and machine
  learning}, 3(1):1--130.

\end{thebibliography}

\clearpage
\appendix

\onecolumn

\aistatstitle{Supplementary Materials}
\section{Glossary}
\label{sec:gloss}

The glossary is given in Table~\ref{table:glossary} below.
\begin{table*}[h]
\centering
\small
\begin{tabular}{l l}
\toprule
Symbol & Used for \\
\midrule
$X$ & An input vector $X\in\mathcal{X}$.\\
$Y$ & A latent ground-truth label $Y\in\mathcal{Y}=\{-1,1\}$.\\
$m$ & Number of sources.\\
$\lf_j$ & $j$th source output $\lf_j: \mathcal{X} \rightarrow \mathcal{Y}$; all $m$ labels make up vector $\bm{\lf}$. \\
%$\bm{\lambda}$ & A vector of source outputs $\bm{\lambda}\in\mathcal{Y}^m$.\\
$\widetilde{Y}$ & Soft label in $[-1, 1]$ output by the latent variable model.\\
$n_U$ & Number of unlabeled samples.\\
$n_L$ & Number of labeled samples.\\
$\theta$ & Canonical parameters of the Ising model for $\Pr(Y,\bm{\lambda})$.\\
$G$ & Dependency graph $G=(V,E)$ over sources and the latent ground-truth label.\\
$E_\lambda$ & Edges among sources in $G$.\\
$d$ & Number of dependencies among sources, $d=|E_\lambda|$.\\
$a_i$ & True accuracy of the $i$th source $\mathbb{E}[\lambda_iY]$.\\
$\widetilde{a}_i^U$ & Estimated accuracy of the $i$th source using unlabeled data via the triplet method.\\
$\widetilde{a}_i^L$ & Estimated accuracy of the $i$th source using labeled data, i.e. $\Ehat{\lf_i Y}$.\\
$\widetilde{a}_i^M$ & Estimated accuracy of the $i$th source using unlabeled data via the\\
& triplet method and median aggregation.\\
$\mathcal{N}$ & Random variable representing dataset used.\\
$\tau$ & Algorithmic randomness for estimating accuracies via triplet method.\\
$R,R_U,R_L,R_M$ & Generalization error $R = \mathbb{E}_{(Y, \bm{\lf}), \N, \tau}[l(\widetilde{Y}, Y)]$. $R_U,R_L,R_M$ are for $\widetilde{a}_i^U,\widetilde{a}_i^L,\widetilde{a}_i^M$, respectively, \\
& and $l(\cdot, \cdot)$ is the cross-entropy loss.\\
% $R_U$ & Generalization error when using the triplet method.\\
% $R_L$ & Generalization error when estimating accuracies empirically.\\
% $R_M$ & Generalization error when using the triplet method with median aggregation.\\
$R^e,R^e_U,R^e_L,R^e_M$ & Excess generalization error $R^e=R-H(Y|\bm{\lambda})$.\\
$\mathcal{B}_I$ & Inference bias $\mathcal{B}_I=\sum_{(i, j) \in E_\lf}I(\lf_i; \lf_j | Y)$.\\
$\mathcal{B}_\text{est}$ & Parameter estimation error.\\
$\varepsilon_{ij}$ & Extent of misspecification on a single pair of sources $\varepsilon_{ij} = \E{}{\lf_i \lf_j} - \E{}{\lf_i Y} \E{}{\lf_j Y}$.\\
$\varepsilon_{\min},\varepsilon_{\max}$ & Smallest and largest $\varepsilon_{ij}$ for $(i,j)\in E_\lf$.\\
$\rho_{n_U}$ & Mean squared error for $\tilde{a}_i^M$, $\rho_{n_U} = \max_i \E{}{(\widetilde{a}_i^M - a_i)^2}$.\\
$f(n_U)$ & Minimum labeled points needed for lower generalization error than $n_U$ unlabeled points.\\
$V(n_U)$ & Data value ratio at $n_U$ unlabeled points.\\
$\widetilde{V}(n_U)$ & Approximation of data value ratio using upper bounds at $n_U$ unlabeled points.\\
$\alpha$ & Weight for unlabeled estimator to combine unlabeled and labeled estimators.\\
$a^{\text{lin}}(\alpha)$ & Linear combination of unlabeled and labeled estimators using weight $\alpha$. \\
\toprule
\end{tabular}
\caption{
	Glossary of variables and symbols used in this paper.
}
\label{table:glossary}
\end{table*}

\vfill

\section{Additional Algorithmic Details} \label{sec:alg}

We provide more details on our algorithm for latent variable estimation. The input is either a labeled dataset $(\bm{X}_L, \bm{Y}_L)$ or unlabeled dataset $\bm{X}_U$ with $m$ sources $\bm{\lf}$. The output is an estimate of the distribution $\Pr(Y | \bm{\lf}(X))$, which we construct using the factorization in \eqref{eq:inference}. For both data types, this requires plugging in the values of $\Pr(Y = 1)$ and the empirical distribution of the sources, $\Phat(\bm{\lf} = \bm{\lf}(X))$.

The approach to estimating $\Ptilde(\lf_i = \lf_i(X) | Y = 1)$ is the only part of the method that differs between the labeled and unlabeled settings. For both, we can focus on estimating $\E{}{\lf_i Y}$ since $\Ptilde(\lf_i = \pm 1 | Y = 1) = \frac{1 \pm \widetilde{a}_i}{2}$ by Lemma \ref{lemma:fs_symmetry}. 
In the labeled setting, the expectation can be estimated directly, i.e. $\Ehat{\lf_i Y} = \frac{1}{n_L}\sum_{j = 1}^n \lf_i(x_j) y_j$. On the other hand, for unlabeled data we use the triplet method from \cite{fu2020fast}, described in Algorithm \ref{alg:triplet}, to estimate $\E{}{\lf_i Y}$. This algorithm takes as input the pairwise rates of agreement between sources $\Ehat{\lf_i \lf_j}$ for all $i, j$, and returns an estimate of each $\E{}{\lf_i Y}$.

The \textsc{Aggregate} subroutine in Algorithm \ref{alg:triplet} distinguishes between the unlabeled case with and without correction. For unlabeled data, we theoretically analyze the approach where we choose $\widetilde{a}_i \sim \mathrm{Unif}(A)$; that is, we randomly select two $\lf_j, \lf_k$ to compute $\widetilde{a}_i^U$, which is similarly done in other method-of-moments approaches. An alternate to this approach is to take the \textit{mean} over all possible pairs $\lf_j, \lf_k$; note that this reduces the estimation error compared to the population-level estimate by a factor of ${m-1 \choose 2}$, but does not mitigate bias from misspecification. We use this approach in our synthetic and real-world experiments for the baseline unlabeled case without correction. Lastly, having \textsc{Aggregate}$(A)$ be the median of the set $A$ is our proposed method of correcting for misspecification.

%In the baseline model the \textsc{Aggregate} subroutine chooses a value at random or computes the mean value, while in the corrected model it chooses the median value.

\begin{algorithm}[t]
\caption{Method-of-Moments Latent Variable Estimation~\citep{fu2020fast}}
\begin{algorithmic}
\STATE{\textbf{Input:} Empirical expectation estimates $\Ehat{\lf_i \lf_j}$}
\FOR{$i=1$ \TO $m$}
    \STATE $A=\emptyset$
    \FOR{$j,k \in \{1, \dots, m\} \backslash \{i\}$ }
        \STATE{$\widetilde{a}_i^{(j,k)} \gets \sqrt{| \Ehat{\lf_i \lf_j} \cdot \Ehat{\lf_i \lf_k} / \Ehat{\lf_j \lf_k}|}$}
        \STATE $A \gets A \cup \widetilde{a}_i^{(j,k)}$
    \ENDFOR
    \STATE $\widetilde{a}_i \gets \textsc{Aggregate}(A)$
\ENDFOR
\RETURN $\widetilde{a}$, estimates of $\E{}{\lf_i Y}$ for all $\lf_i$
\end{algorithmic}
\label{alg:triplet}
\end{algorithm}

\section{Additional Theoretical Results}

In Section \ref{subsec:supp_mom}, we discuss how our generalization error bounds, namely the standing $\mathcal{O}(d/m)$ bias for unlabeled data, and our results for the corrected medians estimator can still apply to other method-of-moments estimators that exploit conditionally independent views of hidden variables. Next, in Section \ref{subsec:combined} we give more details about the combined estimators and the generalization bounds from using them. Finally, in Section \ref{subsec:supp_lowerbound} we present a lower asymptotic bound on the generalization error for labeled versus unlabeled data and combining both.

\subsection{Other Method-of-Moments Estimators}\label{subsec:supp_mom}

We present two other method-of-moments estimators and sketch out arguments for how using them (under misspecification) results in the same scaling of generalization error, and for how the median approach is able to help correct standing bias. We then provide an abstracted argument. 

\paragraph{``Quadratic'' Triplets} 
This alternative latent variable model relies on class-conditional probability terms instead of mean parameters \citep{fu2020fast}, which assume some symmetries in the distribution (see Lemma \ref{lemma:fs_symmetry}). For the $i$th source, we can write the parameters to be estimated as
\[\mu_i =
  \begin{bmatrix}
    \Pr(\lf_i = 1|Y = 1) & \Pr(\lf_i = 1|Y = -1)  \\
    \Pr(\lf_i = -1|Y = 1) & \Pr(\lf_i = -1|Y = -1)
  \end{bmatrix}.\] 

Let
\[O_{ij} =
  \begin{bmatrix}
    \Pr(\lf_i = 1, \lf_j = 1) & \Pr(\lf_i = 1,\lf_j = -1)  \\
    \Pr(\lf_i = -1,\lf_j = 1) & \Pr(\lf_i = -1,\lf_j = -1)
  \end{bmatrix}
  \text{    and    }
  P =
  \begin{bmatrix}
    \Pr(Y=1) & 0  \\
    0 & \Pr(Y=-1)
  \end{bmatrix}.
  \]

Then, we obtain that 
\begin{align}
O_{ij} = \mu_i P \mu_j^\top.
\label{eq:newparam}
\end{align}
The left-hand side is observable, and we can form triplets again to solve for each $\mu_i$. 
Set $\alpha = P(\lf_i =1 |Y=1)$, $c_i = \frac{P(\lf_i = 1)}{P(Y=-1)} $ and $d_i = \frac{P(Y=1)}{P(Y=-1)}$. The top row of $\mu_i$ is then $[\alpha \quad c_i - d_i \alpha]$ with $c_i$ and $d_i$ known. For a triplet $i,j,k$, and the appropriate $\mu$'s, using the $\alpha, \beta, \gamma$ notation above and corresponding $c_i, c_j, c_k$ and $d_i, d_j, d_k$ terms, we obtain the system (see \cite{fu2020fast} for more details)
\begin{align*}
(1+d_i d_j) \alpha \beta + c_i c_j - c_i d_j \beta - c_j d_i \alpha &= O_{ij}/\Pr(Y=1), \\
(1+d_i d_k)\alpha \gamma + c_i c_k - c_i d_k \gamma - c_k d_i\alpha &= O_{ik}/\Pr(Y=1), \\
(1+d_j d_k)\beta \gamma + c_j c_k - c_j d_k \gamma - c_k d_j \beta &= O_{jk}/\Pr(Y=1). \\
\end{align*}
To solve, $\alpha$ and $\gamma$ are expressed with $\beta$ for the first and third equations and this is plugged into the second---yielding a quadratic equation to be solved.

This approach incurs standing bias under misspecification. Quadratic triplets rely on conditional independence by assuming that $\Pr(\lf_i = 1, \lf_j = 1)$ and $\Pr(\lf_i = 1 | Y = 1) \Pr(\lf_j = 1 | Y = 1) \Pr(Y = 1) + \Pr(\lf_i = 1 | Y = -1) \Pr(\lf_j = 1 | Y = -1) \Pr(Y = -1)$ are equal. Suppose, however, that $(i, j) \in E_{\lf}$. Then, $\mu_i P \mu_j^\top$ is no longer equal to $O_{ij}$, but $O_{ij} + \delta_{ij}$, where $\delta_{ij} = \Pr(Y = 1)[\Pr(\lf_i | Y = 1) \Pr(\lf_j | Y = 1) - \Pr(\lf_i, \lf_j | Y = 1)] + \Pr(Y = -1)[\Pr(\lf_i | Y = -1) \Pr(\lf_j | Y = -1) - \Pr(\lf_i, \lf_j | Y = -1)]$. This $\delta_{ij}$ can be written exactly in terms of the canonical parameters $\theta$ and results in an inconsistent estimator of $\Pr(\lf_i | Y)$. We note that the probability of selecting a bad triplet that leads to this is the same for this method and our main triplet method, so the standing bias still scales $\mathcal{O}(\frac{d \delta}{m})$.

This approach can also be corrected using medians and the same conditions from Proposition \ref{prop:medians}, which we prove in Section \ref{subsec:medians}, hold for the estimates to be consistent. %Out of ${m - 1 \choose 2}$ triplets used in estimating $\Pr(\lf_i | Y)$, there are ${m - 1 \choose 2 } - m - d - 3$ triplets that result in a consistent estimate. So as long as ${m - 1 \choose 2} - m - d - 3 > \frac{1}{2} \cdot {m - 1 \choose 2}$ and $n_U$ is sufficiently large, using medians will result in a corrected estimator. See the proof of Proposition \ref{prop:medians} in section \ref{subsec:medians} for more details. 

\paragraph{Method-of-moments for topic exchange} \cite{anandkumar2014tensor} describes tensor method-of-moments estimators for a variety of applications, including topic models. In the topic model case, $h$ is the topic latent variable, $x_1, \ldots, x_{\ell}$ are the words in the document, all assumed to be conditionally independent given $h$ and drawn from an unknown conditional probability distribution $\mu_h$ parametrized by the latent topic variable. Here, $x_t = e_i$, the standard basis vector if the $t$th word is $i$. \cite{anandkumar2014tensor} uses the fact that 
\[\mathbb{E}[x_1 \otimes x_2 \otimes x_3] = \sum_{i=1}^k w_i \mu_i \otimes \mu_i \otimes \mu_i,\]
where $w_i$ is the probability of $h$ being topic $i$, to perform a tensor decomposition of the observable  $\mathbb{E}[x_1 \otimes x_2 \otimes x_3]$ and learn $\mu_h$. Note the similarity to our setting, where $Y$ is used in place of $h$ and where there are two (i.e., a matrix) instead of three views (giving a tensor). Conditional independence (of words given the topic) is required to for this expression to hold. Therefore, when conditional independence is violated, $\sum_{i = 1}^k w_i \mu_i \otimes \mu_i \otimes \mu_i$ is equal to $\E{}{x_1 \otimes x_2 \otimes x_3}$ plus some additional perturbation that is a function of the probability distribution. This error is propagated into the estimate of $\mu_h$, assuming Lipschitzness of this estimator. Furthermore, assuming random triples are selected to learn the accuracy of each word, using this approach to estimate accuracy parameters will again yield a standing bias.

Furthermore, the medians approach can again correct for this standing bias---there are ${m - 1 \choose 2} - m - d - 3$ good triplets out of ${m - 1 \choose 2}$, so we require the same conditions to yield consistent estimators as those for the quadratic triplets case.

\paragraph{Abstraction} Consider in general some observable quantities $o_1, \ldots, o_v$, some unobservable quantities $u_1, \ldots, u_v$ that depend on the value of some latent variable $h$, and a relationship that holds when some set of dependencies $\Omega$ is taken into account, 
\[ f(o_1, \ldots, o_v) = g_{\Omega}(u_1, \ldots, u_v),\]
Next, we call $s(f(o_1, \ldots, o_v))$ an estimator that produces estimates of $u_1, \ldots, u_v$. 

Our approach is simply to account for errors due to accessing an incorrect $\Omega'$, where $|\Omega \setminus \Omega'| = d$. Then, 
\[ f(o_1, \ldots, o_v) = g_{\Omega'}(u_1, \ldots, u_v) + d \times \Delta(u_1, \ldots, u_v),\]
where $\Delta$ is some error term. Given this setup, we then propagate the error term $\Delta$ in the estimator $s$, computing
$s(f(o_1, \ldots, o_v)) - s(f(o_1, \ldots, o_v) - d \Delta(u_1, \ldots, u_v)$. This can be done either via perturbation analysis or Taylor approximation or other methods---the only requirement we place is Lipschitzness on the estimator $s$. Then, by randomly selecting subsets of $(o_1, \ldots, o_v)$ to estimate $u_1, \ldots, u_v$, the probability of picking a subset with error scales in $d$, showing that there exists a standing bias that is a function of the number of unmodeled dependencies. Moreover, there are some subsets of $(o_1, \ldots, o_v)$ that yield consistent estimators $s$; if this quantity is greater than half of all the subsets, then a medians approach can be beneficial when there is enough data.

\subsection{Combined estimator analysis} \label{subsec:combined}

The general form of the combined estimator we consider is $a^{\mathrm{lin}}(\alpha) = \alpha \widetilde{a}^U + (1 - \alpha) \widetilde{a}^L$ for some weight $\alpha \in [0, 1]$. The James-Stein type estimator from \cite{GreenStrawderman2001}, which we evaluate empirically, uses the following:
\begin{align}
    \widetilde{a}^G := \widetilde{a}^U + \bigg(1 - \frac{r}{\|\widetilde{a}^L - \widetilde{a}^U\|_{\Sigma^{-1}}} \bigg)_{\mathclap{+}} (\widetilde{a}^L - \widetilde{a}^U),
\end{align}
where $\Sigma = \Cov{}{\widetilde{a}^L}$ and $r \in [0, 2(m - 2)]$. \cite{GreenStrawderman2001} show that this estimator dominates $\widetilde{a}^L$ when the unbiased estimator is Gaussian and its covariance is known, but since we can only estimate the covariance matrix, we replace $\Sigma$ with an empirical estimate $\hat{\Sigma}$ in practice. This estimator is equivalent to $a^{\mathrm{lin}}\Big(\min \Big\{\frac{r}{\|\widetilde{a}^L - \widetilde{a}^U\|_{\hat{\Sigma}^{-1}}}, 1 \Big\} \Big)$.  We thus focus on analyzing the performance of the general combined estimator $a^{\mathrm{lin}}(\alpha)$. 

The change in estimator only impacts the generalization bound via the parameter estimation error,  $\sum_{i = 1}^m \E{\N, \tau, Y}{\KL (\mathrm{Pr}_{\lf_i | Y} || \Ptilde_{\lf_i | Y} )}$. We simplify this using Lemma \ref{lemma:KL_estimation}, doing a Taylor approximation on a combined asymptotic estimate $\bar{a}_i^C := \alpha \bar{a}_i + (1 - \alpha) a_i$ rather than $\bar{a}_i$. This gives us
\begin{align}
     &\sum_{i = 1}^m \E{\N, \tau, Y}{\KL (\mathrm{Pr}_{\lf_i | Y} || \Ptilde_{\lf_i | Y} )} = \sum_{i = 1}^m \frac{1 + a_i}{2} \log \Big(1 + \frac{\alpha(a_i - \bar{a}_i)}{1 + \bar{a}_i^C} \Big) + \frac{1 - a_i}{2} \log \Big(1 + \frac{\alpha(\bar{a}_i - a_i)}{1 - \bar{a}_i^C} \Big) \label{eq:combined_param_err} \\
    +& \sum_{i = 1}^m \frac{a_i - \bar{a}_i}{1 - (\bar{a}_i^C)^2} \alpha^2 \E{}{\bar{a}_i - \widetilde{a}_i^U}  + \sum_{i = 1}^m \frac{1}{2} \Big( \frac{1}{1 - (\bar{a}_i^C)^2} + \frac{2\alpha (\bar{a}_i - a_i)}{(1 - (\bar{a}_i^C)^2)^2}\Big) \Big(\alpha^2 \E{}{(\widetilde{a}_i^U - \bar{a}_i)^2} + (1 - \alpha)^2 \E{}{(\widetilde{a}_i^L - a_i)^2} \Big) \nonumber 
\end{align}

We present bounds for the three settings discussed in the paper.

\paragraph{Well-specified setting} In the well-specified setting, the unlabeled data accuracy estimator is consistent, so $\bar{a}_i = a_i$, and therefore
\begin{align}
     &\sum_{i = 1}^m \E{\N, \tau, Y}{\KL (\mathrm{Pr}_{\lf_i | Y} || \Ptilde_{\lf_i | Y} )} =  \sum_{i = 1}^m \frac{1}{2} \Big( \frac{1}{1 - a_i^2} \Big) \Big(\alpha^2 \E{}{(\widetilde{a}_i^U - \bar{a}_i)^2} + (1 - \alpha)^2 \E{}{(\widetilde{a}_i^L - a_i)^2} \Big)    
\end{align}

Using the results of the proof of Theorem \ref{thm:labeled} and the bound on $\E{}{(\widetilde{a}_i^U - \bar{a}_i)^2}$ in Lemma \ref{lemma:sampling_error}, we get that this is at most $\alpha^2 \frac{c_4 m }{n_U} + (1 - \alpha)^2 \frac{m}{2n_L}$.

\paragraph{Misspecified Setting} The constant terms for the bound on accuracy parameter estimation error will change due to $\bar{a}_i^C$ in the denominator rather than $\bar{a}_i$, but the derivation follows our proof for Theorem \ref{thm:unlabeled}. Therefore, for some $c'$,
\begin{align}
    \sum_{i = 1}^m \E{\N, \tau, Y}{\KL (\mathrm{Pr}_{\lf_i | Y} || \Ptilde_{\lf_i | Y} )} \le & \varepsilon_{\max}\Big(\frac{c'_1 \alpha d }{m} + \frac{c'_2 \alpha^2 }{\sqrt{n_U}} + \frac{c'_3 \alpha^3 d}{m n_U} + \frac{\alpha (1 - \alpha)^2 c'_5 d}{m n_L} \Big) + \frac{c'_4 \alpha^2 m}{n_U} + \frac{(1 - \alpha)^2 m}{2n_L}. \nonumber
\end{align}

\paragraph{Corrected Setting} Here we consider the combined estimator $\alpha \widetilde{a}^M + (1 - \alpha)\widetilde{a}^L$. Under certain conditions, we know that $\widetilde{a}^M$ asymptotically converges to $a$. Therefore, the accuracy parameter estimation error is
\begin{align}
     &\sum_{i = 1}^m \E{\N, \tau, Y}{\KL (\mathrm{Pr}_{\lf_i | Y} || \Ptilde_{\lf_i | Y} )} =  \sum_{i = 1}^m \frac{1}{2} \Big( \frac{1}{1 - a_i^2} \Big) \Big(\alpha^2 \E{}{(\widetilde{a}_i^M - \bar{a}_i)^2} + (1 - \alpha)^2 \E{}{(\widetilde{a}_i^L - a_i)^2} \Big)
\end{align}

$\E{}{(\widetilde{a}_i^M - \bar{a}_i)^2}$ is just the variance of the median estimator. Therefore, this summation is bounded by $\alpha^2 c_{\rho} m \rho_{n_U} + (1 - \alpha)^2 \frac{m}{2n_L}$ under the conditions in Proposition \ref{prop:medians}.

\subsection{Lower bounds on generalization error} \label{subsec:supp_lowerbound}

While Theorems \ref{thm:labeled} and \ref{thm:unlabeled} provide upper bounds on the excess generalization error, it is also important to consider lower bounds---is the standing bias from misspecification in the unlabeled approach inevitable? We analyze the asymptotic excess risk in the case of labeled data, unlabeled data, and both, and discuss how a lower bound approach to the data value ratio and analyzing combined estimators is possible.

\paragraph{Unlabeled data lower bound} Looking at the decomposition in Theorem \ref{thm:decomposition}, $\E{\N}{\KL(\Pr(\bm{\lf}) || \Phat(\bm{\lf}))}$ approaches $0$ asymptotically and the inference bias $\sum_{(i, j) \in E_{\lf}} I(\lf_i; \lf_j | Y)$ is independent of the amount of data. We thus seek to asymptotically lower bound $\sum_{i = 1}^m \E{\N, \tau, Y}{\KL( \Pr_{\lf_i | Y} || \Ptilde_{\lf_i | Y})}$. Note that in the labeled data case and when using the medians estimator $\widetilde{a}^M$ with unlabeled data, parameter estimation error approaches $0$ as $n$ grows large since the estimated accuracy parameters are consistent. In the unlabeled data case, we show that standing bias persists.

\begin{theorem}
Suppose that there are $|E_{\lf}| = d$ unmodeled dependencies. When we use the latent variable model described in section \ref{sec:background}, the lower bound of the excess generalization error is asymptotically bounded by
\begin{align}
    \lim_{n_U \rightarrow \infty} R_u^e \ge \frac{(m - 2d) d^2 \varepsilon_{\min}^2 b_{\min}^4}{2(m - 1)^2 (m - 2)^2} + \B_I.
\end{align}

When $d$ is $o(m)$, the asymptotic parameter estimation error is $\Omega \Big(\frac{d^2 \varepsilon_{\min}^2}{m^3} \Big)$.
\label{thm:lower_bound}
\end{theorem}

\begin{proof}
We compute an asymptotic lower bound for $\sum_{i = 1}^m \E{\N, \tau, Y}{\KL( \Pr_{\lf_i | Y} || \Ptilde_{\lf_i | Y})}$. Applying Lemma \ref{lemma:KL_estimation}, we see that 
\begin{align}
    \lim_{n_U \rightarrow \infty} \sum_{i = 1}^m \E{\N, \tau, Y}{\KL( \mathrm{Pr}_{\lf_i | Y} || \Ptilde_{\lf_i | Y})} = \sum_{i = 1}^m \frac{1 + a_i}{2} \log \left(1 + \frac{a_i - \bar{a}_i}{1 + \bar{a}_i}\right) + \frac{1 - a_i}{2} \log \left(1 + \frac{\bar{a}_i - a_i}{1 - \bar{a}_i}\right). \label{eq:lower_bound_est_err}
\end{align}

We focus on the lower bound of any one element of this sum. For ease of notation, let $a := a_i$ and $x = a_i - \bar{a}_i$. Then this expression for an arbitrary $i$ becomes
\begin{align}
    \frac{1 + a_i}{2} \log \left(1 + \frac{a_i - \bar{a}_i}{1 + \bar{a}_i}\right) + \frac{1 - a_i}{2} \log \left(1 + \frac{\bar{a}_i - a_i}{1 - \bar{a}_i}\right) = - \frac{1 + a}{2} \log \left(1 - \frac{x}{1 + a}\right) - \frac{1 - a}{2} \log \left(1 + \frac{x}{1 - a}\right). \label{eq:per_i}
\end{align}

Take the negative of this expression and define it as a function $f(x)$ to upper bound:
\begin{align}
    f(x) = \frac{1 + a}{2} \log \left(1 - \frac{x}{1 + a}\right) + \frac{1 - a}{2} \log \left(1 + \frac{x}{1 - a}\right).
\end{align}

We show that $f(x) \le - \frac{1}{2}x^2$. Note that for $x = 0$, $f(x) = 0$ and $\frac{1}{2}x^2 = 0$. Then, we must show that for $x \ge 0$, $f'(x) \le -x$ and for $x < 0$, $f'(x) > -x$. Taking the derivative of $f(x)$ gives us $f'(x) = \frac{-x}{1 - (a - x)^2}$, and it is clear that the previous inequalities are satisfied.

Using this fact in \eqref{eq:lower_bound_est_err}, we have that $
     \lim_{n_U \rightarrow \infty} \sum_{i = 1}^m \E{\N, \tau, Y}{\KL( \mathrm{Pr}_{\lf_i | Y} || \Ptilde_{\lf_i | Y})} = \sum_{i = 1}^m \frac{1}{2} (a_i - \bar{a}_i)^2
$. For $i \in E_{\lf}$, note that by Lemma \ref{lemma:accuracy_bias} it is possible to construct a graphical model such that $a_i - \bar{a}_i = 0$. For $i \notin E_{\lf}$, we know that $|a_i - \bar{a}_i|$ is at least $\frac{d \varepsilon_{\min} b_{\min}^2}{(m - 1)(m - 2)}$. Therefore,
\begin{align}
    \frac{1}{2} \sum_{i = 1}^m (a_i - \bar{a}_i)^2 \ge \frac{1}{2} \sum_{i \notin E_{\lf}} (a_i - \bar{a}_i)^2 \ge \frac{(m - 2d) d^2 \varepsilon_{\min}^2 b_{\min}^4}{2(m-1)^2 (m-2)^2}.
\end{align}
\end{proof}

\paragraph{Combined estimator lower bound} Next, we analyze the excess risk when we use the combined estimator $a^{\mathrm{lin}}(\alpha)$. Note that when we are in the well-specified and corrected settings, the asymptotic excess risk is $0$. Therefore, we only consider the misspecified setting.
\begin{corollary}
Denote $R_{\mathrm{lin}}^e(\alpha)$ as the excess risk of our latent variable model when we use accuracy parameter $a^{\mathrm{lin}}(\alpha)$. The lower bound of the excess generalization error when we combine labeled and unlabeled data (without correction) using weight $\alpha$ is asymptotically bounded by
\begin{align}
    \lim_{n_U, n_L \rightarrow \infty} R_{\mathrm{lin}}^e(\alpha) \ge \frac{\alpha^2 (m - 2d) d^2 \varepsilon_{\min}^2 b_{\min}^4}{2(m - 1)^2 (m - 2)^2} + \B_I.
\end{align}
\label{cor:combined}
\end{corollary}

\begin{proof}

Based on \eqref{eq:combined_param_err}, the asymptotic parameter estimation error is 
\begin{align}
    \lim_{n_U, n_L \rightarrow \infty} \sum_{i = 1}^m \E{\N, \tau, Y}{\KL (\text{Pr}_{\lf_i | Y} || \Ptilde_{\lf_i | Y})} = \sum_{i = 1}^m \frac{1 + a_i}{2} \log \Big(1 + \frac{\alpha(a_i - \bar{a}_i)}{1 + \bar{a}_i^C} \Big) + \frac{1 - a_i}{2} \log \Big(1 + \frac{\alpha(\bar{a}_i - a_i)}{1 - \bar{a}_i^C} \Big),
\end{align}

where $\bar{a}_i^C = \alpha \bar{a}_i + (1 - \alpha) a_i$. If we define $a := a_i$ and $x := \alpha(a_i - \bar{a}_i)$, then the $i$th element of this sum has the form in \eqref{eq:per_i} and is thus at least $\frac{1}{2} x^2$. Therefore, using results from Lemma \ref{lemma:accuracy_bias} the parameter estimation error is
\begin{align}
    \lim_{n_U, n_L \rightarrow \infty} \sum_{i = 1}^m \E{\N, \tau, Y}{\KL (\text{Pr}_{\lf_i | Y} || \Ptilde_{\lf_i | Y})} \ge \sum_{i = 1}^m \frac{\alpha^2}{2} (a_i - \bar{a}_i)^2 \ge \frac{\alpha^2 (m - 2d) d^2 \varepsilon_{\min}^2 b_{\min}^4}{2(m - 1)^2 (m - 2)^2}.
\end{align}

\end{proof}

\paragraph{Applications to data value ratio and combined estimator analysis} Finally, it is possible to define the data value ratio and analyze combined estimators based on lower bounds on the excess risk of labeled vs unlabeled data. To do this, we would use the expressions from Theorem \ref{thm:lower_bound} and Corollary \ref{cor:combined} with standard finite-sample lower bounds on the estimates from observable data. For bounding the variance of accuracy parameters estimated via the triplet method on unlabeled data, we can use the lower bound from Theorem 2 of~\citet{fu2020fast}.
%\section{Analysis of Other Method-of-Moments Estimators}

\section{Proofs}

%We provide technical proofs of our theoretical contributions.
First, we formally state our assumptions on the graphical model that are needed for our results.

\begin{assumption}
Suppose that the distribution of $\Pr(Y, \bm{\lf})$ takes on the form
\begin{align}
    \Pr(Y, \bm{\lf}; \theta) = \frac{1}{Z} \exp \Big(\theta_Y + \sum_{i = 1}^m \theta_i \lf_i Y + \sum_{(i, j) \in E_{\lf}} \theta_{ij} \lf_i \lf_j \Big),
    \label{eq:pgm}
\end{align}

where $Z$ is the cumulant function, and the set of all canonical parameters $\theta$ are positive. This assumption also means that $\E{}{\lf_i \lf_j}, \E{}{\lf_i Y} > 0$ for all $i$ and $j$. Define $a_{\min} = \min_i a_i$ as the minimum true accuracy. Define $b_{\min} = \min_{i,j} \{\E{}{\lf_i \lf_j}, \Ehat{\lf_i \lf_j} \}$. Lastly, define $\bar{a}_{\max} = \max_{i} \bar{a}_i = \max_{i,j,k} \E{\tau}{\sqrt{\frac{\E{}{\lf_i \lf_j} \E{}{\lf_i \lf_k}}{\E{}{\lf_j \lf_k}}}}$.
\label{assumptions}
\end{assumption}

\subsection{Proof of Theorem 1}

Our goal is to evaluate $\E{(Y, \bm{\lf}), \N, \tau}{l(\widetilde{Y}, Y)}$, where $\N$ is the randomness over a sample of $n$ points (either $n_U$ or $n_L$). This expected cross entropy loss can be written as
\begin{align}
    \E{(Y, \bm{\lf}), \N, \tau}{l(\widetilde{Y}, Y)} &= - \E{(Y, \bm{\lf}), \N, \tau}{ \log \frac{\Ptilde(Y' = Y | \bm{\lf}' = \bm{\lf})}{\Pr(Y' = Y | \bm{\lf}' = \bm{\lf})}} + H(Y | \bm{\lf}), \label{eq:cross-entropy}
\end{align}

where $Y', Y$ and $\bm{\lf}', \bm{\lf}$ are independent copies, and the conditional entropy $H(Y | \bm{\lf})$ is by definition
%we use the law of total expectation to write this expectation as
%\begin{align}
%&\E{\bm{\lf}, \N}{\E{Y}{-\left(\frac{1 + Y}{2}\right) \log \Ptilde(Y = 1 | \bm{\lf}' = \bm{\lf}) - \left(\frac{1 - Y}{2}\right) \log \Ptilde(Y = -1 | \bm{\lf}' = \bm{\lf}) \bigg| \bm{\lf}, \N}} \\
%&= - \E{\bm{\lf}, \N}{ \Pr(Y = 1 | \bm{\lf}' = \bm{\lf}) \log \Ptilde(Y = 1 | \bm{\lf}' = \bm{\lf}) + \Pr(Y = -1 | \bm{\lf}' = \bm{\lf}) \log \Ptilde(Y = -1 | \bm{\lf}' = \bm{\lf}) } \\
%&=  - \E{\bm{\lf}, \N}{ \Pr(Y = 1 | \bm{\lf}' = \bm{\lf}) \log \frac{\Ptilde(Y = 1 | \bm{\lf}' = \bm{\lf})}{\Pr(Y = 1 | \bm{\lf}' = \bm{\lf})} + \Pr(Y = -1 | \bm{\lf}' = \bm{\lf}) \log \frac{\Ptilde(Y = -1 | \bm{\lf}' = \bm{\lf})}{\Pr(Y = -1 | \bm{\lf}' = \bm{\lf})} } + H(Y | \bm{\lf}), \label{eq:tower}
%\end{align}
%where the conditional entropy $H(Y | \bm{\lf})$ is by definition
\begin{align}
    H(Y | \bm{\lf}) &= \E{\bm{\lf}}{- \Pr(Y = 1 | \bm{\lf}' = \bm{\lf}) \log \Pr(Y = 1 | \bm{\lf}' = \bm{\lf}) - \Pr(Y = -1 | \bm{\lf}' = \bm{\lf}) \log \Pr(Y = 1 | \bm{\lf}' = \bm{\lf})}.
\end{align}

Next, we evaluate $ \log \frac{\Ptilde(Y' = Y | \bm{\lf}' = \bm{\lf})}{\Pr(Y = 1 | \bm{\lf}' = \bm{\lf})}$. Define $\Pbar$ to be the conditionally independent label model parametrized by the true accuracies $a = \E{}{\bm{\lf} Y}$ in the asymptotic regime; similar to $\Ptilde$'s definition in \eqref{eq:inference}, 
\begin{align}
    \Pbar(Y' = Y | \bm{\lf} =  \bm{\lf}(X)) &= \frac{\Pbar(\bm{\lf} = \bm{\lf}(X) | Y' = Y) \Pr(Y' = Y)}{\Pr(\lf = \lf(X))} = \frac{\prod_{i = 1}^m \Pr(\lf_i = \lf_i(X) | Y' = Y) \Pr(Y' = Y)}{\Pr(\lf = \lf(X))}.
\end{align}

Then, 
\begin{align*}
 \log \frac{\Ptilde(Y' = Y | \bm{\lf}' = \bm{\lf})}{\Pr(Y' = Y | \bm{\lf}' = \bm{\lf})} &= \log \frac{\Ptilde(Y' = Y | \bm{\lf}' = \bm{\lf})}{\Pbar(Y' = Y | \bm{\lf}' = \bm{\lf})} + \log \frac{\Pbar(Y' = Y | \bm{\lf}' = \bm{\lf})}{\Pr(Y' = Y | \bm{\lf}' = \bm{\lf})} \\
&= \sum_{i = 1}^m \log \frac{\Ptilde(\lf'_i = \lf_i | Y' = Y)}{\Pr( \lf'_i = \lf_i | Y' = Y)} + \log \frac{\Pr(\bm{\lf}' = \bm{\lf})}{\Phat(\bm{\lf}' = \bm{\lf})} + \log \frac{\Pbar(\bm{\lf}' = \bm{\lf} | Y' = Y)}{\Pr(\bm{\lf}' = \bm{\lf} | Y' = Y)}.
\end{align*}

We have used the fact that the class balance $\Pr(Y' = Y)$ is the same value across the true distribution, $\Ptilde$, and $\Pbar$. Plugging back into \eqref{eq:cross-entropy}, we get 
\begin{align}
    -\sum_{i = 1}^m \E{(Y, \bm{\lf}), \N, \tau}{\log \frac{\Ptilde(\lf_i' = \lf_i | Y' = Y)}{\Pr(\lf_i' = \lf_i | Y' = Y)}} - \E{(Y, \bm{\lf})}{\log \frac{\Pbar(\bm{\lf}' = \bm{\lf} | Y' = Y)}{\Pr(\bm{\lf}' = \bm{\lf} | Y' = Y)}} - \E{\bm{\lf}, \N}{\log \frac{\Pr(\bm{\lf}' = \bm{\lf})}{\Phat(\bm{\lf}' = \bm{\lf})}} + H(Y | \bm{\lf}). \label{eq:cross-entropy-2}
\end{align}

We simplify each expectation now. 

\begin{enumerate}
    \item $ -\sum_{i = 1}^m \E{(Y, \bm{\lf}), \N, \tau}{\log \frac{\Ptilde(\lf_i' = \lf_i | Y' = Y)}{\Pr(\lf_i' = \lf_i | Y' = Y)}}$:
    
    By definition of conditional KL divergence,
    \begin{align}
        &-\sum_{i = 1}^m \E{(Y, \bm{\lf}), \N, \tau}{\log \frac{\Ptilde(\lf_i' = \lf_i | Y' = Y)}{\Pr(\lf_i' = \lf_i | Y' = Y)}} = \sum_{i = 1}^m \E{(Y, \bm{\lf}), \N, \tau}{ \log \frac{\Pr(\lf_i' = \lf_i | Y' = Y )}{\Ptilde(\lf_i' = \lf_i | Y' = Y)}} \nonumber \\
        &=\sum_{i = 1}^m \E{\N, \tau}{\E{Y}{\KL(\mathrm{Pr}_{\lf_i | Y} || \Ptilde_{\lf_i | Y})}}. \nonumber 
    \end{align}
    
    \item $- \E{(Y, \bm{\lf})}{\log \frac{\Pbar(\bm{\lf}' = \bm{\lf} | Y' = Y)}{\Pr(\bm{\lf}' = \bm{\lf} | Y' = Y)}}$:

    The key difference between $\Pbar$ and $\Pr$ is how the models factorize. The above expression can be written as
    \begin{align*}
    &-\sum_{(i, j) \in E_{\lf}} \E{\lf_i \lf_j, Y}{\log \frac{\Pr(\lf_i' = \lf_i | Y' = Y) \Pr(\lf_j' = \lf_j | Y' = Y)}{\Pr(\lf_i', \lf_j' = \lf_i, \lf_j | Y' = Y)}} \\
    =& \sum_{(i, j) \in E_{\lf}} \E{\lf_i, \lf_j}{ \log \frac{\Pr(\lf_i', \lf_j' = \lf_i, \lf_j | Y = 1)}{\Pr(\lf_i' = \lf_i | Y = 1) \Pr(\lf_j' = \lf_j | Y = 1)} \bigg| \; Y = 1} \Pr(Y = 1)  \\
    &+ \E{\lf_i, \lf_j}{ \log \frac{\Pr(\lf_i', \lf_j' = \lf_i, \lf_j | Y = -1)}{\Pr(\lf_i' = \lf_i | Y = -1) \Pr(\lf_j' = \lf_j | Y = -1)} \bigg| \; Y = -1} \Pr(Y = -1).
    \end{align*}
    
    Note that these expectations are equal to the mutual information between $\lf_i$ and $\lf_j$ conditional on $Y = 1$ or $Y = -1$. Then by definition, the expression is equal to
    \begin{align*}
    \sum_{(i, j) \in E_{\lf}} I(\lf_i; \lf_j | Y = 1) \Pr(Y = 1) + I(\lf_i; \lf_j | Y = -1) \Pr(Y = -1) = \sum_{(i, j) \in E_{\lf}} I(\lf_i; \lf_j | Y).
    \end{align*}

    \item $-\E{\bm{\lf}, \N}{\log \frac{\Pr(\bm{\lf}' = \bm{\lf})}{\Phat(\bm{\lf}' = \bm{\lf})}}$: 
    
    This term is the expected negative KL divergence between the true and estimated distributions of $\bm{\lf}$, $\E{\N}{\KL(\Pr(\bm{\lf}) || \Phat(\bm{\lf}))}$. While there are many ways to estimate this distribution, we stick with simply the MLE estimate so that this expression will converge to $0$ asymptotically. 
\end{enumerate}

Therefore, \eqref{eq:cross-entropy-2} becomes
\begin{align}
   H(Y | \bm{\lf}) - \E{\N}{\KL (\Pr(\bm{\lf}) || \Phat(\bm{\lf}))} + \sum_{(i, j) \in E_{\lf}} I(\lf_i; \lf_j | Y) + \sum_{i = 1}^m \E{\N, \tau, Y}{\KL (\mathrm{Pr}_{\lf_i | Y} || \Ptilde_{\lf_i | Y} )}. \nonumber 
\end{align}
%\todo{bound mutual information}

\subsection{Proof of Theorem 2} \label{app:pf_labeled}

Our goal is to evaluate $\sum_{i = 1}^m \E{\N, \tau, Y}{\KL (\mathrm{Pr}_{\lf_i | Y} || \Ptilde_{\lf_i | Y} )}$ on a labeled dataset. Using Lemma \ref{lemma:KL_estimation}, note that $\E{}{\widetilde{a}_i^L} = \bar{a}_i = a_i$. Therefore,
\begin{align*}
   \E{\N, \tau, Y}{\KL (\mathrm{Pr}_{\lf_i | Y} || \Ptilde_{\lf_i | Y} )} &= \frac{1 + a_i}{2} \cdot \frac{1}{2(1 + a_i)^2} \E{}{(\widetilde{a}_i^L - a_i)^2} + \frac{1 - a_i}{2} \cdot \frac{1}{2(1 - a_i)^2} \E{}{(\widetilde{a}_i^L - a_i)^2} + o(1/n) \\
   &=  \frac{1}{2(1 - a_i^2)} \Var{}{\widetilde{a}_i^L} + o(1/n).
\end{align*}

It can be shown that this is exactly $\frac{1}{2n_L}$. To see this, formally define $\widetilde{a}_i^L = \frac{1}{n_L} \sum_{j = 1}^{n_L} \lf_i^j Y^j$, where $\lf_i^j, Y^j$ belong the $j$th sample of the dataset. Then $
    \Var{}{\widetilde{a}_i^L} = \frac{1}{n_L^2} \sum_{j = 1}^{n_L} \Var{}{\lf_i^j Y^j} = \frac{1}{n_L^2} \sum_{j = 1}^{n_L} \E{}{\lf_i^{j2} Y^{j2}} - \E{}{\lf_i Y}^2 = \frac{1 -a_i^2}{n_L}$.
Therefore, $\sum_{i = 1}^m \E{\N, \tau, Y}{\KL (\mathrm{Pr}_{\lf_i | Y} || \Ptilde_{\lf_i | Y} )} = \frac{m}{2n_L} + o(1/n_L)$, and our proof is complete.

\subsection{Proof of Theorem 3} \label{subsec:supp_unlabeledthm}

We restate the full theorem with the value of the constants. Under assumption \ref{assumptions}, using $n_U$ weakly labeled samples and a misspecified model yields excess generalization error
\begin{align}
    R_U^{e} \le  & \varepsilon_{\max} \left(\frac{c_1 d}{m} + \frac{c_2}{\sqrt{n_U}} + \frac{c_3 d}{m n_U}\right) +\frac{c_4 m}{n_U} + \sum_{(i, j) \in E_{\lf}}I(\lf_i; \lf_j | Y) + o(1/n_U), \nonumber
\end{align}
where 
\begin{align*}
    c_1 &=  \frac{2}{b_{\min}^2 a_{\min}^2} \left(1 + \frac{1}{(1 - \bar{a}_{\max}^2) b_{\min}^2 a_{\min}^2} \right) \\
    c_2 &= \frac{1}{(1 - \bar{a}_{\max}^2) b_{\min}^2 a_{\min}^2} \sqrt{\frac{3(1 - b_{\min}^2)}{b_{\min}^2} \left(\frac{1}{b_{\min}^4} + \frac{2}{b_{\min}^2} \right)} \\
    c_3 &= \frac{3(1 - b_{\min}^2)}{(1 - \bar{a}_{\max}^2)^2 b_{\min}^4 a_{\min}^2} \left(\frac{1}{b_{\min}^4} + \frac{2}{b_{\min}^2} \right) \\
    c_4 &= \frac{3(1 - b_{\min}^2)}{8b_{\min}^2 (1 - \bar{a}_{\max}^2)} \left(\frac{1}{b_{\min}^4} + \frac{2}{b_{\min}^2} \right),
\end{align*}

and $\varepsilon_{\max}$ is an upper bound on $\varepsilon_{ij}$ defined in Lemma \ref{lemma:varepsilon}.

Define $\bar{a}_i = \E{\tau}{\sqrt{\frac{\E{}{\lf_i \lf_j} \E{}{\lf_i \lf_k}}{\E{}{\lf_j \lf_k}}}}$ to be the asymptotic estimator with expectation over triplets. We apply Lemma \ref{lemma:KL_estimation} and simplify it to get
\begin{align}
    \sum_{i = 1}^m \E{\N, \tau, Y}{\KL (\mathrm{Pr}_{\lf_i | Y} || \Ptilde_{\lf_i | Y} )} &= \sum_{i = 1}^m \Big(\frac{1 + a_i}{2} \log \Big(1 + \frac{a_i - \bar{a}_i}{1 + \bar{a}_i}\Big) + \frac{1 - a_i}{2} \log \Big(1 + \frac{\bar{a}_i - a_i}{1 - \bar{a}_i} \Big)\Big)  \label{eq:acc_decomposition}\\
    &+ \sum_{i = 1}^m \frac{a_i - \bar{a}_i}{1 - \bar{a}_i^2} \E{\N, \tau}{\bar{a}_i - \widetilde{a}_i } 
    + \sum_{i = 1}^m \frac{1}{2}\Big( \frac{1}{1 - \bar{a}_i^2} + \frac{2\bar{a}_i(\bar{a}_i - a_i)}{(1 - \bar{a}_i^2)^2}\Big) \E{\N, \tau}{(\widetilde{a}_i - \bar{a}_i)^2} \nonumber \\
    &+ o(1/n).\nonumber 
\end{align}

This shows that there are three quantities to bound: $a_i - \bar{a}_i$, $\E{\N, \tau}{\bar{a}_i - \widetilde{a}_i}$, and $\E{\N, \tau}{(\widetilde{a}_i - \bar{a}_i)^2}$. Recall that for the unlabeled data case, $\widetilde{a}_i = \sqrt{\frac{\Ehat{\lf_i \lf_j} \Ehat{\lf_i \lf_k}}{\Ehat{\lf_j \lf_k}}}$ for random $\lf_j, \lf_k$, and $\bar{a}_i = \E{\tau}{\sqrt{\frac{\E{}{\lf_i \lf_j} \E{}{\lf_i \lf_k}}{\E{}{\lf_j \lf_k}}}}$. The bounds for $\E{\N, \tau}{\bar{a}_i - \widetilde{a}_i}$, and $\E{\N, \tau}{(\widetilde{a}_i - \bar{a}_i)^2}$ are stated in Lemma \ref{lemma:sampling_error}; we focus on bounding the expected asymptotic gap $a_i - \bar{a}_i$ here. 

\begin{lemma}
For $i \in E_{\lf}$, we have that
\begin{align}
    \bar{a}_i - a_i \in \bigg[\frac{\varepsilon_{\min} b_{\min}}{m - 1} - \frac{(d - 1)\varepsilon_{\max}}{(m - 1)(m - 2) b_{\min}^2 a_{\min}^2}, \; \frac{\varepsilon_{\max}}{(m - 1) b_{\min}a_{\min}} \bigg].
\end{align}

For $i \notin E_{\lf}$, we have that
\begin{align}
    \bar{a}_i - a_i \in \bigg[\frac{-d \varepsilon_{\max}}{(m - 1)(m - 2) b_{\min}^2 a_{\min}^2}, \; \frac{-d \varepsilon_{\min} b_{\min}^2}{(m - 1)(m - 2)} \bigg].
\end{align}

\label{lemma:accuracy_bias}
\end{lemma}

And for all $i$, it is thus true that
\begin{align}
    |\bar{a}_i - a_i| \le \frac{\varepsilon_{\max}}{(m - 1)b_{\min}^2 a_{\min}^2}.
\end{align}

\begin{proof}
We define $\varepsilon_{ij} = \E{}{\lf_i \lf_j} - \E{}{\lf_i Y} \E{}{\lf_j Y}$ for $(i, j) \in E_{\lf}$, i.e. the error we get from assuming conditional independence between $\lf_i$ and $\lf_j$. We define the exact value of $\varepsilon_{ij}$ in Lemma \ref{lemma:varepsilon}, and since all canonical parameters are assumed to be positive, we know that there exist $\varepsilon_{\min}, \varepsilon_{\max}$ that satisfy $0 < \varepsilon_{\min} \le \varepsilon_{ij} \le \varepsilon_{\max}$ over the entire edgeset $E_{\lf}$. We now propagate this error to $\bar{a}_i$. Define $\bar{a}_i^{(j, k)}$ before we take the expectation over triplets as
\begin{align*}
\bar{a}_i^{(j, k)} := \sqrt{\frac{\E{}{\lf_i \lf_j} \E{}{\lf_i \lf_k}}{\E{}{\lf_j \lf_k}}}.
\end{align*}

Note that this means $\bar{a}_i \ge b_{\min}$. When each $\E{}{\lf_i \lf_j}$ can be written as $\E{}{\lf_i Y} \E{}{\lf_j Y}$, we get that $\bar{a}_i^{(j, k)} = a_i$. However, by our assumptions on the edgeset, at most one of the above pairwise expectations has nonzero $\varepsilon_{ij}$, in which case the true $a_i$ is computed using $\E{}{\lf_i \lf_j} - \varepsilon_{ij}$, which is equal to $\E{}{\lf_i Y} \E{}{\lf_j Y}$, rather than $\E{}{\lf_i \lf_j}$.

If $(i, j) \in E_{\lf}$ (but not $(j, k)$ or $(i, k)$) then
\begin{align*}
a_i = \sqrt{\frac{(\E{}{\lf_i \lf_j} - \varepsilon_{ij}) \E{}{\lf_i \lf_k}}{\E{}{\lf_j \lf_k}}}.
\end{align*}

This means that $\bar{a}_i \ge a_i$ and we asymptotically overestimate the accuracy. Then the difference  between $\bar{a}_i^{(j,k) 2}$ and $a_i^2$ is
$\bar{a}_i^{(j, k) 2} - a_i^2 = \frac{\varepsilon_{ij} \E{}{\lf_i \lf_k}}{\E{}{\lf_j \lf_k}} \in \big[\varepsilon_{\min} b_{\min}, \frac{\varepsilon_{\max}}{b_{\min}}\big]$. Moreover, $\bar{a}_i^{(j, k)} - a_i = \frac{\bar{a}_i^{(j, k) 2} - a_i^2}{\bar{a}_i^{(j, k)} + a_i}$. Since $\bar{a}_i \ge a_i$ in this case, we have that $\bar{a}_i^{(j, k)} + a_i \in [2a_{\min}, 2]$; as a result,
\begin{align}
\bar{a}_i^{(j, k)} - a_i \in \big[ \frac{\varepsilon_{\min} b_{\min}}{2}, \frac{\varepsilon_{\max}}{2 b_{\min} a_{\min}}\big]. \label{eq:acc_diff_ij}
\end{align}

Similarly, if $(i,k) \in E_{\lf}$, we have the same bounds: $\bar{a}_i^{(j, k)2} - a_i^2 = \frac{\varepsilon_{ik} \E{}{\lf_i \lf_j}}{\E{}{\lf_j \lf_k}} \in \big[\varepsilon_{\min} b_{\min}, \frac{\varepsilon_{\max}}{b_{\min}}\big]$, and thus $\bar{a}_i^{(j, k)} - a_i \in \big[ \frac{\varepsilon_{\min} b_{\min}}{2}, \frac{\varepsilon_{\max}}{2 b_{\min} a_{\min}}\big]$. On the other hand, if $(j, k) \in E_{\lf}$, the true accuracy is written as
\begin{align*}
a_i = \sqrt{\frac{\E{}{\lf_i \lf_j} \E{}{\lf_i \lf_k}}{(\E{}{\lf_j \lf_k}- \varepsilon_{jk})} }.
\end{align*}

This means that $\bar{a}_i^{(j, k)} \le a_i$ and we asymptotically underestimate the accuracy. The difference between $\bar{a}_i^{(j, k) 2}$ and $a_i^2$ is $a_i^2 - \bar{a}_i^{(j, k) 2}  = \frac{\varepsilon_{jk} \E{}{\lf_i \lf_j} \E{}{\lf_i \lf_k}}{\E{}{\lf_j \lf_k}(\E{}{\lf_j \lf_k} - \varepsilon_{jk})} \in \big[\varepsilon_{\min} b_{\min}^2, \frac{\varepsilon_{\max}}{b_{\min}a_{\min}^2} \big]$. In this case, $a_i + \bar{a}_i^{(j, k)} \in [2b_{\min}, 2]$, so 
\begin{align}
a_i - \bar{a}_i^{(j, k)} \in \Big[\frac{\varepsilon_{\min} b_{\min}^2}{2}, \frac{\varepsilon_{\max}}{2b_{\min}^2 a_{\min}^2} \Big].
\label{eq:acc_diff_jk}
\end{align}

Lastly, if none of $i, j, k$ share edges, $\bar{a}_i = a_i$. In our algorithm, we estimate each $a_i$ using $\lf_j$ and $\lf_k$ chosen uniformly at random from the other $m - 1$ sources. We thus need to compute the probabilities that $(i, j), (i, k)$ and $(j, k)$ are in $E_{\lf}$. Note that these probabilities depend on if $i \in E_{\lf}$, which is true for $2d$ sources. 
\begin{align}
    &\Pr((i, j) \cup (i, k) \in E_{\lf} \; | \; i \notin E_{\lf}) = 0 \qquad \qquad \; \Pr((i, j) \cup (i, k) \in E_{\lf} \; | \; i \in E_{\lf}) = \frac{1(m - 2)}{{m - 1 \choose 2}} =  \frac{2}{m - 1} \nonumber \\
    &\Pr((j, k) \in E_{\lf} \;| \; i \notin E_{\lf}) = \frac{2d}{(m - 1)(m - 2)}  \quad \Pr((j, k) \in E_{\lf} \;| \; i \in E_{\lf}) = \frac{2(d - 1)}{(m - 1)(m - 2)} \nonumber
\end{align} 

Therefore, if $i \in E_{\lf}$, we use \eqref{eq:acc_diff_ij} and \eqref{eq:acc_diff_jk} to bound the expected error as
\begin{align}
\bar{a}_i - a_i &\le \frac{2}{m - 1} \cdot \frac{\varepsilon_{\max}}{2b_{\min} a_{\min}} + \frac{2(d - 1)}{(m - 1)(m - 2)} \cdot \frac{-\varepsilon_{\min} b_{\min}^2}{2} \le \frac{\varepsilon_{\max}}{(m - 1) b_{\min} a_{\min}}, \\
\bar{a}_i - a_i &\ge \frac{2}{m - 1} \cdot \frac{\varepsilon_{\min} b_{\min}}{2} + \frac{2(d - 1)}{(m - 1)(m - 2)} \cdot \frac{-\varepsilon_{\max}}{2 b_{\min}^2 a_{\min}^2} = \frac{\varepsilon_{\min} b_{\min}}{m - 1} - \frac{(d - 1)\varepsilon_{\max}}{(m - 1)(m - 2)b^2_{\min} a^2_{\min}}.
\end{align}

Note that this lower bound can be negative in this case, so it is not clear if $\bar{a}_i$ or $a_i$ is bigger in expectation. 

If $i \notin E_{\lf}$, using \eqref{eq:acc_diff_jk} then the expected error is bounded as
\begin{align}
\bar{a}_i - a_i &\le \frac{2d}{(m - 1)(m - 2)} \cdot \frac{-\varepsilon_{\min} b_{\min}^2}{2} = \frac{-d\varepsilon_{\min} b_{\min}^2 }{(m - 1)(m - 2)},  \\
\bar{a}_i - a_i &\ge \frac{2d}{(m - 1)(m - 2)}  \cdot \frac{-\varepsilon_{\max}}{2 b_{\min}^2 a_{\min}^2} = \frac{-d \varepsilon_{\max}}{(m - 1)(m - 2) b_{\min}^2 a_{\min}^2}.
\end{align}

In this case, $\bar{a}_i \le a_i$. Finally, observe that regardless of if $i \in E_{\lf}$ or not, the absolute value of the bias is bounded by  
\begin{align}
    |\bar{a}_i - a_i| \le \frac{\varepsilon_{\max}}{(m - 1)b_{\min}^2 a_{\min}^2}.
\end{align}
\end{proof}

We return to \eqref{eq:acc_decomposition}. Since $a_i \ge \bar{a}_i$ when $i \notin E_{\lf}$, we have that $\frac{1 + a_i}{2}\log(1 + \frac{a_i - \bar{a}_i}{1 + \bar{a}_i}) + \frac{1 - a_i}{2} \log (1 + \frac{\bar{a}_i - a_i}{1 - \bar{a}_i}) \le \frac{1 + a_i}{2} \log (1 + \max \frac{a_i - \bar{a}_i}{1 + \bar{a}_i})$ for $i \notin E_{\lf}$. On the other hand when $i \in E_{\lf}$, this expression can be upper bounded as $\frac{1+a_i}{2} \cdot \frac{a_i - \bar{a}_i}{1 + \bar{a}_i} + \frac{1 - a_i}{2} \frac{\bar{a}_i - a_i}{1 - \bar{a}_i} = \frac{(\bar{a}_i - a_i)^2}{1 - \bar{a}_i^2}$ using the inequality $\log(1 + x) \le x$ for $x > -1$ (it can be easily verified that $\frac{a_i - \bar{a}_i}{1 + \bar{a}_i}$ and $\frac{\bar{a}_i - a_i}{1 - \bar{a}_i}$ are at least $-1$). Since $|E_{\lf}| = 2d$ and $\varepsilon_{\max} \le 1$, the first summation of \eqref{eq:acc_decomposition} is bounded by 
\begin{align}
    &(m - 2d) \log \left(1 + \frac{d\varepsilon_{\max}}{(m - 1)(m - 2) b_{\min}^2 a_{\min}^2 (1 + b_{\min})} \right) + 2d \frac{\varepsilon_{\max}^2}{(1 - \bar{a}_{\max}^2) (m - 1)^2 b_{\min}^4 a_{\min}^4} \\
    \le &\frac{(m - 2d) d\varepsilon_{\max}}{(m - 1)(m - 2) b_{\min}^2 a_{\min}^2 (1 + b_{\min})} + \frac{2d \varepsilon_{\max}}{(1 - \bar{a}_{\max}^2) (m - 1)^2 b_{\min}^4 a_{\min}^4} \nonumber \\
    = & \frac{d\varepsilon_{\max}}{(m - 1)b_{\min}^2 a_{\min}^2} \left( \frac{m - 2d}{(m - 2)(1 + b_{\min})} + \frac{2}{(1 - \bar{a}_{\max}^2)(m - 1) b_{\min}^2 a_{\min}^2}\right) \nonumber \\
    \le & \frac{d\varepsilon_{\max}}{(m - 1)b_{\min}^2 a_{\min}^2} \left(1 + \frac{1}{(1 - \bar{a}_{\max}^2) b_{\min}^2 a_{\min}^2}\right) \le \frac{c_1 d \varepsilon_{\max}}{m}, \nonumber
\end{align}

where $c_1 = \frac{2}{b_{\min}^2 a_{\min}^2} \left(1 + \frac{1}{(1 - \bar{a}_{\max}^2) b_{\min}^2 a_{\min}^2} \right)$. Next, we bound $\sum_{i = 1}^m \frac{a_i - \bar{a}_i}{1 - \bar{a}_i^2} \E{\N, \tau}{\bar{a}_i - \widetilde{a}_i}$:
\begin{align}
    \sum_{i = 1}^m \frac{a_i - \bar{a}_i}{1 - \bar{a}_i^2} \E{\N, \tau}{\bar{a}_i - \widetilde{a}_i} \le &\sum_{i = 1}^m \frac{|\bar{a}_i - a_i|}{1 - \bar{a}_i^2} \E{\N, \tau}{|\bar{a}_i - \widetilde{a}_i|} \\
    \le &\frac{\sqrt{3}}{2\sqrt{n_U}} \cdot \sqrt{\frac{1 - b_{\min}^2}{b_{\min}^2} \left(\frac{1}{b_{\min}^4} + \frac{2}{b_{\min}^2} \right)} \frac{1}{1 - \bar{a}_{\max}^2} \left(\frac{m \varepsilon_{\max}}{(m - 1) b_{\min}^2 a_{\min}^2}\right) \le \frac{c_2 \varepsilon_{\max}}{ \sqrt{n_U}}, \nonumber
\end{align}

where $c_2 = \frac{1}{(1 - \bar{a}_{\max}^2) b_{\min}^2 a_{\min}^2} \sqrt{\frac{3(1 - b_{\min}^2)}{b_{\min}^2} \left(\frac{1}{b_{\min}^4} + \frac{2}{b_{\min}^2} \right)}$. We bound $\sum_{i = 1}^m \frac{1}{2}\Big( \frac{1}{1 - \bar{a}_i^2} + \frac{2\bar{a}_i(\bar{a}_i - a_i)}{(1 - \bar{a}_i^2)^2}\Big) \E{\N, \tau}{(\widetilde{a}_i - \bar{a}_i)^2}$, which can be split into an expression independent of misspecification and one dependent on it:
\begin{align}
    \sum_{i = 1}^m \frac{1}{2}\Big( \frac{1}{1 - \bar{a}_i^2} + \frac{2\bar{a}_i(\bar{a}_i - a_i)}{(1 - \bar{a}_i^2)^2}\Big) \E{\N, \tau}{(\widetilde{a}_i - \bar{a}_i)^2} \le \frac{c_4 m}{n_U} + \sum_{i = 1}^m \frac{\bar{a}_i - a_i}{(1 - \bar{a}_i^2)^2} \E{\N, \tau}{(\widetilde{a}_i - \bar{a}_i)^2}, \label{eq:variance_of_estimator}
\end{align}

where $c_4 = \frac{3(1 - b_{\min}^2)}{8b_{\min}^2 (1 - \bar{a}_{\max}^2)} \left(\frac{1}{b_{\min}^4} + \frac{2}{b_{\min}^2} \right)$. The summation in \eqref{eq:variance_of_estimator} is bounded as follows, using the fact that $\bar{a}_i \le a_i$ for $i \notin E_{\lf}$:
\begin{align}
    \sum_{i = 1}^m \frac{\bar{a}_i - a_i}{(1 - \bar{a}_i^2)^2} \E{\N, \tau}{(\widetilde{a}_i - \bar{a}_i)^2} &\le \frac{3}{4 n_U} \cdot \frac{1 - b_{\min}^2}{b_{\min}^2 (1 - \bar{a}_{\max}^2)^2} \left(\frac{1}{b_{\min}^4} + \frac{2}{b_{\min}^2} \right) \sum_{i \in E_{\lf}} |\bar{a}_i - a_i | \\
    &\le \frac{3}{4 n_U} \cdot \frac{1 - b_{\min}^2}{b_{\min}^2 (1 - \bar{a}_{\max}^2)^2} \left(\frac{1}{b_{\min}^4} + \frac{2}{b_{\min}^2} \right) \left(\frac{2d \varepsilon_{\max}}{(m - 1) b_{\min}^2 a_{\min}^2}\right) \le \frac{c_3 d \varepsilon_{\max}}{m n_U}, \nonumber 
\end{align}

where $c_3 = \frac{3(1 - b_{\min}^2)}{(1 - \bar{a}_{\max}^2)^2 b_{\min}^4 a_{\min}^2} \left(\frac{1}{b_{\min}^4} + \frac{2}{b_{\min}^2} \right)$. This concludes our proof.

\subsection{Proof of Proposition 1} \label{subsec:medians}

To prove the ability of using the median of the accuracies to correct for misspecification, we first examine the asymptotic case. For $i \in E_{\lf}$, note that out of a total of ${m - 1 \choose 2}$ triplets, $m - 2$ of them will involve the edge $(i, j)\in E_{\lf}$, resulting in a higher inconsistent estimate of the accuracy. $d - 1$ of them will involve an edge $(j, k) \in E_{\lf}$, resulting in a lower estimate of the accuracy. Therefore, $\frac{(m - 1)(m - 2)}{2} - m - d - 3$ triplets are consistent. As long as the ${m - 1 \choose 2} - (m - 2)$th largest triplet is greater than half of all the triplets, and the $d - 1$th largest triplet is less than the half of all the triplets, then the median will be a consistent triplet. This gives us the conditions $m > 5$ and $d < \frac{(m - 1)(m - 2)}{4}$.

Next, for $i \notin E_{\lf}$, $d$ triplets will involve an edge $(j, k) \in E_{\lf}$, resulting in lower estimated accuracy, while the other ${m - 1 \choose 2} - d$ triplets are consistent. Therefore, as long as $d < \frac{(m - 1)(m - 2)}{4}$, the median triplet is consistent. 

Lastly, we must consider the finite-sample regime when the ordering of the accuracy estimates are perturbed by sampling noise. When each accuracy's expected sampling noise is less than half of the minimum standing bias of a triplet, the order of the accuracies will not change on average. This translates into the inequality $\E{}{|\widetilde{a}_i - \bar{a}_i|} \le \frac{1}{2} \min_{(j, k)}|a_i - \bar{a}_i^{(j, k)}|$. The minimum standing bias is $\frac{\varepsilon_{\min} b_{\min}^2}{2}$, and $\E{}{|\widetilde{a}_i - \bar{a}_i|} \sim \mathcal{O}(1/\sqrt{n})$ so this means that $n_U \ge n_0 \sim \Omega(1/\varepsilon_{\min}^2)$.

Lastly, we compute the excess risk when using the corrected estimator. From Lemma \ref{lemma:accuracy_bias}, since the asymptotic expectation $\bar{a}$ of the estimator is equal to the true accuracy $a$, we have
\begin{align}
    \E{\N, \tau, Y}{\KL(\text{Pr}_{\lf_i | Y} || \Ptilde_{\lf_i | Y})} &= \frac{1 + a_i}{2} \bigg(\frac{\mathbb{E}[a_i - \widetilde{a}_i^M]}{1 + a_i} + \frac{1}{2(1 + a_i)^2} \mathbb{E}[(\widetilde{a}_i^M - a_i)^2]\bigg) \\
    &+ \frac{1 - a_i}{2} \bigg(\frac{\mathbb{E}[\widetilde{a}_i^M - a_i]}{1 - a_i} + \frac{1}{2(1 - a_i)^2} \mathbb{E}[(\widetilde{a}_i^M - a_i)^2] \bigg). \nonumber 
\end{align}

Note that $\frac{1 + a_i}{2} \cdot \frac{\mathbb{E}[a_i - \widetilde{a}_i^M]}{1 + a_i} + \frac{1 - a_i}{2} \cdot \frac{\mathbb{E}[\widetilde{a}_i^M - a_i]}{1 - a_i} = 0$. Then the parameter estimation error is
\begin{align}
    \sum_{i = 1}^m \bigg(\frac{1}{4(1 + a_i)} + \frac{1}{4(1 - a_i)}\bigg) \E{}{(\widetilde{a}_i^M - a_i)^2} = \sum_{i = 1}^m \frac{1}{2(1 - a_i^2)} \E{}{(\widetilde{a}_i^M - a_i)^2} \le \frac{1}{2(1 - \max_i a_i^2)} \cdot m \rho_{n_U}.
\end{align}

This completes our proof, where $c_\rho = \frac{1}{2(1 - \max_i a_i^2)}$ in Proposition \ref{prop:medians}.

\section{Auxiliary Lemmas}

\begin{lemma} (Symmetry of the distribution).
For any source $\lf_i$ with accuracy $a_i = \E{}{\lf_i Y}$,
\begin{align*}
\Pr(\lf_i = 1 | Y = 1) &= \Pr(\lf_i = -1 | Y = -1) = \frac{1 + a_i}{2}  \\
\Pr(\lf_i = -1 | Y = 1) &= \Pr(\lf_i = 1 | Y = -1) = \frac{1 - a_i}{2}.
\end{align*}
\label{lemma:fs_symmetry}
\end{lemma}
\begin{proof}
By Proposition $2$ of \cite{fu2020fast}, we know that $\lf_i Y \independent Y$ for the binary Ising model we use, defined in section \ref{sec:background}. Intuitively, this means that the accuracy of a source is independent of the value of $Y$, and therefore $\Pr(\lf_i Y = 1 | Y = 1) = \Pr(\lf_i Y = 1) = \frac{1 + a_i}{2}$, since $\E{}{\lf_i Y} = 2 \Pr(\lf_i Y = 1) - 1$. Repeating this calculation with remaining configurations of $\Pr(\lf_i Y = \pm 1 | Y = \pm 1)$ concludes our proof.
\end{proof}

\begin{lemma}
Define $a_i = \E{}{\lf_i Y},$ and let $\widetilde{a}_i$ be our estimated accuracy on $n$ points. Furthermore, let $\bar{a}_i$ be the expected asymptotic value of $\widetilde{a}_i$ over $\tau$.   %$\bar{a}_i = \E{\tau}{\sqrt{\frac{\E{}{\lf_i \lf_j} \E{}{\lf_i \lf_k}}{\E{}{\lf_j \lf_k}}}},$ and $\widetilde{a}_i = \sqrt{\frac{\Ehat{\lf_i \lf_j} \Ehat{\lf_i \lf_k}}{\Ehat{\lf_j \lf_k}}}$ estimated on $n$ points. 
Then, the estimation error is
\begin{align*}
\E{Y, \N, \tau}{\KL (\mathrm{Pr}_{\lf_i | Y} || \Ptilde_{\lf_i | Y} )} = &\frac{1 + a_i}{2} \Big(\log\Big(1 + \frac{a_i - \bar{a}_i }{1 + \bar{a}_i}\Big) + \frac{\E{\N, \tau}{\bar{a}_i - \widetilde{a}_i}}{1 + \bar{a}_i} + \frac{1}{2(1 + \bar{a}_i)^2} \E{\N, \tau}{(\widetilde{a}_i - \bar{a}_i)^2}\Big) \\
&+  \frac{1 - a_i}{2}\Big(\log\Big(1 + \frac{\bar{a}_i - a_i}{1 - \bar{a}_i} \Big) +   \frac{\E{\N, \tau}{\widetilde{a}_i - \bar{a}_i}}{1 - \bar{a}_i} + \frac{1}{2(1 - \bar{a}_i)^2} \E{\N, \tau}{(\widetilde{a}_i - \bar{a}_i)^2}\Big) \\
&+ o(1/n).
\end{align*}
\label{lemma:KL_estimation}
\end{lemma}

\begin{proof}

As discussed previously, this term is equal to $-\E{(Y, \bm{\lf}), \N, \tau}{\log \frac{\Ptilde(\lf_i' = \lf_i | Y' = Y)}{\Pr(\lf_i' = \lf_i | Y' = Y)}}$. By the law of total expectation, we now have
\begin{align}
    -\E{\bm{\lf}, \N, \tau}{ \Pr(Y = 1 | \bm{\lf}' = \bm{\lf}) \log \frac{\Ptilde(\lf_i' = \lf_i | Y = 1)}{\Pr(\lf_i' = \lf_i | Y = 1)} +  \Pr(Y = -1 | \bm{\lf}' = \bm{\lf}) \log \frac{\Ptilde(\lf_i' = \lf_i | Y = -1)}{\Pr(\lf_i' = \lf_i | Y = -1)}}.
    \label{eq:estimation1}
\end{align}

Suppose $\lf_i \notin E_{\lf}$. Conditioning on the value of $\lf_i$ and using Lemma \ref{lemma:fs_symmetry}, \eqref{eq:estimation1} becomes 
\begin{align*}
&-\E{\bm{\lf}_{-i}, \N, \tau}{\E{\lf_i}{\Pr(Y = 1 | \bm{\lf}' = \bm{\lf}) \log \frac{\Ptilde(\lf_i' = \lf_i | Y = 1)}{\Pr(\lf_i' = \lf_i | Y = 1) }  + \Pr(Y = -1 | \bm{\lf}' = \bm{\lf}) \log \frac{\Ptilde(\lf_i' = \lf_i | Y = -1)}{\Pr(\lf_i' = \lf_i | Y = -1)} \Big| \bm{\lf}_{-i}}} \\
= \;&-\mathbb{E}_{\bm{\lf}_{-i}, \N, \tau}\bigg[\big(\Pr(Y = 1 | \bm{\lf}_{-i}, \lf_i = 1) \Pr(\lf_i = 1 | \bm{\lf}_{-i}) + \Pr(Y = -1 | \bm{\lf}_{-i}, \lf_i = -1) \Pr(\lf_i = -1 | \bm{\lf}_{-i}) \big)\log \frac{1 + \widetilde{a}_i}{1 + a_i} \\
&+  \big(\Pr(Y = 1 | \bm{\lf}_{-i}, \lf_i = -1) \Pr(\lf_i = -1 | \bm{\lf}_{-i}) + \Pr(Y = -1 | \bm{\lf}_{-i}, \lf_i = 1) \Pr(\lf_i = 1 | \bm{\lf}_{-i})  \big)\log \frac{1 - \widetilde{a}_i}{1 - a_i} \bigg] \\
= \;& -\mathbb{E}_{\bm{\lf}_{-i}, \N, \tau}\bigg[\Pr(\lf_i Y = 1 | \bm{\lf}_{-i}) \log \frac{1 + \widetilde{a}_i}{1 - a_i} + \Pr(\lf_i Y = -1 | \bm{\lf}_{-i})\log \frac{1 - \widetilde{a}_i}{1 - a_i} \bigg].
\end{align*}

Note that $\Pr(\lf_i = 1, Y = 1 | \bm{\lf}_{-i}) = \Pr(\lf_i = 1 | Y = 1)  \frac{\Pr(\bm{\lf}_{-i}, Y = 1)}{\Pr(\bm{\lf}_{-i})}$ and $\Pr(\lf_i = -1, Y = -1 | \bm{\lf}_{-i}) = \Pr(\lf_i = -1 | Y = -1) \frac{ \Pr(\bm{\lf}_{-i}, Y = -1)}{\Pr(\bm{\lf}_{-i})}$ since $\lf_i$ and $\lf_{-i}$ are conditionally independent given $Y$, so $\Pr(\lf_i Y = 1 | \bm{\lf}_{-i}) = \Pr(\lf_i = 1| Y = 1) = \frac{1 + a_i}{2}$. Similarly, $\Pr(\lf_i Y = -1 | \bm{\lf}_{}-i) = \Pr(\lf_i = -1 | Y = 1) = \frac{1 - a_i}{2}$, so the conditional KL divergence is equal to 
\begin{align}
\E{\N, \tau, Y}{\KL (\mathrm{Pr}_{\lf_i | Y} || \Ptilde_{\lf_i | Y} )} =-&\E{\N, \tau}{\frac{1 + a_i}{2} \log \frac{1 + \widetilde{a}_i}{1 + a_i} + \frac{1 - a_i}{2} \log \frac{1 - \widetilde{a}_i}{1 - a_i}}. \label{eq:estimation2}
\end{align}

Now suppose that $\lf_i \in E_{\lf}$ and has an edge to some $\lf_j$. When we simplify \eqref{eq:estimation1} by conditioning on $\lf_i, \lf_j$, we find that $\sum_{l \in \{\pm 1\}}\Pr(Y = 1 | \bm{\lf}_{-i, j}, \lf_i = 1, \lf_j = l) \Pr(\lf_i = 1, \lf_j = l | \bm{\lf}_{-i, j}) + \Pr(Y = -1 | \bm{\lf}_{-i, j}, \lf_i = -1, \lf_j = l) \Pr(\lf_i = -1, \lf_j = l | \bm{\lf}_{-i, j} )$ (i.e, the coefficient for $\log \frac{1 + \widetilde{a}_i}{1 + a_i}$) is equal to $\Pr(\lf_i Y = 1 | \bm{\lf}_{-i, j})$, and this is still equal to $\frac{1 + a_i}{2}$. The same holds for the coefficient of $\log \frac{1 - \widetilde{a}_i}{1 - a_i}$. Therefore, \eqref{eq:estimation2} holds for all $\lf_i$.

Next, we evaluate $-\E{}{\log \frac{1 + \widetilde{a}_i}{1 + a_i}}$ and $-\E{}{\log \frac{1 - \widetilde{a}_i}{1 - a_i}}$, where expectation is over $\N$ and $\tau$. We apply a second-order Taylor approximation of $f(x) = \log \frac{1 + x}{1 + a_i}$ at $x = \bar{a}_i$:
\begin{align*}
\log \frac{1 + \widetilde{a}_i}{1 + a_i} \approx \log \frac{1 + \bar{a}_i}{1 + a_i} + \frac{1 + a_i}{1 + \bar{a}_i} \cdot \frac{1}{1 + a_i} (\widetilde{a}_i - \bar{a}_i) - \frac{1}{2(1 + \bar{a}_i)^2} (\widetilde{a}_i - \bar{a}_i)^2 + o(1/n).
\end{align*}

Taking the expectation on both sides, we get 
\begin{align*}
-\E{\N, \tau}{\log \frac{1 + \widetilde{a}_i}{1 + a_i}} &\approx -\Big(\log \frac{1 + \bar{a}_i}{1 + a_i} + \frac{\E{\N, \tau}{\widetilde{a}_i} - \bar{a}_i}{1 + \bar{a}_i} - \frac{1}{2(1 + \bar{a}_i)^2} \E{\N, \tau}{(\widetilde{a}_i - \bar{a}_i)^2}\Big) + o(1/n) \\
&= \log \Big(1 + \frac{a_i - \bar{a}_i}{1 + \bar{a}_i} \Big) + \frac{\E{\N, \tau}{\bar{a}_i - \widetilde{a}_i}}{1 + \bar{a}_i} + \frac{1}{2(1 + \bar{a}_i)^2} \E{\N, \tau}{(\widetilde{a}_i - \bar{a}_i)^2} + o(1/n),
\end{align*}
where we have used Lemma \ref{lemma:taylor}. 

Similarly, we apply a second-order Taylor approximation of $f(x) = \log \frac{1 - x}{1 - a_i}$ at $x = \bar{a}_i$:
\begin{align*}
\log \frac{1 - \widetilde{a}_i}{1 - a_i} \approx \log \frac{1 - \bar{a}_i}{1 - a_i} + \frac{1 - a_i}{1 - \bar{a}_i} \cdot \frac{-1}{1 - a_i} (\widetilde{a}_i - \bar{a}_i) - \frac{1}{2(1 - \bar{a}_i)^2} (\widetilde{a}_i - \bar{a}_i)^2 + o(1/n).
\end{align*}

Taking the expectation of both sides,
\begin{align*}
-\E{}{\log \frac{1 - \widetilde{a}_i}{1 - a_i}} &= -\Big(\log \frac{1 - \bar{a}_i}{1 - a_i} + \frac{\E{\N, \tau}{\bar{a}_i - \widetilde{a}_i}}{1 - \bar{a}_i} - \frac{1}{2(1 - \bar{a}_i)^2} \E{\N, \tau}{(\widetilde{a}_i - \bar{a}_i)^2}\Big) + o(1/n) \\
&= \log \Big(1 +  \frac{\bar{a}_i - a_i}{1 - \bar{a}_i}\Big) + \frac{\E{\N, \tau}{\widetilde{a}_i - \bar{a}_i}}{1 - \bar{a}_i} + \frac{1}{2(1 - \bar{a}_i)^2} \E{\N, \tau}{(\widetilde{a}_i - \bar{a}_i)^2} + o(1/n).
\end{align*}

Substituting these expressions into \eqref{eq:estimation2}, we get our desired equation.

%Note that when $i \notin E_{\lf}$, $\bar{a}_i \le a_i$ and so $\log \Big( 1 + \frac{\bar{a}_i - a_i}{1 - \bar{a}_i}\Big) \le 0$, $\log \Big(1 + \frac{a_i - \bar{a}_i}{1 + \bar{a}_i} \Big) \ge 0$.
\end{proof}

\begin{lemma}
The remainder of the Taylor approximation done in Lemma \ref{lemma:KL_estimation} is $o(1/n)$ for estimation done on $n$ samples in both the labeled and unlabeled cases.
\label{lemma:taylor}
\end{lemma}

\begin{proof}
The remainder for $-\E{\N, \tau}{\log \frac{1 + \widetilde{a}_i}{1 + a_i}}$ is bounded by $\frac{1}{3(1 + \bar{a}_i)^3}\E{\N, \tau}{(\bar{a}_i - \widetilde{a}_i)^3}$, and the remainder for $-\E{\N, \tau}{\log \frac{1 - \widetilde{a}_i}{1 - a_i}}$ is bounded by $\frac{1}{3(1 - \bar{a}_i)^3} \E{\N, \tau}{(\bar{a}_i - \widetilde{a}_i)^3}$.

For the labeled data case, it is easy to check that $\E{\N}{(\bar{a}_i - \widetilde{a}_i)^3} \sim \mathcal{O}(1/n_L^2)$. Therefore, we focus on analyzing the unlabeled data case's estimator by bounding $\E{\N}{|\bar{a}_i - \widetilde{a}_i|^3 \; | \; \lf_j, \lf_k}$ independent of choice of $j$ and $k$. For ease of notation, define $X =\lf_i \lf_j$ and $Y = \lf_i \lf_k$, such that $XY = \lf_j \lf_k$, and let 
\begin{align}
    a := \bar{a}_i^{(j, k)} = \sqrt{  \frac{\mathbb{E}[X] \mathbb{E}[Y]}{\mathbb{E}[XY]}   }, \qquad \hat{a} := \widetilde{a}_i = \sqrt{  \frac{\hat{\mathbb{E}}[X] \hat{\mathbb{E}}[Y]}{\hat{\mathbb{E}}[XY]}   }.
\end{align}

Note $a \in [-1,1]$, so clip $\hat{a} \in [-1,1]$. Because $X \in \{-1,1\}$ and $\hat{\mathbb{E}}[X]$ is an i.i.d. sum of $n = n_U$ samples from $X$, we can apply Hoeffding's inequality to get:
\begin{align}
    \Pr \left( |\hat{\mathbb{E}}[X] - \mathbb{E}[X]| \geq \epsilon \right) &\leq 2 \exp \left( - \frac{2 n^2 \epsilon^2}{n \cdot 2^2}   \right) = 2 \exp \left( - \frac{n \epsilon^2}{2} \right).
\end{align}

The same is true for $\hat{\mathbb{E}}[Y]$ and $\hat{\mathbb{E}}[XY]$. Thus, by union bound,
\begin{align}
    \Pr \left( |\hat{\mathbb{E}}[X] - \mathbb{E}[X]| \geq \epsilon \vee |\hat{\mathbb{E}}[Y] - \mathbb{E}[Y]| \geq \epsilon \vee |\hat{\mathbb{E}}[XY] - \mathbb{E}[XY]| \geq \epsilon \right) &\leq 6 \exp \left( - \frac{n \epsilon^2}{2} \right).
\end{align}

Refer to the event $\left( |\hat{\mathbb{E}}[X] - \mathbb{E}[X]| \geq \epsilon \vee |\hat{\mathbb{E}}[Y] - \mathbb{E}[Y]| \geq \epsilon \vee |\hat{\mathbb{E}}[XY] - \mathbb{E}[XY]| \geq \epsilon \right)$ as $B$. If $\neg B$ and $\epsilon< \frac{1}{2} \min(\mathbb{E}[X], \mathbb{E}[Y], \mathbb{E}[XY]) < 1$, then
\begin{align}
    |\hat{\mathbb{E}}[X] - \mathbb{E}[X]| &< \epsilon, \quad|\hat{\mathbb{E}}[Y] - \mathbb{E}[Y]| < \epsilon, \quad |\hat{\mathbb{E}}[XY] - \mathbb{E}[XY]| < \epsilon.
\end{align}

By the mean value theorem with $f(x) = \sqrt{x}$, there exists a $u$ between $\frac{\hat{\mathbb{E}}[X] \hat{\mathbb{E}}[Y]}{\hat{\mathbb{E}}[XY]}$ and $\frac{\mathbb{E}[X] \mathbb{E}[Y]}{\mathbb{E}[XY]}$ such that 
\begin{align}
    \left| \hat{a} -  a \right| &= \left| \frac{1}{2\sqrt{u}} \left(\frac{\hat{\mathbb{E}}[X] \hat{\mathbb{E}}[Y]}{\hat{\mathbb{E}}[XY]} - \frac{\mathbb{E}[X] \mathbb{E}[Y]}{\mathbb{E}[XY]} \right) \right|.
\end{align}

Note that 
\begin{align}
    u &\geq \min \left( \frac{\hat{\mathbb{E}}[X] \hat{\mathbb{E}}[Y]}{\hat{\mathbb{E}}[XY]}, \frac{\mathbb{E}[X] \mathbb{E}[Y]}{\mathbb{E}[XY]} \right) \geq \min \left( \frac{(\mathbb{E}[X] - \epsilon) (\mathbb{E}[Y] - \epsilon)}{\mathbb{E}[XY] + \epsilon}, \frac{\mathbb{E}[X] \mathbb{E}[Y]}{\mathbb{E}[XY]} \right) \\
      &\geq \min \left( \frac{(\mathbb{E}[X]/2) (\mathbb{E}[Y]/2)}{1 + \epsilon}, \frac{\mathbb{E}[X] \mathbb{E}[Y]}{\mathbb{E}[XY]} \right) \geq \min \left( \frac{\mathbb{E}[X] \mathbb{E}[Y]}{8}, \frac{\mathbb{E}[X] \mathbb{E}[Y]}{\mathbb{E}[XY]} \right) \geq \frac{\mathbb{E}[X] \mathbb{E}[Y]}{8}. \nonumber 
\end{align}

Thus,
\begin{align}
    | \hat{a} -  a | &\leq \frac{\sqrt{2}}{\sqrt{\mathbb{E}[X] \mathbb{E}[Y]}} \left| \frac{\hat{\mathbb{E}}[X] \hat{\mathbb{E}}[Y]}{\hat{\mathbb{E}}[XY]} - \frac{\mathbb{E}[X] \mathbb{E}[Y]}{\mathbb{E}[XY]}  \right|.
\end{align}

For the term on the right inside the absolute value:
\begin{align}
    \frac{(\mathbb{E}[X] - \epsilon)(\mathbb{E}[Y] - \epsilon)}{\mathbb{E}[XY] + \epsilon} &\leq \frac{\hat{\mathbb{E}}[X] \hat{\mathbb{E}}[Y]}{\hat{\mathbb{E}}[XY]} \leq \frac{(\mathbb{E}[X] + \epsilon)(\mathbb{E}[Y] + \epsilon)}{\mathbb{E}[XY] - \epsilon} \\
    \frac{(\mathbb{E}[X] - \epsilon)(\mathbb{E}[Y] - \epsilon)}{\mathbb{E}[XY] + \epsilon}  - \frac{\mathbb{E}[X] \mathbb{E}[Y]}{\mathbb{E}[XY]} &\leq \frac{\hat{\mathbb{E}}[X] \hat{\mathbb{E}}[Y]}{\hat{\mathbb{E}}[XY]} - \frac{\mathbb{E}[X] \mathbb{E}[Y]}{\mathbb{E}[XY]} \leq \frac{(\mathbb{E}[X] + \epsilon)(\mathbb{E}[Y] + \epsilon)}{\mathbb{E}[XY] - \epsilon} - \frac{\mathbb{E}[X] \mathbb{E}[Y]}{\mathbb{E}[XY]} \nonumber \\
    \left| \frac{\hat{\mathbb{E}}[X] \hat{\mathbb{E}}[Y]}{\hat{\mathbb{E}}[XY]} - \frac{\mathbb{E}[X] \mathbb{E}[Y]}{\mathbb{E}[XY]} \right| &\leq \max \bigg( \left|  \frac{(\mathbb{E}[X] - \epsilon)(\mathbb{E}[Y] - \epsilon)}{\mathbb{E}[XY] + \epsilon}  - \frac{\mathbb{E}[X] \mathbb{E}[Y]}{\mathbb{E}[XY]} \right|, \nonumber \\
    &\left| \frac{(\mathbb{E}[X] + \epsilon)(\mathbb{E}[Y] + \epsilon)}{\mathbb{E}[XY] - \epsilon} - \frac{\mathbb{E}[X] \mathbb{E}[Y]}{\mathbb{E}[XY]} \right| \bigg). \nonumber 
\end{align}

Examining the left term in the max,
\begin{align}
    \left| \frac{(\mathbb{E}[X] - \epsilon)(\mathbb{E}[Y] - \epsilon)}{\mathbb{E}[XY] + \epsilon}  - \frac{\mathbb{E}[X] \mathbb{E}[Y]}{\mathbb{E}[XY]} \right| &= \left| \frac{(\mathbb{E}[X] - \epsilon)(\mathbb{E}[Y] - \epsilon)\mathbb{E}[XY] - \mathbb{E}[X] \mathbb{E}[Y] (\mathbb{E}[XY] + \epsilon)}{\mathbb{E}[XY](\mathbb{E}[XY] + \epsilon)} \right| \\
    &= \left|\frac{ -\epsilon (\mathbb{E}[X] \mathbb{E}[Y] + \mathbb{E}[X] \mathbb{E}[XY] + \mathbb{E}[Y] \mathbb{E}[XY] - \epsilon \mathbb{E}[XY])}{\mathbb{E}[XY](\mathbb{E}[XY] + \epsilon)} \right| \nonumber \\
    &\leq \epsilon \left| \frac{ \mathbb{E}[X] \mathbb{E}[Y] + \mathbb{E}[X] \mathbb{E}[XY] + \mathbb{E}[Y] \mathbb{E}[XY]}{\mathbb{E}[XY]^2} \right| \nonumber \\
    &= \epsilon C_1 > 0 \nonumber 
\end{align}

Examining the right term in the max,
\begin{align}
    \left| \frac{(\mathbb{E}[X] + \epsilon)(\mathbb{E}[Y] + \epsilon)}{\mathbb{E}[XY] - \epsilon}  - \frac{\mathbb{E}[X] \mathbb{E}[Y]}{\mathbb{E}[XY]} \right| &= \left| \frac{(\mathbb{E}[X] + \epsilon)(\mathbb{E}[Y] + \epsilon)\mathbb{E}[XY] - \mathbb{E}[X] \mathbb{E}[Y] (\mathbb{E}[XY] - \epsilon)}{\mathbb{E}[XY](\mathbb{E}[XY] - \epsilon)} \right| \\
    &= \left|\frac{ \epsilon (\mathbb{E}[X] \mathbb{E}[Y] + \mathbb{E}[X] \mathbb{E}[XY] + \mathbb{E}[Y] \mathbb{E}[XY] + \epsilon \mathbb{E}[XY])}{\mathbb{E}[XY](\mathbb{E}[XY] - \epsilon)} \right| \\
    &\leq \epsilon \left| \frac{ \mathbb{E}[X] \mathbb{E}[Y] + \mathbb{E}[X] \mathbb{E}[XY] + \mathbb{E}[Y] \mathbb{E}[XY] + \mathbb{E}[XY]}{\mathbb{E}[XY]^2/2} \right| \\
    &= \epsilon C_2 > 0
\end{align}

Combining the max argument bounds, we have that $ \left| \frac{\hat{\mathbb{E}}[X] \hat{\mathbb{E}}[Y]}{\hat{\mathbb{E}}[XY]} - \frac{\mathbb{E}[X] \mathbb{E}[Y]}{\mathbb{E}[XY]} \right| \leq \epsilon \max(C_1, C_2) \leq \epsilon C_2$. Therefore, 
\begin{align}
    | \hat{a} -  a | &\leq \epsilon \frac{\sqrt{2}C_2}{\sqrt{\mathbb{E}[X] \mathbb{E}[X]}} = \epsilon C_3
\end{align}

where $C_3$ is a positive function of $\mathbb{E}[X]$, $\mathbb{E}[Y]$, and $\mathbb{E}[XY]$. To recap, this is satisfied if $\neg B$ and $\epsilon$ is small. Let $\epsilon = n^{-3/8}$, thus for large enough $n$, $\epsilon$ is smaller than any constant. Recall, $\Pr(B) \leq 6 \exp( - n \epsilon^2 / 2)$. With this definition of $\epsilon$, $\Pr(B) \leq 6 \exp(-n^{1/4}/2)$.

Now, we are finally ready to evaluate the limit:
\begin{align}
    \lim_{n \rightarrow \infty} n \mathbb{E}[ |\hat{a} - a|^3 ] &= \lim_{n \rightarrow \infty} n \left( \mathbb{E}[ |\hat{a} - a|^3 | B] \Pr(B) + \mathbb{E}[ |\hat{a} - a|^3 | \neg B] P(\neg B) \right) \\
    &\leq \lim_{n \rightarrow \infty} n \left(C_3^3 \epsilon^3 \cdot 1 + 2^3 \cdot 6\exp(-n^{1/4}/2) \right) \\
    &= C_3^3 \lim_{n \rightarrow \infty} n (n^{-3/8})^3 + 48 \lim_{n \rightarrow \infty} n \exp(-n^{1/4}/2) \\
    &= C_3^3 \lim_{n \rightarrow \infty} n^{-1/8} + 48 \lim_{m \rightarrow \infty} m^4 \exp(-m/2) = 0
\end{align}

Trivially, $\lim_{n \rightarrow \infty} n \mathbb{E}[ |\hat{a} - a|^3 ] \geq 0$. Thus, $\lim_{n \rightarrow \infty} n \mathbb{E}[ |\hat{a} - a|^3 ] = 0$.
\end{proof}

\begin{lemma} (Quantifying per-edge misspecification.)
If $(i, j) \in E_{\lf}$, then
\begin{align}
    \varepsilon_{ij} = \Delta_{ij} - \Delta_i a_j' - \Delta_j a_i' - \Delta_i \Delta_j,
\end{align}

where
\begin{align}
    \Delta_i &= \frac{2}{z_{ij} z_{ij}'} (\exp(\theta_{ij}) - \exp(-\theta_{ij})) (\exp(2\theta_j) - \exp(-2\theta_j)) \\
    \Delta_j &= \frac{2}{z_{ij} z_{ij}'} (\exp(\theta_{ij}) - \exp(-\theta_{ij})) (\exp(2\theta_i) - \exp(-2\theta_i)) \\
    \Delta_{ij} &= \frac{2}{z_{ij} z_{ij}'} (\exp(\theta_{ij}) - \exp(-\theta_{ij})) (\exp(2\theta_i) + \exp(-2\theta_i) + \exp(2\theta_j) + \exp(-2\theta_j)) \\
    a_i' &= \frac{2}{z_{ij}'} \exp(\theta_i) (\exp(\theta_j) + \exp(-\theta_j)) - 1 \\
    a_j' &= \frac{2}{z_{ij}'} \exp(\theta_j) (\exp(\theta_i) + \exp(-\theta_i)) - 1 \\
    z_{ij} &= \sum_{{s_i, s_j}} \exp(s_i \theta_i + s_j \theta_j + s_i s_j \theta_{ij}) \\
    z_{ij}' &=  \sum_{{s_i, s_j}} \exp(s_i \theta_i + s_j \theta_j)
\end{align}

Using these values, it is also possible to verify that $\varepsilon_{ij} \in (0, 1)$ if $\theta_i, \theta_j, \theta_{ij} > 0$.

\label{lemma:varepsilon}

\end{lemma}

\begin{proof}
We define a new distribution, which we denote by $\Pr'$ and $\mathbb{E}'$, that does not have an edge between $\lf_i$ and $\lf_j$:
\begin{align}
    \mathrm{Pr}'(Y, \bm{\lf}) = \frac{1}{Z'} \exp \Big(\theta_Y + \sum_{i = 1}^m \theta_i \lf_i Y + \sum_{(k, l) \neq (i,j)} \theta_{kl} \lf_k \lf_l \Big).
    \label{eq:wrong_pgm}
\end{align}

This distribution uses all the same canonical parameters as \eqref{eq:pgm} except $\theta_{ij} \lf_i \lf_j$. We know that for this distribution, $\Eprime{}{\lf_i \lf_j} = \Eprime{}{\lf_i Y} \Eprime{}{\lf_j Y}$. Our approach to compute $\varepsilon_{ij} = \E{}{\lf_i \lf_j} - \E{}{\lf_i Y} \E{}{\lf_j Y}$ is to bound the differences between $\mathbb{E}$ and $\mathbb{E}'$. 

First, we evaluate $\E{}{\lf_i Y} - \E{}{\lf_j Y}$. We write $\E{}{\lf_i Y}$ as $2\Pr(\lf_i Y = 1) - 1 = \frac{2}{p} \Pr(\lf_i = 1, Y = 1) - 1$ and $\Eprime{}{\lf_i Y}$ as $\frac{2}{p} \Pr'(\lf_i = 1, Y = 1) - 1$  by Lemma \ref{lemma:fs_symmetry}, where $p = \Pr(Y = 1).$ Then, letting $s_{-i}$ represent all combinations of labels on all $\bm{\lf}$ besides $\lf_i$,
\begin{align}
    \Delta_i &= \E{}{\lf_i Y} - \Eprime{}{\lf_i Y} = \frac{2}{p} \sum_{s_{-i}} \exp \bigg(\theta_Y + \theta_i + \sum_{k \neq i} \theta_k l_k + \sum_{(k, l) \neq (i, j)} \theta_{kl} s_k s_l \bigg) \left(\frac{\exp(\theta_{ij} l_j)}{Z} - \frac{1}{Z'} \right)
\end{align}

Next, note that $p = \frac{z_Y}{Z}$ and $p = \frac{z_Y'}{Z'}$, where $z_Y = \sum_{s} \exp(\theta_Y + \sum_{k = 1}^m \theta_k s_k + \sum_{(k, l) \neq (i, j)} \theta_{kl} s_k s_l + \theta_{ij} s_i s_j)$ and  $z_Y' = \sum_s \exp(\theta_Y + \sum_{k = 1}^m \theta_i s_i + \sum_{(k, l) \neq (i, j)} \theta_{kl} s_k s_l)$ (we can check these expressions for $p$ are equal, since the edgewise potentials are canceled out). $\Delta_i$ is now 
\begin{align}
     &2 \exp(\theta_i) \sum_{s_{-i}} \exp \bigg(\theta_Y + \sum_{k \neq i} \theta_k l_k + \sum_{(k, l) \neq (i, j)} \theta_{kl} s_k s_l\bigg) \bigg(\frac{\exp(\theta_{ij} l_j)}{z_Y} - \frac{1}{z_Y'} \bigg) \\
     = &2 \exp(\theta_i + \theta_j) \sum_{s_{-i, j}} \exp \bigg(\theta_Y + \sum_{k \neq i, j} \theta_k l_k + \sum_{(k, l) \neq (i, j)} \theta_{kl} s_k s_l\bigg) \bigg(\frac{\exp(\theta_{ij})}{z_Y} - \frac{1}{z_Y'} \bigg) \nonumber \\
     + &2 \exp(\theta_i - \theta_j) \sum_{s_{-i, j}} \exp \bigg(\theta_Y + \sum_{k \neq i, j} \theta_k l_k + \sum_{(k, l) \neq (i, j)} \theta_{kl} s_k s_l\bigg) \bigg(\frac{\exp(-\theta_{ij})}{z_Y} - \frac{1}{z_Y'} \bigg) \nonumber 
\end{align}

$\frac{\exp(\pm \theta_{ij})}{z_Y} - \frac{1}{z_Y'}$ can be written as $\frac{1}{z_Y z_Y'} \sum_{s'} \exp(\theta_Y + \sum_k \theta_k s_k' + \sum_{(k, l) \neq (i, j)} \theta_{kl} s_k' s_l') (\exp(\pm \theta_{ij}) - \exp(\theta_{ij} s_i' s_j'))$. Then for positive $\theta_{ij}$, this becomes $\frac{1}{z_Y z_Y'} (\exp(\theta_i - \theta_j) + \exp(-\theta_i + \theta_j))\sum_{s'} \exp(\theta_Y + \sum_{k \neq i, j} \theta_k s_k' + \sum_{(k, l) \neq (i, j)} \theta_{kl} s_k' s_l') (\exp(\theta_{ij}) - \exp(-\theta_{ij}))$, and for negative $-\theta_{ij}$, this becomes $\frac{1}{z_Y z_Y'} (\exp(\theta_i + \theta_j) + \exp(-\theta_i - \theta_j))\sum_{s'} \exp(\theta_Y + \sum_{k \neq i, j} \theta_k s_k' + \sum_{(k, l) \neq (i, j)} \theta_{kl} s_k' s_l') (\exp(-\theta_{ij}) - \exp(\theta_{ij}))$. Then, our expression becomes
\begin{align}
    \frac{2}{z_Y z_Y'} \bigg(\sum_{s_{-i, j}} &\exp \Big(\theta_Y + \sum_{k \neq i, j} \theta_k l_k + \sum_{(k, l) \neq (i, j)} \theta_{kl} s_k s_l\Big)\bigg)^2   (\exp(\theta_{ij}) - \exp(-\theta_{ij})) \\
    &\times \big(\exp(\theta_i + \theta_j) (\exp(\theta_i - \theta_j) + \exp(-\theta_i + \theta_j)) - \exp(\theta_i - \theta_j) (\exp(\theta_i + \theta_j) +\exp(-\theta_i - \theta_j)) \big) \nonumber 
\end{align}

The second line simplifies $\exp(2\theta_i) + \exp(2\theta_j) - \exp(2\theta_i) - \exp(-2\theta_j) = \exp(2\theta_j) - \exp(-2\theta_j)$. Lastly, note that $z_Y = \sum_{s_{-i, j}} \exp \Big(\theta_Y + \sum_{k \neq i, j} \theta_k l_k + \sum_{(k, l) \neq (i, j)} \theta_{kl} s_k s_l\Big) \cdot \sum_{s_i, s_j} \exp(s_i \theta_i + s_j \theta_j + s_i s_j \theta_{ij})$, and $z_Y' = \sum_{s_{-i, j}} \exp \Big(\theta_Y + \sum_{k \neq i, j} \theta_k l_k + \sum_{(k, l) \neq (i, j)} \theta_{kl} s_k s_l\Big) \cdot \sum_{s_i, s_j} \exp(s_i \theta_i + s_j \theta_j)$. Canceling out the summations over the other sources, we have our desired expression for $\Delta_i$. We can do the same to get our result for $\Delta_j$.

Next, we compute $\Delta_{ij} = \E{}{\lf_i \lf_j} - \Eprime{}{\lf_i \lf_j},$ which is equal to $2(\Pr(\lf_i = 1, \lf_j = 1) - \Pr'(\lf_i = 1, \lf_j = 1) + \Pr(\lf_i = -1, \lf_j = -1) - \Pr'(\lf_i = -1, \lf_j = -1))$:
\begin{align}
    \Pr(\lf_i = 1, \lf_j = 1) &- \mathrm{Pr}'(\lf_i = 1, \lf_j = 1) \label{eq:deltaij_pos} \\
    &= \sum_{Y, s_{-i, j}} \exp \Big(\theta_Y Y + \theta_i Y + \theta_j Y + \sum_{k \neq i, j} \theta_k s_k Y + \sum_{(k, l) \in (i, j)} \theta_{kl} s_k s_l\Big) \left(\frac{\exp(\theta_{ij})}{Z} - \frac{1}{Z'} \right), \nonumber  \\
    \Pr(\lf_i = -1, \lf_j = -1) &- \mathrm{Pr}'(\lf_i = -1, \lf_j = -1) \label{eq:deltaij_neg} \\
    &= \sum_{Y, s_{-i, j}} \exp \Big(\theta_Y Y - \theta_i Y - \theta_j Y + \sum_{k \neq i, j} \theta_k s_k Y + \sum_{(k, l) \in (i, j)} \theta_{kl} s_k s_l\Big) \left(\frac{\exp(\theta_{ij})}{Z} - \frac{1}{Z'} \right). \nonumber 
\end{align}

We can write $\frac{\exp(\theta_{ij})}{Z} - \frac{1}{Z'}$ as $\frac{1}{Z' Z} \sum_{Y, s} \exp\big(\theta_Y Y + \sum_{k = 1}^m \theta_k s_k Y + \sum_{(k, l) \neq (i, j)} \theta_{kl} s_k s_l\big) (\exp(\theta_{ij}) - \exp(\theta_{ij} s_i s_j))$, which is equal to $\frac{1}{Z' Z} (\exp(\theta_{ij}) - \exp(-\theta_{ij})) (\exp(\theta_i - \theta_j) + \exp(-\theta_i + \theta_j)) \sum_{Y, s_{-i, j}} \exp\big(\theta_Y Y + \sum_{k \neq i, j} \theta_k s_k Y + \sum_{(k, l) \neq (i, j)} \theta_{kl} s_k s_l\big)$. Therefore, $\Delta_{ij}$ is equal to two times $\eqref{eq:deltaij_pos}$ plus $\eqref{eq:deltaij_neg}$:
\begin{align}
    \Delta_{ij} =& 2 \left(\frac{\exp(\theta_{ij})}{Z} - \frac{1}{Z'} \right) \sum_{Y, s_{-i, j}} \exp \big(\theta_Y Y + \sum_{k \neq i, j} \theta_k s_k Y + \sum_{(k, l) \in (i, j)} \theta_{kl} s_k s_l\big) \big(\exp(\theta_i Y + \theta_j Y) + \exp(-\theta_i Y - \theta_j Y) \big) \\
    =& \frac{2}{Z' Z} (\exp(\theta_{ij}) - \exp(-\theta_{ij})) (\exp(\theta_i - \theta_j) + \exp(-\theta_i + \theta_j))(\exp(\theta_i + \theta_j) + \exp(-\theta_i - \theta_j))  \nonumber \\
    & \times \bigg(\sum_{Y, s_{-i, j}} \exp \big(\theta_Y Y + \sum_{k \neq i, j} \theta_k s_k Y + \sum_{(k, l) \in (i, j)} \theta_{kl} s_k s_l\big)\bigg)^2 \nonumber \\
    =& \frac{2}{Z' Z} (\exp(\theta_{ij}) - \exp(-\theta_{ij})) (\exp(2\theta_i) + \exp(-2\theta_i) + \exp(2\theta_j) + \exp(-2\theta_j))  \nonumber \\
    & \times \bigg(\sum_{Y, s_{-i, j}} \exp \big(\theta_Y Y + \sum_{k \neq i, j} \theta_k s_k Y + \sum_{(k, l) \in (i, j)} \theta_{kl} s_k s_l\big)\bigg)^2. \nonumber 
\end{align}

Note that $Z = \sum_{Y, s_{-i, j}} \exp(\theta_Y Y + \sum_{k \neq i, j} \theta_k s_k Y + \sum_{(k, l) \in (i, j)} \theta_{kl} s_k s_l) \sum_{s_i, s_j} \exp(s_i \theta_i + s_j \theta_j + s_i s_j \theta_{ij})$ and $Z' = \sum_{Y, s_{-i, j}} \exp(\theta_Y Y + \sum_{k \neq i, j} \theta_k s_i Y + \sum_{(k, l) \in (i, j)} \theta_{kl} s_k s_l) \sum_{s_i, s_j} \exp(s_i \theta_i + s_j \theta_j)$. Plugging this back in and canceling out summations, we obtain our desired result for $\Delta_{ij}$.

We now can compute $\varepsilon_{ij}$:
\begin{align}
    \varepsilon_{ij} &= \E{}{\lf_i \lf_j} - \E{}{\lf_i Y} \E{}{\lf_j Y} = \Eprime{}{\lf_i \lf_j} + \Delta_{ij} - (\Eprime{}{\lf_i Y} + \Delta_i)(\Eprime{}{\lf_j Y} + \Delta_j) \\
    &= \Delta_{ij} - \Delta_i \Eprime{}{\lf_j Y} - \Delta_j \Eprime{}{\lf_i Y} - \Delta_i \Delta_j. \nonumber 
\end{align}

Lastly, we need to compute $\Eprime{}{\lf_i Y}$ and $\Eprime{}{\lf_j Y}$:
\begin{align}
    \Eprime{}{\lf_i Y} = &2 \left(\textrm{Pr}'(\lf_i = 1, Y = 1) + \textrm{Pr}'(\lf_i = -1, Y = -1) \right) - 1 \\
    = &\frac{2}{Z'} \exp(\theta_i) (\exp(\theta_j) + \exp(-\theta_j)) \nonumber \\
    & \times \sum_{s_{-i, j}} \exp \Big( \sum_{(k, l) \neq (i, j)} \theta_{kl} s_k s_l\Big) \Big(\exp\Big(\theta_Y + \sum_{k \neq i, j} \theta_k s_k\Big) + \exp\Big(-\theta_Y - \sum_{k \neq i, j} \theta_k s_k\Big)\Big) - 1. \nonumber 
\end{align}

$Z'$ can be written as $\sum_{s_i, s_j} \exp(s_i \theta_i + s_j \theta_j) \sum_{s_{-i, j}} \exp \Big(\sum_{(k, l) \notin (i, j)} \theta_{kl} s_k s_l \Big)\Big(\exp(\theta_Y + \sum_{k \neq i, j} \theta_k s_k) + \exp(-\theta_Y - \sum_{k \neq i, j} \theta_k s_k) \Big)$, Therefore $\Eprime{}{\lf_i Y}$ is equal to
\begin{align}
    \Eprime{}{\lf_i Y} &= \frac{2 \exp(\theta_i) (\exp(\theta_j) + \exp(-\theta_j))}{\sum_{s_i, s_j} \exp(s_i \theta_i + s_j \theta_j)} - 1.
\end{align}

The key takeaways from this lemma are:
\begin{enumerate}
    \item Impact of misspecification in our computations exhibits some form of Lipschitzness, i.e. it is bounded in terms of the canonical parameters of our distribution. 
    \item One misspecified edge only contributes error defined in terms of the canonical parameters on the two vertices and the unmodeled edge between them.
    \item Under our assumptions, $\varepsilon_{ij} > 0$.
\end{enumerate}

\end{proof}

\begin{lemma} (Estimation error of accuracies via triplet method.) In the case of unlabeled data, accuracies estimated using the triplet method in \eqref{eq:triplet} satisfy
\begin{align*}
\E{\N, \tau}{\widetilde{a}_i - \bar{a}_i} &\le \frac{\sqrt{3}}{2\sqrt{n_U}} \cdot \sqrt{\frac{1 - b_{\min}^2}{b_{\min}^2} \left(\frac{1}{b_{\min}^4} + \frac{2}{b_{\min}^2} \right)}, \\
\E{\N, \tau}{(\widetilde{a}_i - \bar{a}_i)^2} &\le \frac{3}{4 n_U} \cdot \frac{1 - b_{\min}^2}{b_{\min}^2} \left(\frac{1}{b_{\min}^4} + \frac{2}{b_{\min}^2} \right).
\end{align*}
\label{lemma:sampling_error}
\end{lemma}

\begin{proof}
First, note that $\E{\N, \tau}{\bar{a}_i - \widetilde{a}_i} = \E{\N, \tau}{\E{\tau}{\bar{a}_i^{(j, k)}} - \widetilde{a}_i} = \E{\N, \tau}{\bar{a}_i^{(j, k)} - \widetilde{a}_i}$. Therefore, is it sufficient to produce an upper bound on $\E{\N}{\bar{a}_i^{(j, k)} - \widetilde{a}_i | \lf_j, \lf_k}$ independent of $j, k$. For ease of notation, we refer to this expectation as $\E{}{\bar{a}_i - \widetilde{a}_i}$. Then, $\E{}{\bar{a}_i - \widetilde{a}_i} = \E{}{ \frac{\bar{a}_i^2 - \widetilde{a}_i^2}{\bar{a}_i + \widetilde{a}_i} } \le  \frac{1}{2b_{\min}} \E{}{|\bar{a}_i^2 - \widetilde{a}_i^2 |}$. Denote $M_{ij} = \E{}{\lf_i \lf_j}$ and $\hat{M}_{ij} = \Ehat{\lf_i \lf_j}$. Then, by definition of our estimator in \eqref{eq:triplet},
\begin{align}
    \E{}{\bar{a}_i - \widetilde{a}_i} &\le \frac{1}{2b_{\min}} \E{}{\frac{\hat{M}_{ij} \hat{M}_{ik}}{\hat{M}_{jk} M_{jk}} |\hat{M}_{jk} - M_{jk}| + \frac{\hat{M}_{ij}}{M_{jk}}|\hat{M}_{ik} - M_{ik}| + \frac{M_{ik}}{M_{jk}} |\hat{M}_{ij} - M_{ij} | } \label{eq:M_decomposition} \\
    &\le \frac{1}{2b_{\min}} \E{}{\frac{1}{b_{\min}^2} |\delta_{jk}| + \frac{1}{b_{\min}} |\delta_{ik} | + \frac{1}{b_{\min}}|\delta_{ij} |},\nonumber
\end{align}

where $\delta_{ij} = \hat{M}_{ij} - M_{ij}$ is the estimation error for the pairwise expectations. Using Cauchy-Schwarz inequality,
\begin{align*}
    \E{}{\bar{a}_i - \widetilde{a}_i} &\le \frac{1}{2b_{\min}} \sqrt{\frac{1}{b_{\min}^4} + \frac{2}{b_{\min}^2}} \E{}{\sqrt{\delta_{ij}^2 + \delta_{ik}^2 + \delta_{jk}^2}} \\
    &\le \frac{1}{2b_{\min}} \sqrt{\frac{1}{b_{\min}^4} + \frac{2}{b_{\min}^2}} \sqrt{\Var{}{\hat{M}_{ij}} + \Var{}{\hat{M}_{ik}} + \Var{}{\hat{M}_{jk}}}.
\end{align*}

Formally, $\hat{M}_{ij} = \frac{1}{n_U} \sum_{l = 1}^{n_U} \lf_i^l \lf_j^l$. Therefore, $\Var{}{M_{ij}} = \frac{1}{n_U^2} \sum_{l = 1}^{n_U} \E{}{(\lf_i^l)^2 (\lf_j^l)^2} - M_{ij}^2 = \frac{1 - M_{ij}^2}{n_U} \le \frac{1 - b_{\min}^2}{n_U}$, and our bound becomes
\begin{align*}
    \E{}{\bar{a}_i - \widetilde{a}_i} \le \frac{\sqrt{3}}{2\sqrt{n_U}} \cdot \sqrt{\frac{1 - b_{\min}^2}{b_{\min}^2} \left(\frac{1}{b_{\min}^4} + \frac{2}{b_{\min}^2} \right)}.
\end{align*}

Next, to bound $\E{\N, \tau}{(\widetilde{a}_i - \bar{a}_i)^2}$, it is sufficient to upper bound $\E{\N}{(\widetilde{a}_i - \bar{a}_i^{(j, k)})^2 \; | \; \lf_j, \lf_k}$ independent of choice of $j$ and $k$. Refer to this expectation as $\E{}{(\widetilde{a}_i - \bar{a}_i)^2}$. Then, $\E{}{(\widetilde{a}_i - \bar{a}_i)^2 } = \E{}{\frac{(\widetilde{a}_i^2 - \bar{a}_i^2)^2}{(\widetilde{a}_i + \bar{a}_i)^2}} \le \frac{1}{4b_{\min}^2} \E{}{(\widetilde{a}_i^2 - \bar{a}_i^2)^2}$. Similar to \eqref{eq:M_decomposition},
\begin{align}
    \E{}{(\widetilde{a}_i - \bar{a}_i)^2 } &\le \frac{1}{4b_{\min^2}} \E{}{\left(\frac{\hat{M}_{ij} \hat{M}_{ik}}{\hat{M}_{jk} M_{jk}} |\hat{M}_{jk} - M_{jk}| + \frac{\hat{M}_{ij}}{M_{jk}}|\hat{M}_{ik} - M_{ik}| + \frac{M_{ik}}{M_{jk}} |\hat{M}_{ij} - M_{ij} |  \right)^2} \\
    &\le \frac{1}{4b_{\min}^2} \E{}{\left(\frac{1}{b_{\min}^2} |\delta_{jk}| + \frac{1}{b_{\min}} |\delta_{ik} | + \frac{1}{b_{\min}}|\delta_{ij} | \right)^2} \nonumber \\
    &\le \frac{1}{4b_{\min}^2} \left(\frac{1}{b_{\min}^4} + \frac{2}{b_{\min}^2}\right) \left(\Var{}{\hat{M}_{ij}} + \Var{}{\hat{M}_{ik}} + \Var{}{\hat{M}_{jk}}\right) \nonumber \\
    &\le \frac{3}{4 n_U} \cdot \frac{1 - b_{\min}^2}{b_{\min}^2} \left(\frac{1}{b_{\min}^4} + \frac{2}{b_{\min}^2} \right) \nonumber 
\end{align}

\end{proof}

\section{Additional Experimental Details}

We provide additional details on experiments. Our code can be found at \url{https://github.com/bencw99/comparing-labeled-and-unlabeled-data}.

\subsection{Synthetic Experiments} \label{subsec:supp_synthetics}

In this section, we first provide our protocol for generating synthetic data, which is fixed across our synthetic experiments. We then discuss the details of the experiments performed for each of the plots in \autoref{sec:theory} and \autoref{sec:applications}. 

\paragraph{Generating synthetic data}

We use the same synthetic data distributions for all of our synthetic experiments. We set the number of sources to $m=10$, and draw accuracies uniformly from $[.55,.75]$, both of which would be typical in relevant applications (ex., in weak supervision). We report these accuracies in \autoref{tab:synthetic_accuracies}. For experiments with dependencies, when $d=1$ we add the edge $(0,1)$, when $d=2$ we add a second edge $(2,3)$ and so on. Every dependency is fixed at $\varepsilon_{ij}=\mathbb{E}[\lambda_i\lambda_j]-\mathbb{E}[\lambda_i]\mathbb{E}[\lambda_j]=0.1$.

\begin{table}[t]
\vskip 0.15in
\renewcommand{\arraystretch}{1.25} % Default value: 1
\begin{center}
\begin{small}
\begin{tabular}{l|cccccccccccr}
\hline
$i$ & 0 & 1 & 2 & 3 & 4 & 5 & 6 & 7 & 8 & 9\\
\hline
Accuracy & .6893 & .6072 & .5954 & .6603 & .6939 & .6346 & .7462 & .6870 & .6462 & .6284 \\
\hline
\end{tabular}
\end{small}
\end{center}
\vskip -0.1in
\caption{The source accuracies used for synthetic experiments. They were each drawn uniformly from $[.55,.75]$.}
\label{tab:synthetic_accuracies}
\end{table}

\paragraph{\autoref{fig:gen_err}: Excess generalization error} 

We measure the expected excess generalization error for several different estimators and values of $n$. For each value of $n$, we take $1000$ samples and measure the generalization error of an estimator trained on this sample. We average the results over these $1000$ samples.

\paragraph{\autoref{fig:data_value_ratio}: Computing the data value ratio}

We compute the data value ratio for unlabeled models with mean and median aggregation for different numbers of dependencies $d$. The definition of the data value ratio requires finding the smallest $n_L$ with which learning from $n_L$ labeled points achieves lower expected generalization error than learning from $n_U$ unlabeled points. To measure the expected generalization error for some $n$, we average over $1000$ samples, which would be intractable to do for every $n_L$. Therefore, we measure the expected generalization error for every $n_L$ between $10$ and $100$, every $n_L$ divisible by $2$ between $100$ and $1000$ and every $n_L$ divisible by $10$ between $1000$ and $5000$. Besides this shortcut, we compute the data value ratio according to its definition.

\paragraph{\autoref{fig:combined}: Combining labeled and unlabeled data}

We compare the practical approach of weighting the unlabeled and labeled estimators according to \cite{GreenStrawderman2001}, formally defined in section \ref{subsec:combined}, with the optimal weight. We let the optimal weight vary with $n_U$ and $n_L$, but not with the specific data points drawn. In other words, we compute the optimal weight to be that which minimizes the average generalization error over $1000$ trials for each $n_L$. On the other hand, the weight from \cite{GreenStrawderman2001} is a function of the learned accuracies (and thus of the specific data points drawn). In \autoref{fig:combined_with_alphas} we report the optimal $\alpha$ for each $n_L$ ($n_U$ is fixed at $1000$) as well as the \textit{average} weight from \cite{GreenStrawderman2001} over $1000$ trials.

\begin{figure}
    \centering
    \includegraphics[width=.6\textwidth]{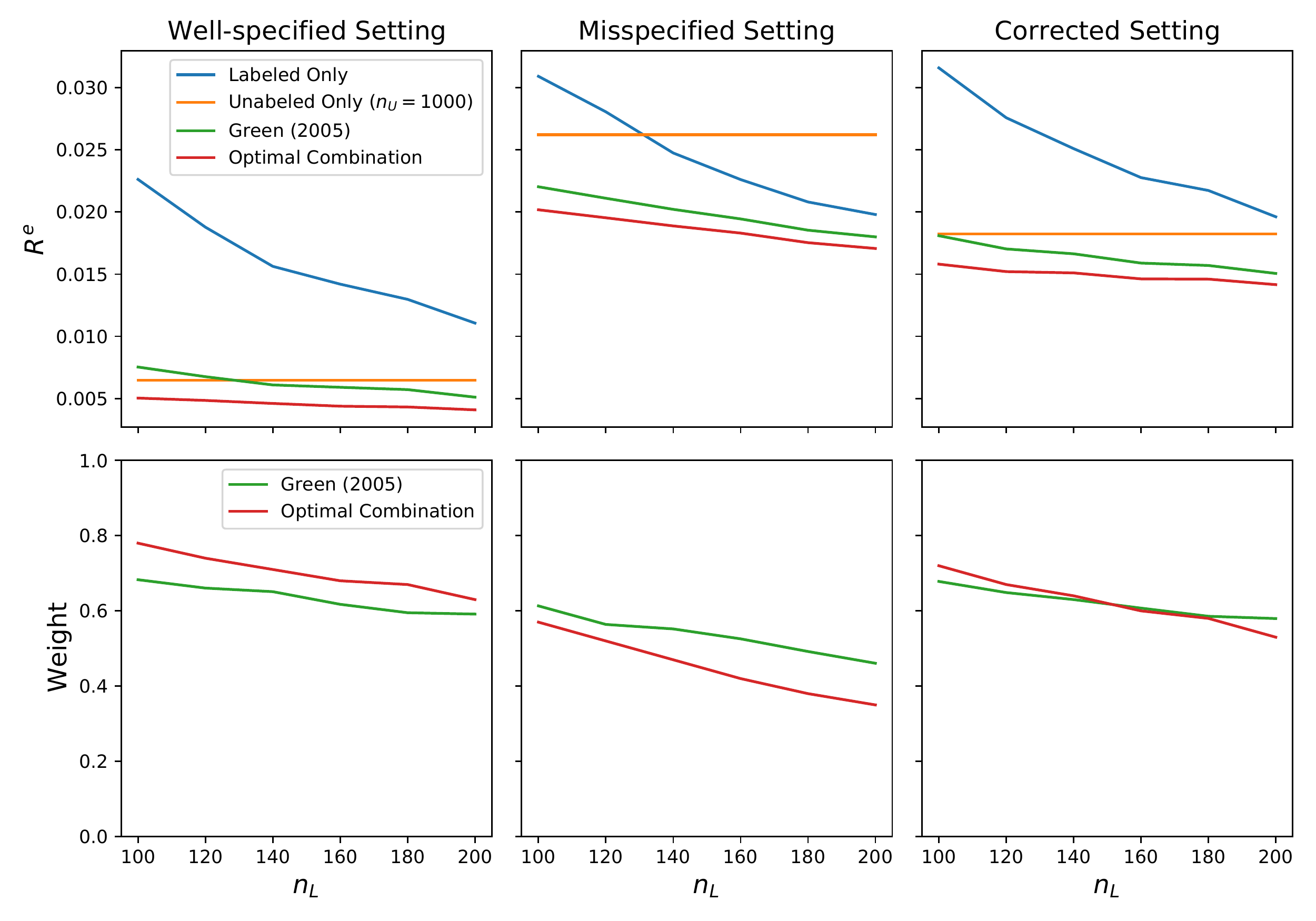}
    \caption{Excess generalization error and associated combination weight $\alpha$ for an optimally weighted combination of labeled and unlabeled estimators, and a combination weighted according to \cite{GreenStrawderman2001} across the well-specified (left), misspecified (center), and corrected (right) settings. The number of unlabeled points is fixed at $n_U=1000$.}
    \label{fig:combined_with_alphas}
\end{figure}

\subsection{Real-World Case Study: Weak Supervision} \label{subsec:supp_realdata}

We discuss the weak supervision dataset we create and clarify the details of our experimental protocol for the real-world case study.

\paragraph{Creating a weak supervision dataset}

In weak supervision, soft labels from latent variable estimation are used as an alternative to a hand-labeled dataset. The sources used are usually heuristics which incorporate domain-specific knowledge about a particular task and can be acquired relatively cheaply. For our real-world case study, we choose the simple sentiment analysis task of classifying IMDB reviews as positive or negative. Our sources are defined simply: for a collection of positive sentiment words, output ``yes'' if the word appears in the review and ``no'' otherwise; for a collection of negative sentiment words, similarly output ``no'' if the word appears and ``yes'' otherwise. The specific words used and their sentiments are reported in \autoref{tab:imdb_words}. We select these words because they are empirically predictive, appear relatively frequently in reviews and are intuitively associated with positive/negative reviews.

\begin{table}[t]
\vskip 0.15in
\renewcommand{\arraystretch}{1.25} % Default value: 1
\begin{center}
\begin{small}
\begin{tabular}{l|cccccccccccccr}
\hline
Word & love & like & good & great & best & excellent & terrible & worst & bad & better & could & would \\
\hline
Sentiment & + & + & + & + & + & + & - & - & - & - & - & - \\
\hline
\end{tabular}
\end{small}
\end{center}
\vskip -0.1in
\caption{The words used as sources for the real-world weak supervision task of classifying IMDB reviews as positive or negative.}
\label{tab:imdb_words}
\end{table}

\paragraph{\autoref{fig:real_biases_data_value_ratio} and \autoref{tab:real_combo}: Experiments with real data} We measure excess generalization error, the data value ratio and the performances of combined estimators for the real-world dataset. Our protocols for these experiments mirror those we used for synthetic datasets, with two key differences: (1) for each trial, we sample points uniformly from the training set of 40,000 points, since we cannot sample directly from the distribution and (2) we measure generalization error on the test set, since we cannot compute the expected generalization error directly.

\end{document}